\documentclass[twoside,11pt]{article}

\usepackage{jmlr2e}
\usepackage{dsfont}
\usepackage{afterpage}

\usepackage{microtype}
\usepackage{graphicx}
\usepackage{booktabs} 
\usepackage{tikz}
\usepackage{paralist}

\usepackage[utf8]{inputenc} 
\usepackage[T1]{fontenc}    
\usepackage{hyperref}       
\usepackage{url}            
\usepackage{booktabs}       
\usepackage{amsfonts}       
\usepackage{nicefrac}       
\usepackage{microtype}      
\usepackage{xcolor}         
\usepackage{mathtools}       

\usepackage{multicol}
\usepackage{microtype}
\usepackage{graphicx}

\usepackage{bbm}

\usepackage{float}
\usepackage{booktabs} 
\usepackage{amssymb}
\usepackage{amsmath}
\usepackage{enumitem}
\usepackage{color}
\usepackage{xcolor}
\usepackage{bm}
\usepackage{bbm}
\usepackage{sidecap}
\sidecaptionvpos{figure}{t}
\usepackage{wrapfig}
\usepackage[font={small}]{caption}
\usepackage{subcaption}

\usepackage{amsfonts} 
\DeclareMathSymbol{\shortminus}{\mathbin}{AMSa}{"39}

\DeclareMathOperator*{\argmin}{argmin}
\newcommand{\app}{{App.~}}
\newcommand{\appl}{{Appendix~}}

\newcommand{\fig}{{Fig.~}}
\newcommand{\tab}{{Table~}}

\newcommand{\prop}{{Prop.~}}
\newcommand{\thm}{{Thm~}}
\newcommand{\cor}{{Corr~}}
\newcommand{\lem}{{Lemma~}}

\newcommand{\ba}{\bm{a}}

\newcommand{\bx}{\bm{x}}

\newcommand{\by}{\bm{y}}
\newcommand{\bw}{\bm{w}}
\newcommand{\sQ}{\textsf{Q}}
\newcommand{\bW}{\bm{W}}
\newcommand{\bmm}{\bm{m}}

\newcommand{\X}{\mathbb{X}}
\newcommand{\R}{\mathbb{R}}

\newcommand{\E}{\mathbb{E}}

\newcommand{\bX}{{\mathbf X}}

\newcommand{\hQ}{\bm{Q}}

\newcommand{\hW}{{\bW}}

\newcommand{\hV}{{V}}

\newcommand{\cm}{\mathbbm{m}}
\newcommand{\iX}{\mathbb{X}}
\newcommand{\iE}{\mathbb{E}}
\newcommand{\var}{\mathrm{var}} 

\newcommand{\Dm}{\Delta m}
\newcommand{\lmat}{{V}}
\newcommand{\rmat}{{U}}
\newcommand{\Bl}{{B}}

\newcommand{\Al}{{A}}

\newcommand{\ri}{{l}}
\newcommand{\rn}{{n}}
\newcommand{\rl}{{l}}

\newcommand{\tL}{{N}}
\newcommand{\mL}{K}
\newcommand{\ttL}{{L}}

\newcommand{\omverm}{O\hspace{-2pt}\left(\hspace{-2pt}\frac{1}{\cm}\hspace{-2pt}\right)}
\newcommand{\omvermK}{O\left(\frac{\mL}{\cm}\right)}
\newcommand{\omvermq}{O\left(\frac{q}{\cm}\right)}

\newcommand{\repI}{F}
\newcommand{\nunu}{{t'}}
\newcommand{\rhorho}{{t''}}

\newcommand{\SXX}{\Sigma_{\scriptscriptstyle XX}}
\newcommand{\SYX}{\Sigma_{\scriptscriptstyle YX}}

\definecolor{airforceblue}{rgb}{0.36, 0.54, 0.66}
\newcommand{\hlt}[1]{{\color{airforceblue} #1}}
\newcommand{\texp}[1]{\quad\mathrm{\hlt{(#1)}}}
\newcommand{\texq}[1]{\mathrm{\hlt{(#1)}}}
\newcommand{\comment}[1]{{}}

\firstpageno{1}

\usepackage{lastpage}
\jmlrheading{23}{2022}{1-\pageref{LastPage}}{8/21; Revised
5/22}{5/22}{21-0991}{Guy Hacohen and Daphna Weinshall}
\ShortHeadings{Principal Components Bias in Over-parameterized Linear Models}{Hacohen and Weinshall}

\begin{document}

\title{Principal Components Bias in Over-parameterized Linear Models, and its Manifestation in Deep Neural Networks}

\author{\name Guy Hacohen \email guy.hacohen@mail.huji.ac.il \\
       \addr The School of Computer Science and Engineering\\
       Edmond and Lily Safra Center for Brain Sciences\\
       The Hebrew University of Jerusalem\\
       Jerusalem, Israel \\
       \AND
       \name Daphna Weinshall \email daphna@mail.huji.ac.il \\
       \addr The School of Computer Science and Engineering\\
       The Hebrew University of Jerusalem\\
       Jerusalem, Israel \\}

\editor{Christoph Lampert}

\maketitle

\begin{abstract}
Recent work suggests that convolutional neural networks of different architectures learn to classify images in the same order. To understand this phenomenon, we revisit the over-parametrized deep linear network model. Our analysis reveals that, when the hidden layers are wide enough, the convergence rate of this model's parameters is exponentially faster along the directions of the larger principal components of the data, at a rate governed by the corresponding singular values. We term this convergence pattern the \emph{Principal Components bias (PC-bias)}. Empirically, we show how the \emph{PC-bias} streamlines the order of learning of both linear and non-linear networks, more prominently at earlier stages of learning. We then compare our results to the simplicity bias, showing that both biases can be seen independently, and affect the order of learning in different ways. Finally, we discuss how the \emph{PC-bias} may explain some benefits of early stopping and its connection to PCA, and why deep networks converge more slowly with random labels.
\end{abstract}


\section{Introduction}
\label{sec:intro}

The dynamics of learning in deep neural networks is an intriguing subject, not yet sufficiently understood. Recent empirical studies of learning dynamics \citep{hacohen2020let,pliushch2021deep,DBLP:conf/acl/ChoshenHWA22} showed that neural networks memorize the training examples of natural datasets in a consistent order, and further impose a consistent order on the successful recognition of unseen examples. Below we call this effect \emph{Learning Order Constancy} (LOC). Currently, the characteristics of visual data, which may explain this phenomenon, remain unclear. Surprisingly, this universal order persists despite the variability introduced into the training of different models and architectures.

To understand this phenomenon, we start by analyzing the deep linear network model \citep{DBLP:journals/corr/SaxeMG13,saxe2019mathematical}, defined by the concatenation of linear operators in a multi-class classification setting. Accordingly, in Section~\ref{sec:theo-res} we prove that the convergence of the weights of deep linear networks is governed by the eigendecomposition of the raw data, which is blind to the labels of the data, in a phenomenon we term \emph{PC-bias}. These results are valid when the hidden layers are wide enough, a generalization of the known behavior of the single-layer convex linear model.  
In Section~\ref{sec:results}, we empirically show that this pattern of convergence is indeed observed in deep linear networks of rather a moderate width, validating the plausibility of our assumptions. We continue by showing that the \emph{LOC-effect} in deep linear networks is determined solely by their \emph{PC-bias}. We prove a similar (weaker) result for the non-linear two-layer ReLU model trained on a binary classification problem, introduced by \citet{DBLP:conf/nips/Allen-ZhuLL19}.

In Section~\ref{sec:pc_bias_explain_loc}, we extend the study empirically to non-linear networks and investigate the relationship between the \emph{PC-bias} and the \emph{LOC-effect} in general deep networks. We first show that the order by which examples are learned by linear networks is highly correlated with the order induced by prevalent deep CNN models. We then show directly that the learning order of non-linear CNN models is affected by the principal components decomposition of the data. Moreover, the \emph{LOC-effect} diminishes when the data is whitened, indicating a tight connection between the \emph{PC-bias} and the \emph{LOC-effect}.

Our results are reminiscent of another phenomenon, termed \emph{Spectral bias} (see Section~\ref{sec:simplicity}), which associates the learning dynamics of neural networks with the Fourier decomposition of functions in the hypothesis space. In Section~\ref{sec:empricial-freq-bias} we investigate the relation between the \emph{PC-bias}, \emph{spectral bias}, and the \emph{LOC-effect}. We find that the \emph{LOC-effect} is very robust: (i) when we neutralize the \emph{spectral bias} by using low complexity models such as deep linear networks, the effect is still observed; (ii) when we neutralize the \emph{PC-bias} by using whitened data, the \emph{LOC-effect} persists. We hypothesize that at the beginning of learning, the learning dynamics of neural models is governed by the eigendecomposition of the raw data. As learning proceeds, control of the dynamics slowly shifts to other factors.

The PC-bias has implications beyond the \emph{LOC-effect}, as expanded in Section~\ref{sec:implication}:

\paragraph{(i) Early stopping.}
It is often observed that when training deep networks with real data, the highest generalization accuracy is obtained before convergence. Consequently, early stopping is often prescribed to improve generalization.
Following the commonly used assumption that in natural images the lowest principal components correspond to noise \citep{torralba2003statistics}, our results predict the benefits of early stopping, and relate it to PCA. In Section~\ref{sec:implication} we investigate the relevance of this conclusion to real non-linear networks \citep[see][for complementary accounts]{ronen2019convergence,li2020gradient}.

\paragraph{(ii) Slower convergence with random labels.}
\citet{DBLP:conf/iclr/ZhangBHRV17} showed that neural networks can learn any label assignment. However, training with random label assignments is known to converge slower as compared to training with the original labels \citep{DBLP:conf/iclr/KruegerBJAKMBFC17}. We report a similar phenomenon when training deep linear networks. Our analysis shows that when the principal eigenvectors are correlated with class identity, as is often the case in natural images, the loss decreases faster when given true label assignments as against random label assignments. In Section~\ref{sec:implication} we investigate this hypothesis empirically in linear and  non-linear networks.

\paragraph{(iii) Weight initialization.}
Different weight initialization schemes have been proposed to stabilize the learning and minimize the hazard of "exploding gradients" \citep[for example,][]{glorot2010understanding,He_2015_ICCV}. Our analysis (see \appl\ref{app:random}) identifies a related variant, which eliminates the hazard when all the hidden layers are roughly of equal width. In the deep linear model, it can be proven that the proposed normalization variant in a sense minimizes repeated gradient amplification.

\section{Scientific Background and Previous Work}
\label{sec:previous-work}

Below we review related work, and discuss the relation of our work to this prior art.

\subsection{Deep Over-Parametrized Neural Networks and Other Simple Models}
\label{sec:background}

A large body of work concerns the analysis of deep over-parameterized linear networks. While not a universal approximator, this model is nevertheless trained by minimizing a non-convex objective function with a multitude of equally valued global minima. The investigation of such networks is often employed to shed light on the learning dynamics when complex geometric landscapes are explored by GD \citep{fukumizu1998effect,arora2018optimization,DBLP:conf/nips/WuWM19,du2019width,du2019gradient,DBLP:conf/iclr/HuXP20,DBLP:conf/iclr/YunKM21}. Note that while such networks provably achieve a simple low-rank solution \citep{DBLP:conf/iclr/JiT19,DBLP:conf/nips/DuHL18} when given linearly separable data, in our work this strict assumption is not needed.

Early on \citet{baldi1989neural} characterized the optimization landscape of the over-parameterized linear model and its relation to PCA. More recent work suggests that with sufficient over-parameterization, this landscape is well behaved \citep{DBLP:conf/nips/Kawaguchi16,zhou2018critical} and all its local minima are global \citep{DBLP:journals/corr/abs-1712-01473}. Deep linear networks are also used to study biases induced by architecture or optimization \citep{DBLP:conf/iclr/JiT19,wu2019towards}. 

Several related studies investigated the dynamics of learning in over-parameterized models by way of spectral analysis. However, while we employ the eigenvectors of the data covariance matrix\footnote{$X$ and $Y$ denote the matrices whose columns are the data points and one-hot label vectors respectively, see notations in Section~\ref{sec:theo-res}.} $\SXX = X X^\top$, most of the behavior identified in these studies is driven by the eigenvectors of the cross-covariance matrix $\SYX = Y X^\top$. This difference is crucial, as $\SXX$ is blind to the class labels, unlike $\SYX$. In addition, quite a few of the studies reviewed below assumed a shallow 2-layer network, while our analysis accommodates over-parametrized networks of any depth, further showing that convergence depends on the eigenvectors of $\SXX$ rather than $\SYX$ without any additional assumption. 

More specifically, \citet{saxe2019mathematical} analyzed the evolution of learning dynamics in a shallow two-layered linear model, showing that convergence is guided by the eigenvectors of $\SYX$. This analysis assumed $\SXX = I_d$, thus obscuring the dependence of convergence on the raw data eigenvectors of $\SXX$, as all the eigenvalues are now identical by assumption. Likewise, \citet{arora2018optimization} also assumed that $\SXX = I_d$ while investigating continuous gradients of the deep linear model and allowing for a more general loss. Although there is some superficial resemblance between their pre-conditioning and our \emph{gradient scale matrices} as defined below, they are defined quite differently and have different properties.

Similarly to \citet{saxe2019mathematical}, \citet{gidel2019implicit} also investigated a shallow 2-layers linear model, but no longer assumed that $\SXX = I_d$. Instead, they assumed that $\SXX$ and $\SYX$ can be jointly decomposed, having the same eigenvectors. Under these conditions, they were able to show that  convergence is guided by the eigenvectors of $\SYX$, which correspond to the eigenvectors of $\SXX$ by the joint decomposition assumption. However, $\SXX$ is now linked to the class labels by assumption, which is not the case in most datasets. Recently, \citet{nguyen2021analysis} showed that the learning dynamics of two-layered non-linear auto-encoders converge faster along the eigenvectors of $\SXX$. While their results resemble ours, unlike us they focus on auto-encoders and assume that the data is sampled from a Gaussian distribution.



Relations between shallow and deep over-parameterized linear networks are often studied through the lens of gradient flow \citep{arora2018optimization,DBLP:conf/iclr/AroraCGH19}. \citet{DBLP:conf/iclr/AroraCGH19} derived an equation for the gradient-flow of over-parameterized deep linear networks, which was later expanded to shallow ReLU networks by \citet{DBLP:conf/nips/WilliamsTPSZB19}. \citet{DBLP:journals/corr/abs-1910-05505} showed that this gradient flow can be re-interpreted as a Riemannian gradient flow of matrices of some fixed ranked, hinting at a strong connection between the optimization process of shallow and deep linear networks, in accordance with our results.

Another line of related work was pioneered by \citet{DBLP:conf/nips/JacotHG18}, who showed how neural networks can be analyzed using kernel theory and introduced the Neural Tangent Kernel (NTK). In this framework, it was shown that convergence is fastest along the largest \textbf{kernel's} principal components. In the one-layer linear model, this result implies that convergence depends on the eigenvectors of $\SXX$, but to the best of our knowledge, it was not extended to the over-parameterized deep model analyzed here. While similar results to ours may be achieved in the future by using kernel theory, here we provide direct proof, thus bypassing possible limitations of the kernel theory \citep{DBLP:conf/nips/ChizatOB19,yehudai2019power}. This direct proof further enables us to examine the behavior outside the limit of infinite width, and examine our assumptions using empirical methods.

\subsection{Related Research Paradigms}
\label{sec:simplicity}

Diverse empirical observations seem to support the hypothesis that neural networks start by learning a simple model, which then gains complexity as learning proceeds \citep{DBLP:conf/nips/GunasekarLSS18,soudry2018implicit,DBLP:conf/nips/HuXAP20,DBLP:conf/nips/KalimerisKNEYBZ19,DBLP:conf/iclr/GissinSD20,DBLP:conf/iclr/HeckelS20,ulyanov2018deep,DBLP:conf/iclr/PerezCL19,conf_ijcai_CaoFWZG21,DBLP:journals/corr/abs-2006-07356}. This phenomenon is sometimes called \emph{simplicity bias} \citep{dingle2018input,DBLP:conf/nips/ShahTR0N20}.

In a related line of work, \citet{rahaman2019spectral} empirically demonstrated that the complexity of classifiers learned by ReLU networks increases with time.  \citet{ronen2019convergence,basri2020frequency} showed theoretically, by way of analyzing elementary neural network models, that these models first fit the data with low-frequency functions, then gradually acquire higher frequencies to improve the fit. Nevertheless, the \emph{spectral bias} and \emph{PC-bias} are inherently different. Indeed, the eigendecomposition of raw images is closely related to the Fourier analysis of images as long as the statistical properties of images are (approximately) translation-invariant \citep{simoncelli2001natural,torralba2003statistics}. Still, the \emph{PC-bias} is guided by spectral properties of the raw data and is therefore blind to class labels. In contrast, the \emph{spectral bias}, as well as the related \emph{frequency bias} that has been shown to characterize NTK models \citep{basri2020frequency}, are all guided by spectral properties of the learned hypothesis, which crucially depend on label assignment. 

Other related research paradigms include curriculum learning \citep{bengio2009curriculum,hacohen2019power,weinshall2020theory}, self-paced learning \citep{kumar2010self,tullis2011effectiveness}, hard data mining \citep{fu2017easy}, and active learning \citep{krogh1994neural,hacohen2022active, yehuda2022}. In these frameworks, the research focus is on the order in which data is presented to the learner. Differently, our focus is on characterizing the order by which data is learned without additional guidance to this effect.

\section{Theoretical Analysis}
\label{sec:theo-res}
In this section, we analyze the over-parameterized deep linear networks model, in which the principal component bias fully governs the learning dynamics.

\subsection{Notations and Definitions}
Let $\X = \{(\bx_i, \by_i)\}_{i=1}^n$ denote the training data, where $\bx_i\in\R^q$ denotes the i-th data point and one-hot vector $\by_i\in\{0,1\}^K$ denotes its corresponding label. Let $\frac{1}{n_i}\bmm_i$ denote the centroid (mean) of class $i$ with $n_i$ points $\frac{1}{n_i}\bmm_i = \frac{1}{n_i}\sum_{j=1}^{n}{\bx_j\mathbbm{1}_{\left[y_j=i\right]}}$, and $M=[\bmm_1\ldots \bmm_K]^\top$ the matrix whose rows are vectors $\bmm_i^\top$. Finally, let $X$ and $Y$ denote the matrices whose $i^{th}$ column is $\bx_i$ and $\by_i$ respectively. $\SXX=X X^\top$ and $\SYX=Y X^\top$ denote the covariance matrix of $X$ and cross-covariance of $X$ and $Y$ respectively. We note that $\SXX$ captures the structure of the data irrespective of class identity, and that $\SYX=M$.

\subsubsection{Definitions}

\begin{definition}[Principal coordinate system]
\label{def:canon}   
The coordinate system obtained by rotating the data in $\R^q$ by an orthonormal matrix $U^\top$\!\!, where $SVD(\SXX) \!= \!UDU^\top$\!. 
\end{definition}
In this system $\Sigma_{\scriptscriptstyle XX}\!=\!D$, a diagonal matrix whose elements are the singular values of $X X^\top$, arranged in decreasing order $d_1\geq d_2\geq \ldots \geq d_q \geq 0$. 

We analyze a deep linear network with $L$ layers, whose parameters are computed by minimizing (using gradient descent) the following loss $L(X,Y)$
\begin{align}
L(X,Y) =  \frac{1}{2} \Vert \hW X-Y \Vert_F^2, \quad\quad \hW \coloneqq \prod_{l=L}^1 W_l, \label{eq:multi-loss}
\quad W_l\in\R^{m_l\times m_{l-1}}.
\end{align}
Above $m_l$ denotes the number of neurons in layer $l$, where $m_0 =q$ and $m_L=K$. 

\begin{definition}[Compact representation]
Given a deep linear network $\by=W_L\cdot\ldots \cdot W_1 \bx$, its \textit{compact representation} is denoted $\hW\in\R^{K\times q}$ where $\hW=\prod_{l=L}^1 W_l$.
\end{definition}

\begin{definition}[Gradient matrix]
\label{def:err_mat}
For a deep linear network whose compact representation is $\hW$, the gradient matrix of $L(X,Y)$ with respect to $W$ is $G_r = \hW\SXX-\SYX$. 
\end{definition}
In the principal coordinate system, $G_r = \hW D-M$.

\begin{definition}[STD initialization]
\label{def:1}
Given a sequence of random matrices $\{W_\rl\}_{\rl=1}^\ttL$ where $W_\rl\in\R^{m_\rl\times m_{\rl-1}}$, and all the elements in $W_\rl$ are chosen i.i.d from a distribution with mean 0 and variance $\sigma^2_\rl$, define 
\begin{equation*}
 \sigma^2_\rl = \frac{2}{m_{\rl-1}+m_\rl} ~\forall \rl\in[2\ldots\ttL-1], \qquad\sigma^2_1=\frac{1}{m_1},\qquad \sigma^2_\ttL=\frac{1}{m_{\ttL-1}}.
\end{equation*}
\end{definition}
Def.~\ref{def:1} is a variant of the Glorot initialization \citep{glorot2010understanding}. For more analytical and empirical results on \emph{STD initialization}, refer to \app\ref{app:sec:evolution} and \app\ref{app:sec:weight-initialization} respectively.

\subsubsection{Asymptotic Analysis}

Our analysis involves two asymptotic quantities: the learning rate (or gradient step size) $\mu$, and the network's width $\cm$ ($\cm$ denotes the width of the smallest hidden layer). Accordingly, our analysis neglects terms of magnitude $O(\mu^2)$ and $\omverm$, and employs convergence in probability when $\cm\to\infty$. In \fig\ref{fig:gradient_scale_validation}, we show the plausibility of these assumptions, as the predicted dynamics is seen to hold even when the width $\cm$ is smaller than the input size, and the step size $\mu$ permits convergence in a relatively short time. 

To simplify the presentation and to allow for modular analysis, we use the following conventions: \begin{inparaenum}[(i)]
\item $O(\mu^2)$ and $\omverm$ denote terms which are upper bounded by $c\mu^2$ when $\mu$ is small and $c\frac{1}{\cm}$ when $\cm$ is large respectively, where $c\in\R$ denotes a fixed constant that does \emph{not} depend on neither $\mu$ nor $\cm$. 
\item Depending on the context, if matrices are involved then $O(\mu^2)$ and $O(\frac{1}{\cm})$ denote full rank matrices with element-wise asymptotic quantities defined as above. 
\item The notation $\xrightarrow{p}$ denotes convergence in probability when $\cm\to\infty$, where the probabilistic lower bound on width $\cm$ does \emph{not} depend on $\mu$. When matrices are involved, convergence in probability occurs element-wise.
\end{inparaenum}

The treatment of the two asymptotic quantities $\mu$ and $\cm$ is consolidated in the formal analysis of the network's dynamics, as detailed in the proof of \thm\ref{thm:delB} in \appl\ref{sec:random-mat-dyn}, and the proofs of Thms~\ref{thm:convergence-rate}-\ref{thm:teles} in \appl\ref{app:sec:evolution}. Importantly, the analysis is valid at each time step $t$, where the only thing that changes with time is the magnitude of the constants governing the asymptotic behavior.


\subsection{The Dynamics of Deep Over-Parametrized Linear Networks}
\label{sec:deep}

Extending the analysis described in \citet{du2019width}, we derive here the temporal dynamics of $\hW$, denoted $\hW^{(t)}$, as it changes with each GD step.

\begin{proposition}
\label{thm:1}
Let $G_r^{(t)}$ (Def.~\ref{def:err_mat}) denote the gradient matrix at time ${t}$. Let $\Bl_{\ri}^{(t)}$ and $\Al_{\ri}^{(t)}$ denote the gradient scale matrices, which are defined as follows
\begin{equation}
\label{eq:gradient-scale}
\begin{split}
\Bl_{\ri}^{(t)} \coloneqq \Big(\prod_{j=l}^{1} W_j^{(t)}\Big )^\top \Big(\prod_{j=l}^{1} W_j^{(t)}\Big ) \in\R^{q\times q},\qquad
\Al_{\ri}^{(t)} \coloneqq \Big(\prod_{j=\ttL}^{\ri+1} W_j^{(t)}\Big )  \Big(\prod_{j=\ttL}^{\ri+1} W_j^{(t)}\Big )^\top \in\R^{K\times K}.
\end{split}
\end{equation}
The compact representation $\hW^{(t)}$ obeys the following dynamics:
\begin{equation}
\label{eq:canon_deep_linear}
\hW^{(t+1)}=\hW^{(t)} - \mu \sum_{\ri=1}^L \Al_{\ri}^{(t)} \cdot G_r^{(t)} \cdot \Bl_{\ri-1}^{(t)} +O(\mu^2).
\end{equation}
\end{proposition}
The proof can be found in \appl\ref{app:deep-thm1}. Note that $\Bl_{0}^{(t)}=\Al_{\ttL}^{(t)}=I$, $\Bl_{\ttL}^{(t)}={\hW^{(t)}}^\top\hW^{(t)}$, and $\Al_{0}^{(t)}=\hW^{(t)}{\hW^{(t)}}^\top$.

\paragraph{Gradient scale matrices.} When the number of hidden layers is $0$ ($L=1$), both gradient scale matrices reduce to the identity matrix and the dynamics in (\ref{eq:canon_deep_linear}) is reduced to the known result that $\hW^{(t+1)}=\hW^{(t)} - \mu G_r^{(t)}$. Recall, however, that our focus is the over-parameterized linear model with $L>1$, in which the loss is not convex. Since the difference between the convex linear model and the over-parametrized deep model boils down to these matrices, our convergence analysis henceforth focuses on the dynamics of the \emph{gradient scale matrices}. In accordance, we analyze the evolution of the \emph{gradient scale matrices} as learning proceeds.


\begin{theorem}
\label{thm:AlsBls}
Let $\hW$ denote the compact representation of a deep linear network, where $\cm=\min\left({m_1,...,m_{L-1}}\right)$ denotes the size of its smallest hidden layer and $\cm\geq \{m_0,m_L\}$. At each layer $l$, assume weight initialization $W^{(0)}_l$ obtained by sampling from a distribution with mean $0$ and variance $\sigma_l^2$, normalized as specified in Def.~\ref{def:1}. Let $\left(\Bl^{(t)}_\ri(\cm)\right)_{\cm=1}^\infty$ and $\left( \Al^{(t)}_\ri(\cm)\right)_{\cm=1}^\infty$ denote two sequences of gradient scale matrices as defined in (\ref{eq:gradient-scale}), where the $\cm^{th}$ element of each series corresponds to a network whose smallest hidden layer has $\cm$ neurons.  Let $\xrightarrow{p}$ denote element-wise convergence in probability as $\cm\to\ \infty$. Then $\forall {t},l$:
\begin{equation}
\label{eq:basic-thm}
\begin{split}
&\Bl^{(t)}_\rl(\cm) \xrightarrow{p} I+O(\mu^2)~~\forall\rl\in[1\ldots \ttL-1],\qquad\Bl^{(t)}_{\ttL}(\cm) \xrightarrow{p} \frac{K}{\cm}I + O(\mu^2)\xrightarrow{p}O(\mu^2),\\
&\var[\Bl^{(t)}_\rl(\cm) ]= O(\mu^2)+\omverm \qquad\forall\rl,
\end{split}
\end{equation}
and
\begin{equation}
\label{eq:basic-thm-A}
\begin{split}
&\Al^{(t)}_\rl(\cm) \xrightarrow{p} I+O(\mu^2)~~\forall\rl\in[1\ldots \ttL-1],\qquad\Al^{(t)}_{0}(\cm) \xrightarrow{p} \frac{q}{\cm}I + O(\mu^2)\xrightarrow{p}O(\mu^2),\\
&\var[\Al^{(t)}_\rl(\cm) ]= O(\mu^2)+\omverm \qquad\forall\rl.
\end{split}
\end{equation}
\end{theorem}

\paragraph{Proof sketch.} The proof proceeds by induction on time $t$. To begin with, we recall that all the weight matrices $\{W^{(0)}_\rl\}_{\rl=1}^\ttL$ are initialized by sampling from a distribution with mean $0$ and variance $\sigma_l^2 = O(\frac{1}{\cm})$. In \appl\ref{sec:random-mat} we prove some statistical properties of such matrices, and their corresponding \emph{gradient scale matrices} $B_\ri^{(0)}$ and $A_\ri^{(0)}$. This analysis shows that (\ref{eq:basic-thm}) and (\ref{eq:basic-thm-A}) are true at $t=0$. To prove the induction step, we resort to results stated in \appl\ref{sec:random-mat-dyn}, where we analyze random matrices that are defined similarly to the gradient scale matrices, and whose dynamics is consistent with the update rule defined in (\ref{eq:canon_deep_linear}). The main result can be used to show directly that if the theorem's assertion holds for $B_{\ri}^{(t)}$, it also holds for $B_{\ri}^{(t+1)}$. A complete proof and details can be found in \appl\ref{app:sec:evolution}.

\paragraph{Relation to the one-layer linear model.} From \thm\ref{thm:AlsBls}, if the hidden layers are sufficiently wide (as measured by $\cm$), we can assume that $B_{\ri}^{(t)}(\cm) \approx I$ and $A_{\ri}^{(t)}(\cm) \approx I$ $\forall \ri$. In this case, the dynamics of the over-parameterized model is identical to the dynamics of the convex linear model, $\hW^{(t+1)}=\hW^{(t)} - \mu G_r^{(t)}$. This is shown more formally below in \thm\ref{thm:convergence-rate}.


\paragraph{Note about convergence, and convergence rate.} 
In \thm\ref{thm:AlsBls} we prove two  asymptotic results: $\Bl^{(t)}_\rl(\cm) \xrightarrow{p} I$ and $\Al^{(t)}_\rl(\cm) \xrightarrow{p} I$ up to small deviations of magnitude $O(\mu^2)$. This is true $\forall \rl\in[1\ldots \ttL-1]$ and $\forall t$, irrespective of the convergence of the network. The only thing that changes with $t$ is the magnitude of the constants governing this asymptotic behavior, hidden in the notations $\omverm$ and $O(\mu^2)$, whose magnitude increases with time. For a fixed network, there will come a time when the asymptotic analysis is longer applicable, and this may happen before or after the network has converged to its final solution.

Importantly, these constants are significantly different in the two sequences. Specifically, in (\ref{eq:basic-thm}) and (\ref{eq:basic-thm-A}) we see that the convergence of $\Bl^{(t)}_\ttL(\cm)$ is governed by $\omvermK$, while the convergence of $\Al^{(0)}_\rl(\cm)$ is governed by $\omvermq$. Typically $q\gg \mL$, as $q$ denotes the dimension of the data space which can be fairly large, while $K$ denotes the typically small number of classes. Considered in the context of the proof of \thm\ref{thm:AlsBls} by induction, we note that these constants are amplified in each iteration $t$. As a result, when $\cm$ is large but fixed, $\Al^{(t)}_\rl(\cm)$ is expected to deviate from $I$ sooner than $\Bl^{(t)}_\rl(\cm)$ as the GD steps are reiterated. 

\paragraph{Empirical validation.}
To understand the practical significance of the observations stated above, about the difference in convergence rate between the sequences $\Bl^{(t)}_\rl(\cm)$ and $\Al^{(t)}_\rl(\cm)$, we resort to simulations whose results are shown in \fig\ref{fig:gradient_scale_validation}. These empirical results, recounting linear networks with 4 hidden layers of width 1024, clearly show that during a significant part of the training both \emph{gradient scale matrices} remain approximately $I$. The difference between the convergence rate of $B_{\ri}^{(t)}$ and $A_{\ri}^{(t)}$ is seen later on, when $A_{\ri}^{(t)}$ starts to deviate from $I$ shortly before the networks have reached their maximal test accuracy, while $B_{\ri}^{(t)}$ remains essentially the same throughout.

\begin{figure}[thb!]
\begin{center}
  
    \begin{subfigure}{.4\textwidth}
      \centering
      \includegraphics[width=1\linewidth]{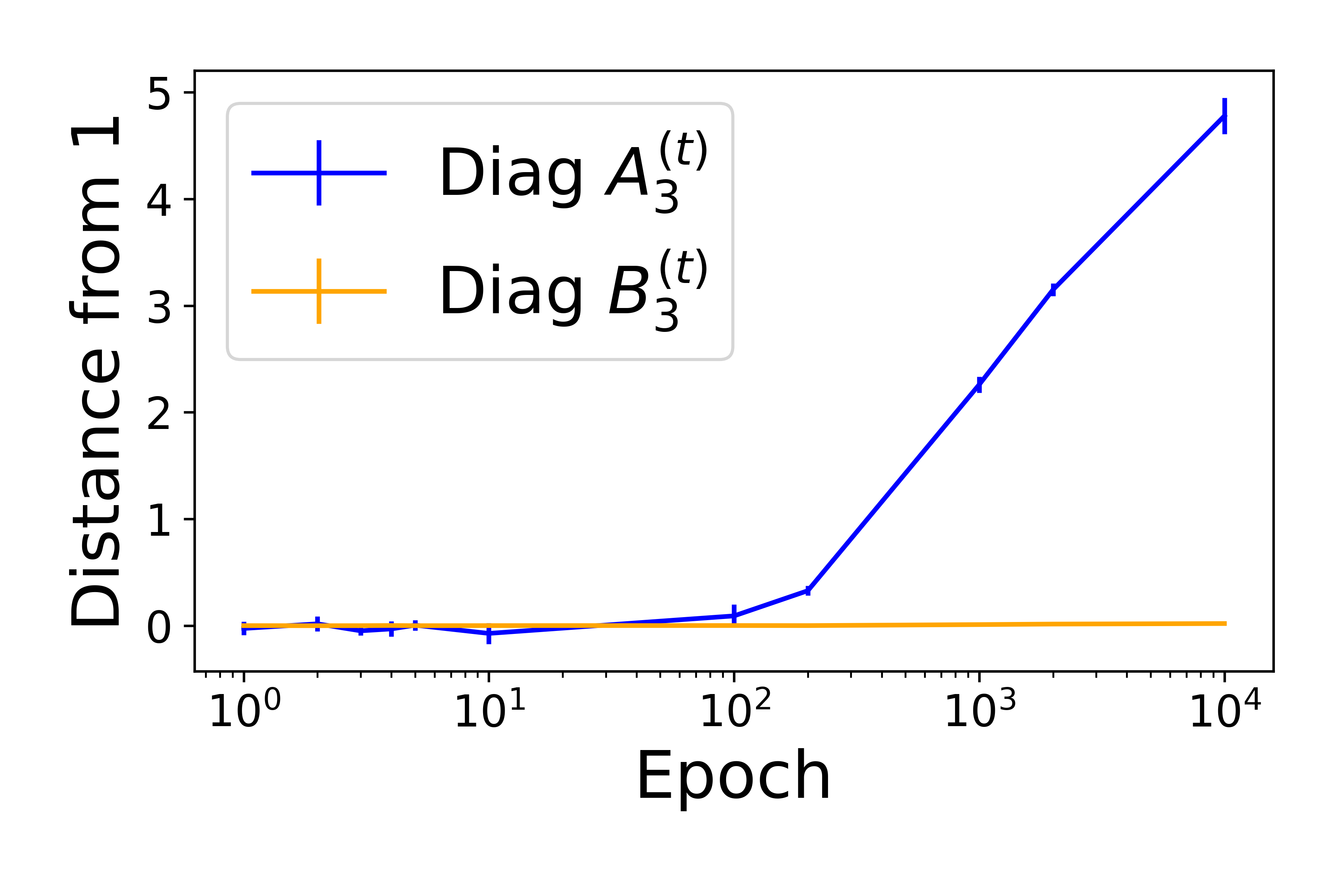}
     \caption{Diagonal elements, layer 3}
    \end{subfigure}
    \hspace{1.5cm}
    \begin{subfigure}{.4\textwidth}
      \centering
      \includegraphics[width=1\linewidth]{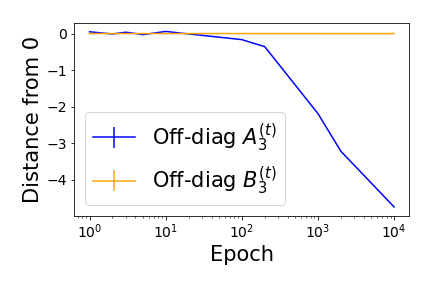}
     \caption{Off-diagonal elements, layer 3}
    \end{subfigure}

\end{center}

\caption{The dynamics of $B_3^{(t)}$ and $A^{(t)}_3$ when training $10$ $5$-layered linear networks on the small mammals dataset (see \app\ref{app:datasets}). (a) The empirical $L_2$-distance of the diagonal elements of $B_3^{(t)}$ and $A^{(t)}_3$ from\protect\footnotemark $~1$.  (b) The empirical $L_2$-distance of the off-diagonal elements of $B_3^{(t)}$ and $A^{(t)}_3$ from 0. The networks reach maximal test accuracy in epoch $t=100$, before the divergence of $A^{(t)}_3$. In these experiments $B_3^{(t)}$ never reaches the point of divergence. These result are typical of all the layers in the networks, see Fig.~\ref{fig:app:gradient_scale_validation} in \appl\ref{app:divergence}.}
\label{fig:gradient_scale_validation}
\end{figure}

\footnotetext{More precisely, the distance is computed respectively to their analytical values of $\alpha_3^{(t)}$ and $\beta_3^{(t)}$, computed as $diag\left(A^{(t)}_3-\alpha_3^{(t)}I\right)$ and  $diag\left(B_3^{(t)} - \beta_3^{(t)}I\right)$, where $\alpha_i^{(t)}, \beta_i^{(t)}$ are defined in \thm\ref{thm:qtq}, \S\ref{sec:random-mat}.}

\subsection{Weight Evolution}
\label{sec:evolution}

Next, we investigate the implications of \thm\ref{thm:AlsBls} regarding the evolution of the compact representation $\hW$. In \thm\ref{thm:convergence-rate}, we show that at the beginning of training, when the assumption that both $B^{(t)}_l\approx I$ and $A^{(t)}_l\approx I$ is applicable, this evolution resembles the single-layer linear model and is governed by the eigendecomposition of the data. In \thm\ref{thm:teles}, we show that later on, when only $B^{(t)}_l\approx I$ is applicable and the evolution no longer resembles the single-layer linear model, it is still governed by the eigendecomposition of the data.

More specifically, let $\bw_j^{opt}$ denote the $j^\mathrm{th}$ column of the optimal solution of (\ref{eq:multi-loss}), and $\bw_j^{(0)}$ the $j^\mathrm{th}$ column of the initial weight matrix. \thm\ref{thm:convergence-rate} states that if the hidden layers of the network are all wider than a certain fixed width $\hat\cm$ and if the learning rate is slower than $\hat\mu$, then at the beginning of learning (time $t\in[0\ldots\hat {t}]$), the $j^\mathrm{th}$ column of the compact representation $\hW^{(t)}$ is roughly the following linear combination of $\bw_j^{(0)}$ and $\bw_j^{opt}$:
\begin{equation}
\label{eq:linear-comb}
\bw_j^{(t+1)}\approx \lambda_j^{t} \bw_j^{(0)} + [1- \lambda_j^{t}]\bw_j^{opt}, \quad\quad \lambda_j = 1-\mu d_j L.
\end{equation}
with high probability (at least $(1-\delta)$) and low error (at most $\varepsilon$). Constant $\lambda_j$ depends on the $j^\mathrm{th}$ singular value of the data $d_j$ and the number of layers $L$. 

Result (\ref{eq:linear-comb}) is reminiscent of the well-understood dynamics of training the convex one-layer linear model. It is composed of two additive terms, revealing two parallel and separate processes:
\begin{inparaenum}[(i)]
\item The dependence on random initialization tends to 0 exponentially fast as a function of time ${t}$, when $\lambda_j^{t}$ tends to $0$. 
\item The optimal solution is reached exponentially fast as a function of time ${t}$, when $1-\lambda_j^{t}$ tends to $1$. 
\end{inparaenum}
In either case, convergence is fastest for the largest singular eigenvalue, or the first column of $\hW$, and slowest for the smallest singular value. This behavior is visualized in Fig.~\ref{fig:dw_linear_nets}, which shows that the decrease in the variance of weight estimation between networks is indeed faster for larger singular values. 

\begin{figure}[ht!]
\begin{minipage}{.6\textwidth}
    \vspace{.32cm}
    \begin{subfigure}{.49\textwidth}
      \centering
      \includegraphics[width=\linewidth]{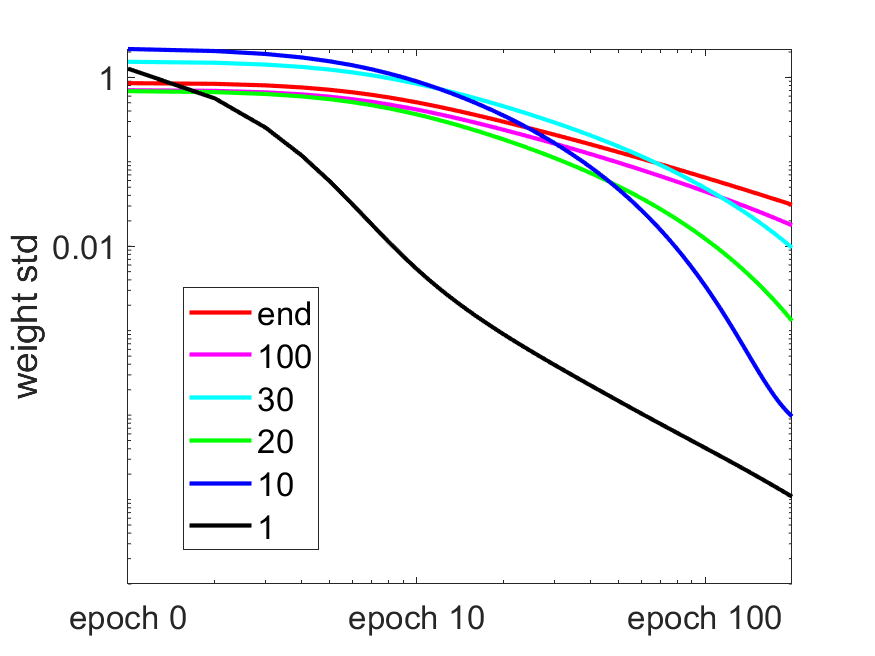}
      \caption{2-layer linear network}
      \label{fig:dw_linear_nets}
    \end{subfigure}
    \begin{subfigure}{.49\textwidth}
      \centering
      \includegraphics[width=1\linewidth]{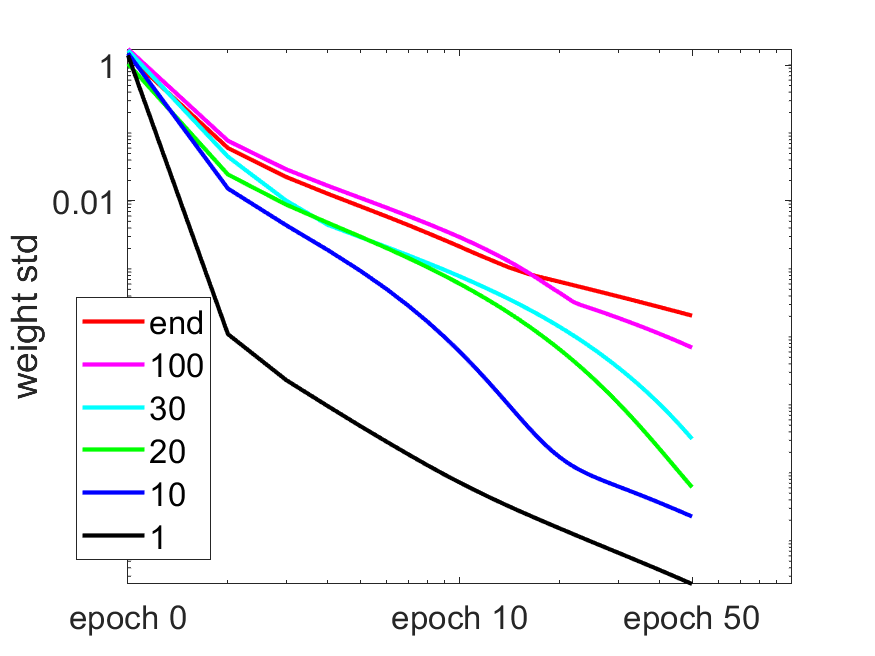}
      \caption{2-layer ReLU network}
      \label{fig:dw_relu_nets}
    \end{subfigure}
\caption{Empirical evaluation of the dependence of convergence rate on the eigendecomposition in a \emph{binary} classification problem. Here the classifier is a vector, denoted $\bw$. Each line corresponds to a specific principal eigenvector, plotting in log-log scale the std of element $w_i$ ($i=1,10,20,30,100,3072$, see legend) over 10 independently trained networks.}
\label{fig:std-relu}
\end{minipage}
\hspace{0.3cm}
\begin{minipage}{.39\textwidth}
    \includegraphics[width=1\textwidth]{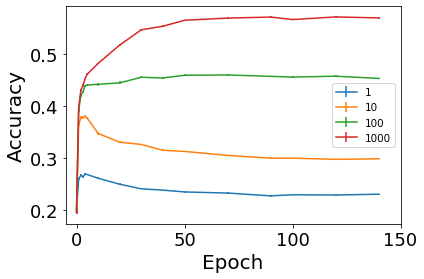}
    \caption{Mean accuracy of 10 st-VGG networks evaluated on test data projected using PCA to $\{1,10,100,1000\}$ dimensions (see text for more details).}  
    \label{fig:acc_on_pca_components}
\end{minipage}
\end{figure}

Formally, the theorem can be stated as follows (see proof in \appl\ref{app:sec:evolution}):
\begin{theorem}
\label{thm:convergence-rate}
Let $\bw_j^{(t)}$ denote the $j^\mathrm{th}$ column of the compact representation matrix $\hW^{(t)}$, and $\bw_j^{opt}$ the $j^\mathrm{th}$ column of the optimal solution of (\ref{eq:multi-loss}). Assume that the data is rotated to its principal coordinate system, and let $d_j$ denote the $j^\mathrm{th}$ singular value of the data. Then there exists $\hat {t}$ such that $\forall \delta, \varepsilon$ and $\forall {t}\leq\hat {t}$, $\exists \hat \cm,\hat\mu$ such that $\forall \mu<\hat\mu, \cm\geq\hat\cm$
\begin{equation*}
Prob\bigg (\Big\Vert\bw_j^{(t+1)}- \big [ \lambda_j^{t} \bw_j^{(0)} + [1- \lambda_j^{t}]\bw_j^{opt}\big ]\Big\Vert < \varepsilon\bigg ) > (1-\delta), \quad\quad \lambda_j = 1-\mu d_j L.
\end{equation*}
\end{theorem}

\paragraph{Proof sketch.} We first prove by induction on time $t$ that $\forall t$, $\hW^{(t+1)} = \hW^{(t)} - \mu LG_r^{(t)}$ with high probability and small error. We then shift to the principal coordinate system, and note that this is true separately for each column of matrix $\hW$. The evolution of each column is then shown to be a telescopic series, whose solution is $\bw_j^{(t)}= \lambda_j^{t} \bw_j^{(0)} + [1- \lambda_j^{t}]\bw_j^{opt}$ for $\lambda_j = 1-\mu L d_j$, with high probability and small error.

Next we state \thm\ref{thm:teles}, which remains valid for longer (see footnote~\ref{ft:13} below) as it does not depend on the convergence of $\Al^{(t)}_\rl$. This is the case discussed in \thm\ref{thm:AlsBls} and illustrated in \fig\ref{fig:gradient_scale_validation}, when $\Bl^{(t)}_\rl\approx I$ remains approximately true while $\Al^{(t)}_\rl\neq I$. More specifically, let $A^{(t)}=\sum_{\ri=1}^L A_{\ri}^{(t)}$. \thm\ref{thm:teles} asserts that if the hidden layers of the network are all wider than $\breve\cm$ and if the learning rate is slower than $\breve\mu$, then at time $t\in[0\ldots\breve{t}]$,\footnote{\label{ft:13}Presumably, in comparison to \thm\ref{thm:convergence-rate} and if we fix the network's width $\breve m=\hat m$, then $\breve{t}\gg\hat {t}$.} the $j^\mathrm{th}$ column of the compact representation $\hW^{(t)}$ is approximately the following
\begin{equation*}
\bw_j^{(t+1)} \approx \prod_{\nunu=1}^{{t}}(I-\mu d_j A^{(\nunu)})\bw_j^{(0)}
+ \mu d_j\left [\sum_{\nunu=1}^{{t}}  \prod_{\rhorho=\nunu+1}^{{t}} (I-\mu d_j A^{(\rhorho )})A^{(\nunu)}\right ] \bw_j^{opt},
\end{equation*}
with high probability (at least $(1-\delta)$) and low error (at most $\varepsilon$).  Note that the $j^\mathrm{th}$ column of the compact representation $\hW^{(t)}$ is still a linear combination of $\bw_j^{(0)}$ and $\bw_j^{opt}$, but now the coefficients are much more complex and depend on $A_{\ri}^{(t)}$.

Formally, the theorem can be stated as follows (see proof in \appl\ref{app:sec:evolution}):
\begin{theorem}
\label{thm:teles}
Let $\bw_j^{(t)}$ denote the $j^\mathrm{th}$ column of the compact representation matrix $\hW^{(t)}$, and $\bw_j^{opt}$ the $j^\mathrm{th}$ column of the optimal solution of (\ref{eq:multi-loss}). Assume that the data is rotated to its principal coordinate system, and let $d_j$ denote the $j^\mathrm{th}$ singular value of the data. Then there exists $\breve{t}$ such that $\forall \delta, \varepsilon$ and $\forall {t}\leq\breve{t}$, $\exists \breve \cm,\breve\mu$ such that $\forall \mu<\breve\mu, \cm\geq\breve\cm$
\begin{equation*}
\begin{split}
Prob\Bigg (\bigg\Vert \bw_j^{(t+1)} -&\bigg [ \prod_{\nunu=1}^{{t}}(I-\mu d_j A^{(\nunu)})\bw_j^{(0)}
+ \\
&\mu d_j\Big (\sum_{\nunu=1}^{{t}}  \prod_{\rhorho=\nunu+1}^{{t}} (I-\mu d_j A^{(\rhorho )})A^{(\nunu)}\Big ) \bw_j^{opt} \bigg ]\bigg\Vert < \varepsilon\Bigg ) > (1-\delta).
\end{split}
\end{equation*}
\end{theorem}

Although the dynamics now depend on matrices $A^{(t)}_\rl$ as well, it is still the case that the convergence of each column is governed by its singular value $d_j$. This suggests that while the \emph{PC-bias} is more pronounced at earlier stages of learning, its effect persists throughout.

\subsection{Adding ReLU Activation}
\label{sec:relu}

We extend the results above to a relatively simple non-linear model suggested and analyzed by \citet{DBLP:conf/nips/Allen-ZhuLL19,arora2019fine,ronen2019convergence}, here trained on a classification task rather than a regression task. Specifically, it is a two-layer model with ReLU activation, where only the weights of the first layer are being learned. Similarly to (\ref{eq:multi-loss}), the optimization problem is defined as
\begin{equation*}
W^* =  \argmin\limits_W\frac{1}{2}\sum_{i=1}^n \Vert f(\bx_i)-\by_i \Vert^2, \qquad f(\bx_i) = \ba^\intercal \cdot \sigma(W \bx_i), ~~ \ba\in\R^{\cm},~ W\in\R^{\cm\times d},
\end{equation*}
where $\cm$ denotes the number of neurons in the hidden layer and $\ba$ a fixed vector. We consider a binary classification problem with 2 classes, where $y_i=1$ for $\bx_i\in C_1$, and $y_i=-1$ for $\bx_i\in C_2$. $\sigma(\cdot)$ denotes the ReLU activation function $\sigma(u) = \max(u,0)$. Unlike deep linear networks, the analysis in this case requires two additional symmetry assumptions:
\begin{enumerate}
    \item Each point $\bx_i$ is drawn from a symmetric distribution $\mathcal{D}$ with density $f_\mathcal{D}(\bX)$, such that: $f_\mathcal{D}(\bx_i)=f_\mathcal{D}(-\bx_i)$.
    \item $W$ and $\ba$ are initialized symmetrically so that $\bw^{(0)}_{{\textstyle\mathstrut}2i}\hspace{-3pt}=\hspace{-2pt} - \bw^{(0)}_{{\textstyle\mathstrut}2i-1}$ and $a_{{\textstyle\mathstrut}2i} \hspace{-2pt}=\hspace{-2pt} - a_{{\textstyle\mathstrut}2i-1}$ $\forall i\in[\frac{\cm}{2}]$.
\end{enumerate}

\begin{theorem}
\label{thm:dynamics_relu}
Retaining the assumptions stated above, at the beginning of the learning, the temporal dynamics of the model can be shown to obey the following update rule:
\begin{equation*}
W^{(t+1)} \approx W^{(t)}-\mu \frac{1}{2}\Big [ (\ba\ba^\intercal) W^{(t)}\SXX - \tilde M^{(t)} \Big ].
\end{equation*}
Above $\tilde M^{(t)}$ denotes the difference between the centroids of the 2 classes, computed in the half-space defined by $\bw^{(t)}_r\cdot\bx\geq 0$.
\end{theorem}
The complete proof can be found in \app\ref{app:relu}. \thm\ref{thm:dynamics_relu} is reminiscent of the single-layer linear model dynamics $\hW^{(t+1)}=\hW^{(t)} - \mu G_r^{(t)}$, and we may conclude that when it holds, using the principal coordinate system, the rate of convergence of the $j$-th column of $W^{(t)}$ is governed by the singular value $d_j$. \fig\ref{fig:dw_relu_nets} empirically demonstrated this result, showing the weight convergence of $10$ ReLU models trained on the small mammals dataset (5 classes out of CIFAR-100, see \app\ref{app:datasets} for details), along different principal directions of the data.

\section{PC-Bias: Empirical Study}
\label{sec:results}

In this section, we first analyze deep linear networks, showing that the convergence rate is indeed governed by the principal singular values of the data, which demonstrates the plausibility of the assumptions made in Section~\ref{sec:theo-res}. We continue by extending the scope of the investigation to non-linear neural networks, finding evidence for the \emph{PC-bias} mostly in the earlier stages of learning. 

\subsection{Methodology}
\label{sec:methods}

We say that a linear network is $L$-layered when it has $L-1$ hidden fully connected (FC) layers (without convolutional layers). In our empirical study, we relaxed some assumptions of the theoretical analysis, to increase the resemblance of the trained networks to networks in common use. Specifically, we changed the initialization to the commonly used Glorot initialization, replaced the $L_2$ loss with the cross-entropy loss, and employed SGD instead of the deterministic GD. As to be expected, the original assumptions yielded similar results (see \S\ref{app:sec:stds_square_loss}). The results presented summarize experiments with networks of equal width across all hidden layers, fixing the moderate\footnote{Note that for most image datasets, $q$ (the input dimension) is much larger than $1024$, and therefore $\frac{q}{\cm}>>1$. In this sense $\cm=1024$ is moderate, as one can no longer assume $A_\ri^{(t)}(\cm)\approx I$.} value of $\cm=1024$ in order to assess the relevance of our asymptotic results obtained with $\cm\to\infty$. Using a different width for each layer yielded similar qualitative results. Details regarding the hyper-parameters, architectures, and datasets can be found in \S\ref{app:sec:hyper_params}, \S\ref{app:architectures}, and \S\ref{app:datasets} respectively.

\subsection{PC-bias in Deep Linear Networks}
\label{sec:validation}

In this section, we train $L$-layered linear networks, then compute their compact representations $\hW$ rotated to align with the canonical coordinate system (Def.~\ref{def:canon}). Note that each row $\bw(r)$ in $\hW$ essentially defines the one-vs-all separating hyper-plane corresponding to class $r$.

To examine both the variability between models and their convergence rate, we inspect $\bw(r)$ at different time points during learning. The rate of convergence can be measured directly, by observing the changes in the weights of each element in $\bw(r)$. These weight values\footnote{We note that the weights tend to start larger for smaller principal components, see \fig\ref{fig:stds_cats_and_dogs_2_matrices} left.} are compared with the corresponding optimal weights $w_{opt}$.
The variability between models is measured by computing the standard deviation (std) of each row vector $\bw(r)$ across $N$ models, obtained with different random initializations.

We begin with linear networks. We trained $10$ $5$-layered FC linear networks, and $10$ linear st-VGG convolutional networks. When analyzing the compact representation of such networks we observe a similar behavior---weights corresponding to larger principal components converge faster to the optimal value, and their variability across models converges faster to 0 (Figs.~\ref{fig:stds_cats_and_dogs_2_matrices},\ref{fig:stds_linearnet_cats_and_dogs}). Thus, while the theoretical results are asymptotic, the \emph{PC-bias} is empirically seen throughout the entire learning process of deep linear networks. 

We note that in the principal coordinate system, the absolute values of the optimal solution are often higher in directions corresponding to lower principal components. Nevertheless, faster convergence in the directions corresponding to the higher principal components is seen even after this confounding factor is eliminated by normalization (see \app\ref{app:sec:distance_to_convergence}).

\begin{figure}[h!]
\begin{subfigure}{.5\textwidth}
\begin{subfigure}{.32\textwidth}
      \centering
     \includegraphics[width=1\linewidth]{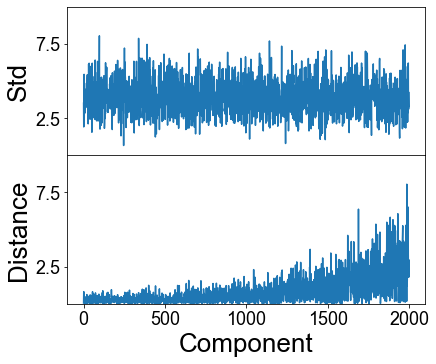}
    \end{subfigure}
    \begin{subfigure}{.32\textwidth}
      \centering
      \includegraphics[width=1\linewidth]{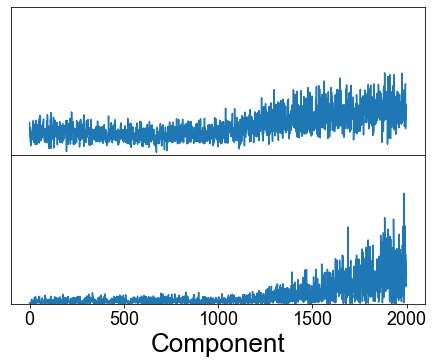}
    \end{subfigure}
    \begin{subfigure}{.32\textwidth}
      \centering
      \includegraphics[width=1\linewidth]{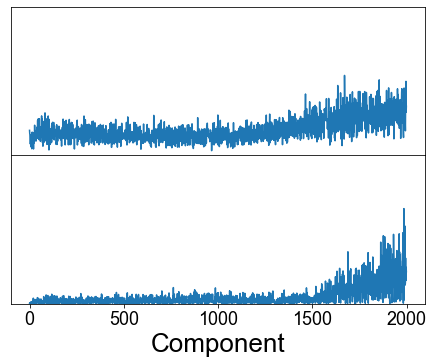}
    \end{subfigure}
    \caption{5-Layered linear network}
    \label{fig:stds_cats_and_dogs_2_matrices}
\end{subfigure}
\begin{subfigure}{.5\textwidth}
   \begin{subfigure}{.32\textwidth}
      \centering
      \includegraphics[width=1\linewidth]{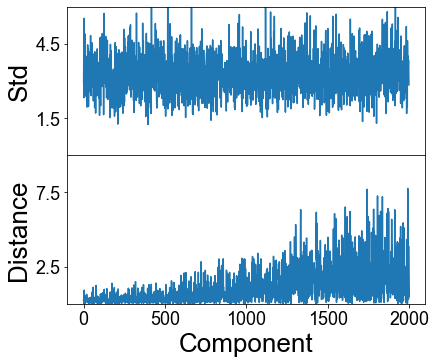}
    \end{subfigure}
    \begin{subfigure}{.32\textwidth}
      \centering
      \includegraphics[width=1\linewidth]{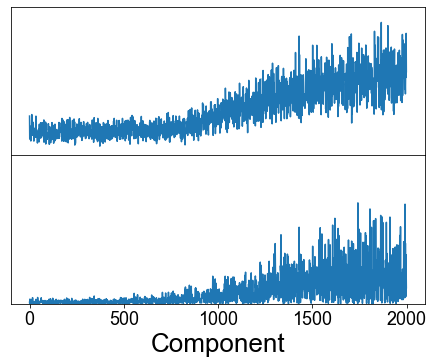}
    \end{subfigure}
    \begin{subfigure}{.32\textwidth}
      \centering
      \includegraphics[width=1\linewidth]{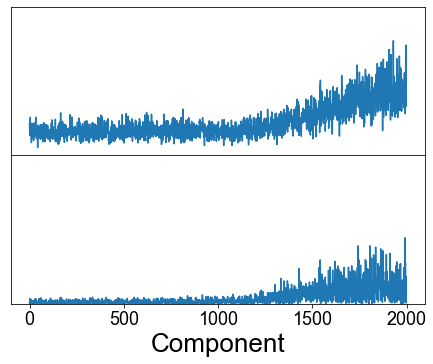}
    \end{subfigure}
    \caption{Linear convolutional network}
    \label{fig:stds_linearnet_cats_and_dogs}
\end{subfigure}

\begin{subfigure}{.5\textwidth}
\begin{subfigure}{.32\textwidth}
      \centering
      \includegraphics[width=1\linewidth]{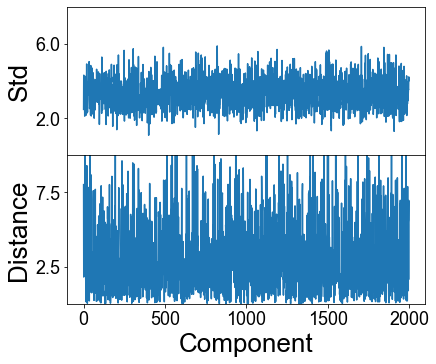}
    \end{subfigure}
    \begin{subfigure}{.32\textwidth}
      \centering
      \includegraphics[width=1\linewidth]{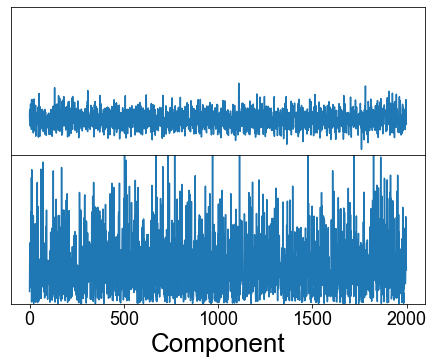}
    \end{subfigure}
    \begin{subfigure}{.32\textwidth}
      \centering
      \includegraphics[width=1\linewidth]{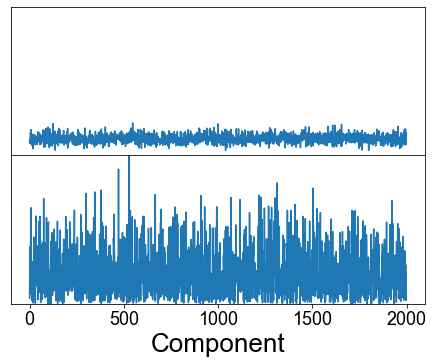}
    \end{subfigure}
\caption{5-Layered linear network, whitened data}
\label{fig:stds_cats_dogs_1_matrix_whitening}
\end{subfigure}
\begin{subfigure}{.5\textwidth}
\begin{subfigure}{.32\textwidth}
      \centering
      \includegraphics[width=1\linewidth]{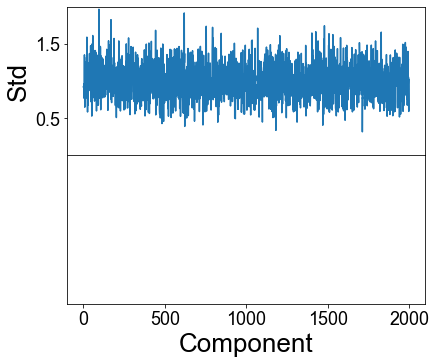}
    \end{subfigure}
    \begin{subfigure}{.32\textwidth}
      \centering
      \includegraphics[width=1\linewidth]{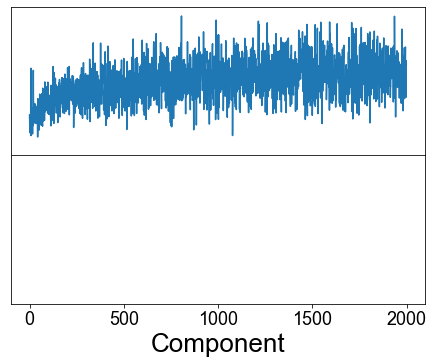}
    \end{subfigure}
    \begin{subfigure}{.32\textwidth}
      \centering
      \includegraphics[width=1\linewidth]{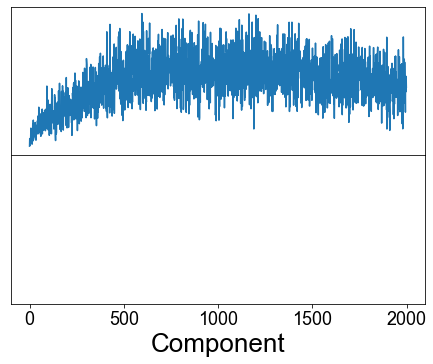}
    \end{subfigure}
\caption{Non-linear convolutional network}
\label{fig:stds_cats_and_dogs_CNN}
\end{subfigure}
\caption{Convergence of the compact representation along the principal directions in different epochs. The value of the $X$-axis corresponds to the index of a principal eigenvalue, from the most significant to the least significant. (a) $10$ $5$-layered linear networks trained on the cats and dogs dataset. 3 plots are provided, corresponding to snapshots taken at different stages of learning: the beginning (epoch $1$, left), intermediate stage (epoch $100$, middle), and close to convergence (epoch $1500$ right). Bottom panel: average distance of the weights in $\bw(1)$ from the optimal linear classifier; top panel: respective std. (b) Similarly, for $10$ linear st-VGG convolutional networks, trained on CIFAR-10 (epochs plotted: $1,100,500$). (c) Similarly, for $10$ $5$-layered linear networks, trained on the cats and dogs dataset, with ZCA-whitening (epochs plotted: $1,2,1000$). (d) Similarly, for $10$ \textbf{non-linear} st-VGG networks trained on the cats and dogs dataset. Here the distance to the optimal solution is not well defined and we therefore only show the std (epochs plotted: $1,10,100$). In all plots, the data dimension is $3072$, while only the largest $2000$ principal components are shown for clarity. Visualization of the values of $\lambda_i$ can be found in Fig.~\ref{app:fig:lambda_j_figure}.}
\label{fig:stds_and_distances_all_cases}
\end{figure}

\paragraph{Whitened data.}
The \emph{PC-bias} is neutralized when the data is whitened, at which point $\SXX$ is the scaled identity matrix. In Fig.~\ref{fig:stds_cats_dogs_1_matrix_whitening}, we plot the results of the same experimental protocol while using a ZCA-whitened dataset. As predicted, the networks no longer show any bias towards any principal direction. Weights in all directions are scaled similarly, and the std over all models is the same in each epoch, irrespective of the principal direction. (Additional experiments show that this is \emph{not} an artifact, due to the lack of uniqueness when deriving the principal eigenvectors of a white signal).

\subsection{PC-Bias in General CNNs}
\label{sec:pc_bias_non_linear_networks}

In this section, we investigate the manifestation of the \emph{PC-bias} in non-linear deep convolutional networks. As we cannot directly track the learning dynamics separately in each principal direction of non-linear networks, we adopt two different evaluation mechanisms:

\paragraph{(i) Linear approximation.}
We considered several linear approximations, but since all of them showed the same qualitative behavior, we report results with the simplest one. Specifically, to obtain a linear approximation of a non-linear network, we ignore the non-linear layers of max-pooling and batch-normalization, and compute the compact representation as in Section~\ref{sec:theo-res} using only linear activations. We then align this matrix with the canonical coordinate system (Def.~\ref{def:canon}), and observe the evolution of the weights and their std across models along the principal directions during learning. Note that now the networks do not converge to the same compact representation, which is not unique. Nevertheless, we see that the \emph{PC-bias} governs the weight dynamics to a noticeable extent.

We note that, in these networks, a large fraction of the lowest principal components hardly ever changes during learning. Nevertheless, the \emph{PC-bias} affects the higher principal components, most notably at the beginning of training (see \fig\ref{fig:stds_cats_and_dogs_CNN}). Thus weights corresponding to higher principal components converge faster, and the std across models of such weights decreases faster for higher principal components.


\paragraph{(ii) Projection to higher PC's.}
We created a modified \emph{test-set}, by projecting each test example on the span of the first $P$ principal components. This is equivalent to reducing the dimensionality of the test set to $P$ using PCA. We trained an ensemble of $N$=$100$ st-VGG networks on the original training set of the small mammals dataset (see \app\ref{app:datasets}), then evaluated these networks during training on 4 versions of the test-set, reduced to $P$=$1$,$10$,$100$,$1000$ dimensions respectively. Mean accuracy is plotted in \fig\ref{fig:acc_on_pca_components}. Similar results are obtained when training VGG-19 networks on CIFAR-10, see \S\ref{app:sec:project_to_higher_pc}.

Taking a closer look at \fig\ref{fig:acc_on_pca_components}, we see that when evaluated on lower dimensionality test-data ($P$=$1$,$10$), the networks' accuracy peaks after a few epochs, at which point performance starts to decrease. This result suggests that the networks rely more heavily on these dimensions in the earlier phases of learning, and then continue to learn other things. In contrast, when evaluated on higher dimensionality test-data ($P$=$100$,$1000$), accuracy continues to rise, longer so for larger $P$. This suggests that significant learning of the additional dimensions continues in later stages of the learning.

\section{PC-Bias: Learning Order Constancy}
\label{sec:pc_bias_explain_loc}

In this section, we show that the \emph{PC-bias} is significantly correlated with the learning order of deep neural networks, and can therefore partially account for the \emph{LOC-effect} described in Section~\ref{sec:intro}. Following \citet{hacohen2020let}, we measure the "speed of learning" of each example by computing its \emph{accessibility} score. This score is computed per example, and characterizes how fast an ensemble of $N$ networks memorizes it. Formally, $accessibility(\bx)=\E\left[\mathbbm{1}(f_i^e(\bx)=y(\bx)\right)]$, where $f_i^e(\bx)$ denotes the outcome of the $i$-th network trained over $e$ epochs, and the mean is taken over networks and epochs. For the set of datapoints $\{(\bx_j, \by_j)\}_{j=1}^n$, \emph{Learning Order Constancy} is manifested by the high correlation between 2 instances of $accessibility(\bx)$, each computed from a different ensemble.

\emph{PC-bias} is shown to pertain to \emph{LOC} in two ways: First, in Section~\ref{sec:results-1} we show a high correlation between the learning order in deep linear and non-linear networks. Since the \emph{PC-bias} fully accounts for \emph{LOC} in deep linear networks, this suggests that it also  accounts (at least partially) for the observed \emph{LOC} in non-linear networks. Comparison with the \emph{critical principal component} verifies this assertion. Second, we show in Section~\ref{sec:whitened} that when the \emph{PC-bias} is neutralized, \emph{LOC} diminishes as well. In Section~\ref{sec:empricial-freq-bias} we discuss the relationship between the \emph{spectral bias}, \emph{PC-bias} and the \emph{LOC-effect}.

\subsection{PC-Bias is Correlated with LOC}
\label{sec:results-1}

We first compare the order of learning of non-linear models and deep linear networks by computing the correlation between the \emph{accessibility} scores of both models. This comparison reveals high correlation ($r=0.85$, $p<10^{-45}$), as seen in \fig\ref{fig:accessibility-slow-lr-and-linear-small-mammals}. To investigate directly the connection between the \emph{PC-bias} and \emph{LOC}, we define the \emph{critical principal component} of an example to be the first principal component $P$, such that a linear classifier trained on the original data can classify the example correctly when projected to $P$ principal components. We trained $N$=$100$ st-VGG networks on the cats and dogs dataset (2 classes out of CIFAR-10, see \S\ref{app:datasets} for details), and computed for each example its \emph{accessibility} score and its \emph{critical principal component}. In \fig\ref{fig:accessibility_critical_component_correlation} we see strong negative correlation between the two scores ($p$=$-0.93$, $r\!<\!10^{-4}$), suggesting that the \emph{PC-bias} affects the order of learning as measured by \emph{accessibility}. 

\begin{figure}[ht]
\begin{minipage}{.5\textwidth}
    \begin{subfigure}{.49\textwidth}
      \centering
      \includegraphics[width=1\linewidth]{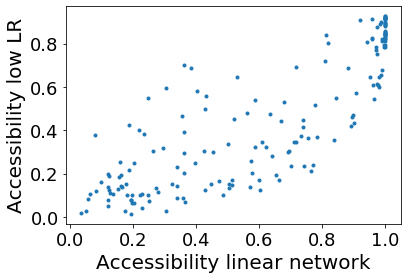}
      \caption{}
      \label{fig:accessibility-slow-lr-and-linear-small-mammals}
    \end{subfigure}
    \begin{subfigure}{.49\textwidth}
     \centering
      \includegraphics[width=1\linewidth]{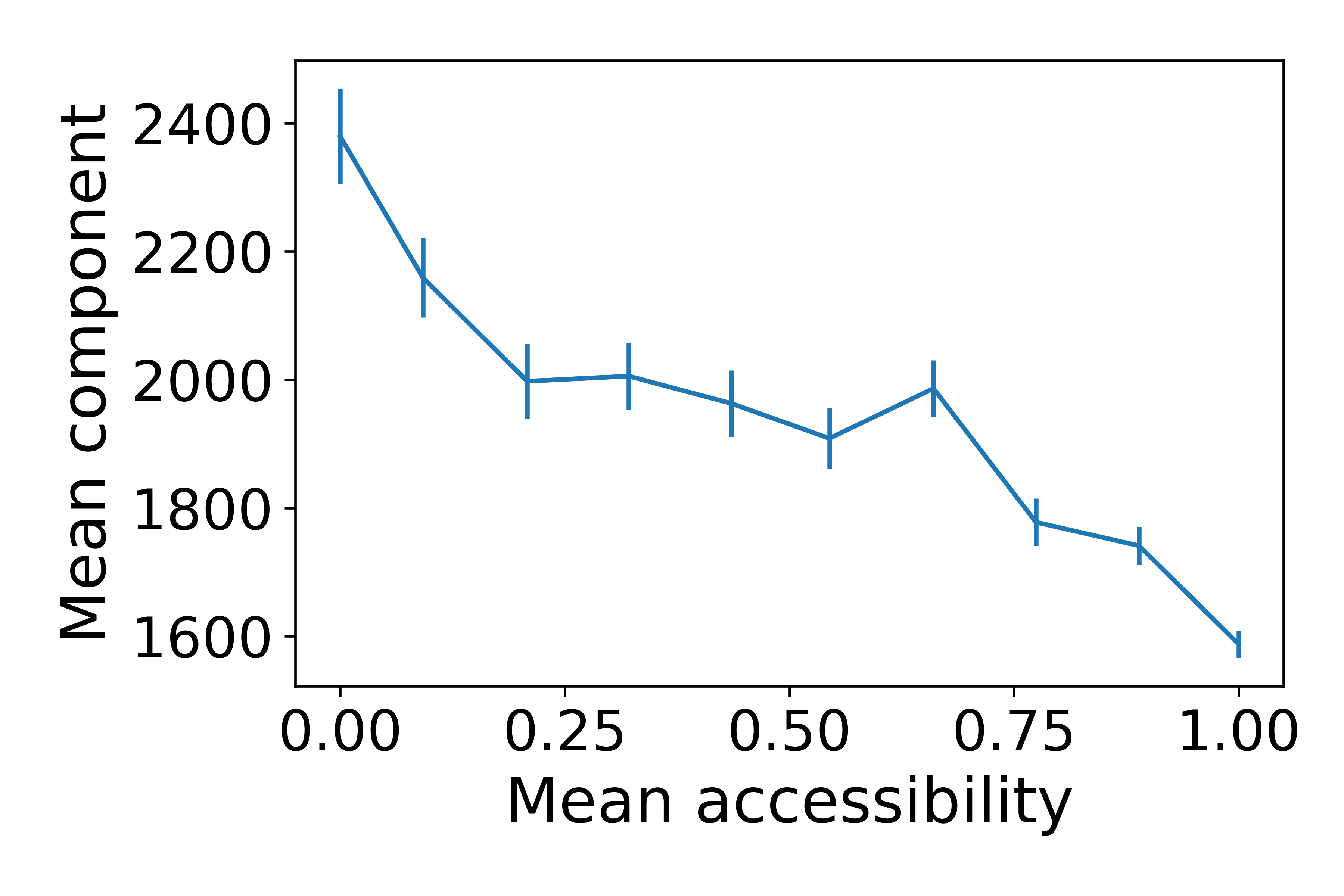}
      \caption{}
      \label{fig:accessibility_critical_component_correlation}
    \end{subfigure}
    \caption{(a) Correlation between the \emph{accessibility} score of $N$=$100$ st-VGG networks trained with a low learning rate\protect\footnotemark, and $N$=$100$ linear st-VGG networks, trained on small mammals (see \S\ref{app:architectures}, \S\ref{app:datasets} for details). (b) The \emph{critical principal component} score, plotted against the accessibility score of $N$=$100$ st-VGG networks trained on the cats and dogs dataset. \emph{Accessibility} values were smoothed by a moving average filter of width 10. Error bars indicate standard error.}
\end{minipage}
\hspace{0.2cm}
\begin{minipage}{.5\textwidth}
    \begin{subfigure}{1\textwidth}
        \begin{subfigure}{.49\textwidth}
          \centering
          \includegraphics[width=1\linewidth]{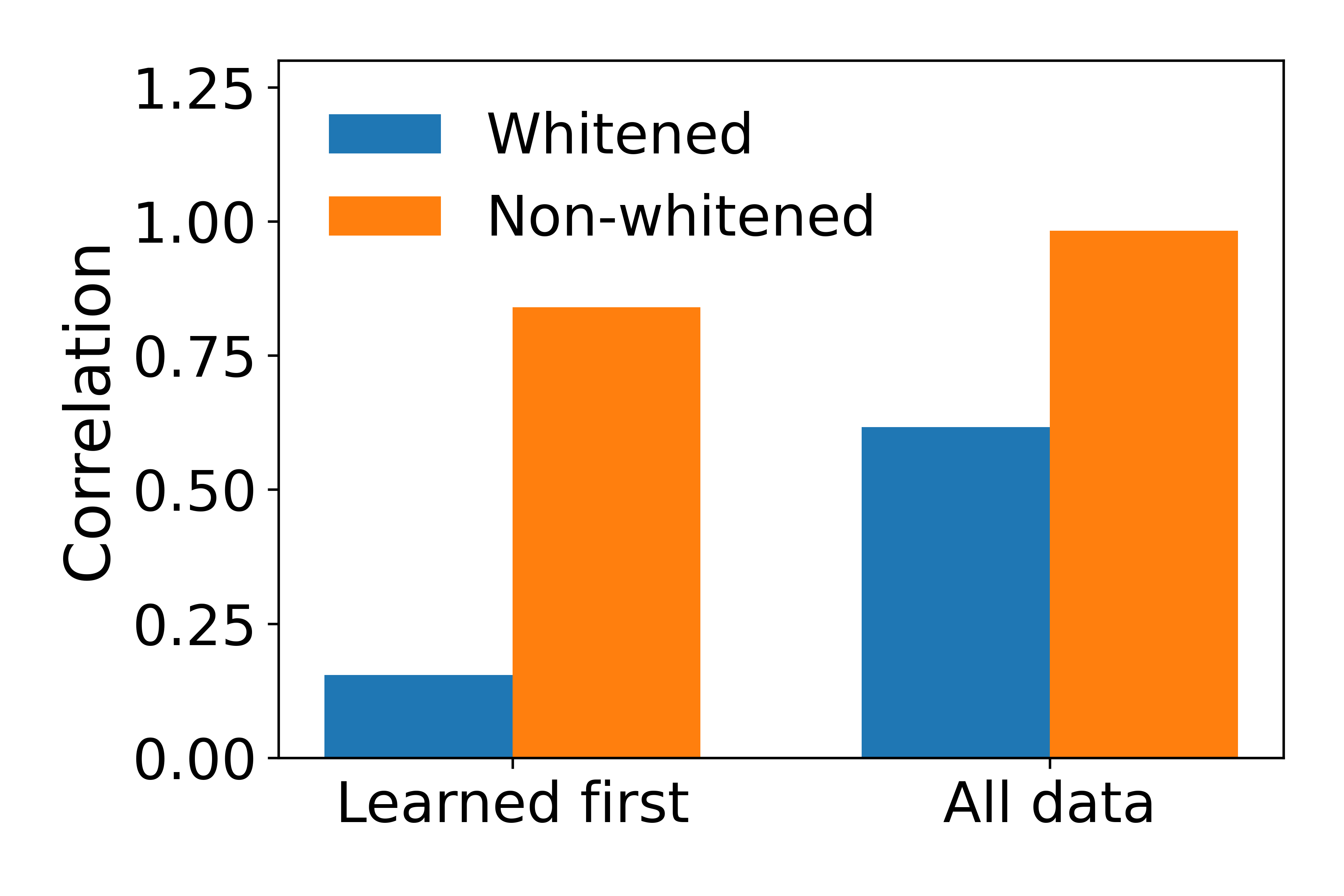}
          \caption{Linear networks}
          \label{subfig:accessilibity_whiten_cats_and_dogs_linear_network}
        \end{subfigure}
        \begin{subfigure}{.49\textwidth}
          \centering
          \includegraphics[width=1\linewidth]{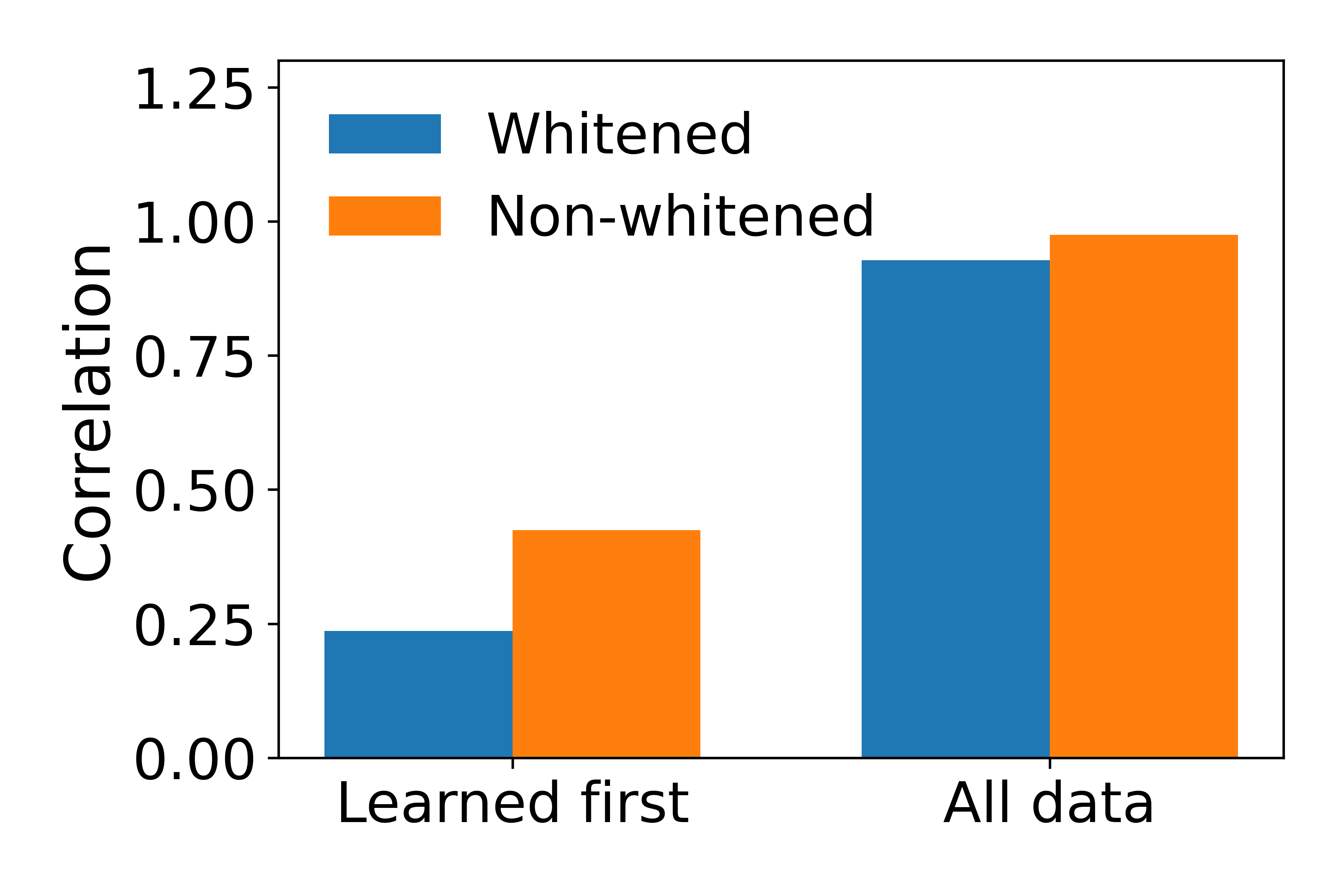}
          \caption{Non-linear networks}
          \label{subfig:accessilibity_whiten_cifar10_vgg}
        \end{subfigure}
    \end{subfigure}

    \caption{\emph{LOC} as measured with and without \emph{PC-bias}. Each bar represents the correlation between the learning order of two collections of $10$ networks trained on CIFAR-10. Orange bars represent natural images, in which the \emph{PC-bias} is present, while blue bars represent whitened data, in which the \emph{PC-bias} is neutralized. As \emph{PC-bias} is more prominent earlier on, we compute these correlations using the entire data (right two bars), and using the subset of $20\%$ "fastest learned" examples (left two bars).}
    \label{fig:accessibility_zca_regular}
\end{minipage}
\end{figure}
\footnotetext{As non-linear models achieve the accuracy of linear models within an epoch or 2, a low learning rate is used.}

\subsection{Neutralizing the PC-bias Leads to Diminishing LOC}
\label{sec:whitened}

Whitening the data eliminates the \emph{PC-bias} as shown in \fig\ref{fig:stds_cats_dogs_1_matrix_whitening}, since all the singular values are now identical. Here we use this observation to further probe into the dependency of the \emph{Learning Order Constancy} on the \emph{PC-bias}. Starting with the linear case, we trained four ensembles of $N$=$10$ two-layered linear networks (following the same methodology as Section~\ref{sec:methods}) on the cats and dogs dataset, two with and two without ZCA-whitening. We compute the \emph{accessibility} score for each ensemble separately, and correlate the scores of the two ensembles in each test case. Each correlation captures the consistency of the \emph{LOC-effect} for the respective condition. This correlation is expected to be very high for natural images. Low correlation implies that the \emph{LOC-effect} is weak, as training the same network multiple times yields a different learning order. 

\fig\ref{subfig:accessilibity_whiten_cats_and_dogs_linear_network} shows the results for deep linear networks. As expected, the correlation when using natural images is very high. However, when using whitened images, correlation plummets, indicating that the \emph{LOC-effect} is highly dependent on the \emph{PC-bias}. We note that the drop in the correlation is much higher when considering only the $20\%$ "fastest learned" examples, suggesting that the \emph{PC-bias} affects learning order more strongly at earlier stages of learning.

Fig.~\ref{subfig:accessilibity_whiten_cifar10_vgg} shows the results when repeating this experiment with non-linear networks, training two collections of $N$=$10$ VGG-19 networks on CIFAR-10. We observe that neutralizing the \emph{PC-bias} in this case affects \emph{LOC} much less, suggesting that the \emph{PC-bias} can only partially account for the \emph{LOC-effect} in the non-linear case. Nevertheless, we note that at the beginning of learning, when the \emph{PC-bias} is most pronounced, once again the drop is much larger and very significant (down by half), see left two bars of Fig.~\ref{subfig:accessilibity_whiten_cifar10_vgg}.

\subsection{Spectral Bias, PC-bias, and LOC}
\label{sec:empricial-freq-bias}

The \emph{spectral bias} \citep{rahaman2019spectral} characterizes the dynamics of learning in neural networks differently, asserting that initially neural models can be described by low frequencies only. This may provide an alternative explanation to LOC. Recall that LOC is manifested in the consistency of the \emph{accessibility} score across networks. To compare the \emph{spectral bias} and \emph{accessibility} score, we first need to estimate for each example whether it can be correctly classified by a low-frequency model. Accordingly, we define for each example a \emph{discriminability} score---the percentage out of its $k$ neighbors that share with it a class identity. Intuitively, an example has a low \emph{discriminability} score when it is surrounded by examples from other classes, which would presumably force the learned boundary to incorporate high frequencies. In \S\ref{app:sec:spectral_bias} we show that in the 2D case analyzed by \citet{rahaman2019spectral}, this measure strongly correlates ($r$=$-0.8$, $p<10^{-2}$) with the spectral bias.

We trained several networks (VGG-19 and st-VGG) on several real datasets, including small mammals, STL-10, CIFAR-10/100, and a subset of ImageNet-20 (see \app\ref{app:datasets}). For each network and dataset, we compute the \emph{accessibility} score as well as the \emph{discriminability} of each example. The vector space, in which discriminability is evaluated, is either the raw data or the network's perceptual space (penultimate layer activation). The correlation between these scores is shown in \tab\ref{tab:table1}.

\begin{figure}[h!]
  \captionof{table}{Correlation between \emph{accessibility} and \emph{discriminability}.}
  \begin{center}
      \begin{tabular}{lcc}
        \toprule
        Dataset       & Raw data     & Penultimate \\
        \midrule
        Small mammals & 0.46       & 0.85   \\
        ImageNet 20   & 0.01       & 0.54   \\
        CIFAR-100     & 0.51       & 0.85   \\
        STL10         & 0.44       & 0.7    \\
        \bottomrule
      \end{tabular}    
  \end{center}
\label{tab:table1}
\end{figure}

Using raw data, low correlation is still seen between the \emph{accessibility} and \emph{discriminability} scores when inspecting the smaller datasets (small mammals, CIFAR-100, and STL10). This correlation disappears when considering the larger ImageNet-20 dataset. It would appear that on its own, the \emph{spectral bias} cannot adequately explain the \emph{LOC-effect}. On the other hand, in the perceptual space, the correlation between \emph{discriminability} and \emph{accessibility} is quite significant for all datasets. Differently from the supposition of the spectral bias, it seems that networks learn a representation where the \emph{spectral bias} is evident, but this bias does not necessarily govern its learning before the representation has been obtained. 

\paragraph{Discussion.}
In the limit of infinite width, the spectral bias phenomenon in deep linear networks follows from the PC-bias. In function space, the training process of neural networks can be decomposed along different directions
defined by the eigenfunctions of the neural tangent kernel, where each direction has its 
convergence rate and the rate is determined by the corresponding eigenvalue \citep{conf_ijcai_CaoFWZG21}. Since the Neural Tangent Kernel for a deep linear network is proportional to $\Sigma_{XX}$, the spectral bias, in this case, corresponds to the statement that the right singular vectors of $X$ are learned at rates corresponding to their 
(squared) singular values. Looking at the spectral decomposition of $\hW X$, we see that the PC Bias implies in function space that the right singular vectors of $X$ are learned at rates corresponding to their singular values. Thus the PC Bias implies the Spectral Bias in this model.

\begin{figure}[thb]
\begin{center}
    \begin{subfigure}{.1\textwidth}
      \centering
      \includegraphics[width=.95\linewidth]{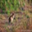}
    \end{subfigure}
    \begin{subfigure}{.1\textwidth}
      \centering
      \includegraphics[width=.95\linewidth]{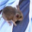}
    \end{subfigure}
    \begin{subfigure}{.1\textwidth}
      \centering
      \includegraphics[width=.95\linewidth]{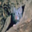}
    \end{subfigure}
    \begin{subfigure}{.1\textwidth}
      \centering
      \includegraphics[width=.95\linewidth]{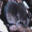}
    \end{subfigure}
    \begin{subfigure}{.1\textwidth}
      \centering
      \includegraphics[width=.95\linewidth]{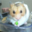}
    \end{subfigure}
    \begin{subfigure}{.1\textwidth}
      \centering
      \includegraphics[width=.95\linewidth]{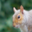}
    \end{subfigure}
    \begin{subfigure}{.1\textwidth}
      \centering
      \includegraphics[width=.95\linewidth]{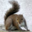}
    \end{subfigure}
    \begin{subfigure}{.1\textwidth}
      \centering
      \includegraphics[width=.95\linewidth]{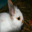}
    \end{subfigure}
    \begin{subfigure}{.1\textwidth}
      \centering
      \includegraphics[width=.95\linewidth]{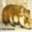}
    \end{subfigure}

  \begin{subfigure}{.1\textwidth}
      \centering
      \includegraphics[width=.95\linewidth]{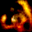}
    \end{subfigure}
    \begin{subfigure}{.1\textwidth}
      \centering
      \includegraphics[width=.95\linewidth]{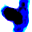}
    \end{subfigure}
    \begin{subfigure}{.1\textwidth}
      \centering
      \includegraphics[width=.95\linewidth]{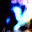}
    \end{subfigure}
    \begin{subfigure}{.1\textwidth}
      \centering
      \includegraphics[width=.95\linewidth]{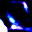}
    \end{subfigure}
    \begin{subfigure}{.1\textwidth}
      \centering
      \includegraphics[width=.95\linewidth]{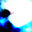}
    \end{subfigure}
    \begin{subfigure}{.1\textwidth}
      \centering
      \includegraphics[width=.95\linewidth]{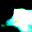}
    \end{subfigure}
    \begin{subfigure}{.1\textwidth}
      \centering
      \includegraphics[width=.95\linewidth]{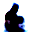}
    \end{subfigure}
    \begin{subfigure}{.1\textwidth}
      \centering
      \includegraphics[width=.95\linewidth]{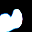}
    \end{subfigure}
    \begin{subfigure}{.1\textwidth}
      \centering
      \includegraphics[width=.95\linewidth]{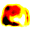}
    \end{subfigure}

    \begin{subfigure}{.1\textwidth}
      \centering
      \includegraphics[width=.95\linewidth]{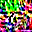}
    \end{subfigure}
    \begin{subfigure}{.1\textwidth}
      \centering
      \includegraphics[width=.95\linewidth]{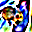}
    \end{subfigure}
    \begin{subfigure}{.1\textwidth}
      \centering
      \includegraphics[width=.95\linewidth]{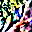}
    \end{subfigure}
    \begin{subfigure}{.1\textwidth}
      \centering
      \includegraphics[width=.95\linewidth]{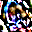}
    \end{subfigure}
    \begin{subfigure}{.1\textwidth}
      \centering
      \includegraphics[width=.95\linewidth]{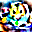}
    \end{subfigure}
    \begin{subfigure}{.1\textwidth}
      \centering
      \includegraphics[width=.95\linewidth]{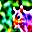}
    \end{subfigure}
    \begin{subfigure}{.1\textwidth}
      \centering
      \includegraphics[width=.95\linewidth]{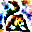}
    \end{subfigure}
    \begin{subfigure}{.1\textwidth}
      \centering
      \includegraphics[width=.95\linewidth]{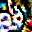}
    \end{subfigure}
    \begin{subfigure}{.1\textwidth}
      \centering
      \includegraphics[width=.95\linewidth]{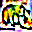}
    \end{subfigure}
\end{center}
\caption{Visualization of the small mammals dataset, with amplification of $1.5\%$ of its principal components by a factor of $10$. Top: original data; middle: data amplified along the highest principal components; bottom: data amplified along the lowest principal components}
\label{fig:small_mammals_amplification_visualization}
\end{figure}

\section{PC-Bias: Further Implications}
\label{sec:implication}

\paragraph{Early Stopping and the Generalization Gap.} 
Considering natural images, it is often assumed that the least significant principal components of the data represent noise \citep{torralba2003statistics}. In such cases, our analysis predicts that as noise dominates the components learned later in learning, early stopping is likely to be beneficial. To test this hypothesis directly, we manipulated CIFAR-10 to amplify the signal in either the $1.5\%$ most significant (higher) or $1.5\%$ least significant (lower) principal components (see examples in \fig\ref{fig:small_mammals_amplification_visualization}). Accuracy over the original test set, after training $10$ st-VGG and linear st-VGG networks on these manipulated images, can be seen in \fig\ref{fig:first_and_last_principal_directions_noise}. Both in linear and non-linear networks, early stopping is more beneficial when lower principal components are amplified, and significantly less so when higher components are amplified, as predicted by the \emph{PC-bias}.



\paragraph{Slower Convergence with Random Labels.} Deep neural models can learn any random label assignment to a given training set \citep{DBLP:conf/iclr/ZhangBHRV17}. However, when trained on randomly labeled data, convergence appears to be much slower \citep{DBLP:conf/iclr/KruegerBJAKMBFC17}. Assume, as before, that in natural images the lower principal components are dominated by noise. We argue that the \emph{PC-bias} now predicts this empirical result, since learning randomly labeled examples requires a signal present in lower principal components. To test this hypothesis directly, we trained $10$ two-layered linear networks (following the same methodology as in Section~\ref{sec:methods}) using datasets of natural images. Indeed, these networks converge slower with random labels (see Fig.~\ref{fig:random-labels_linear_case}). In Fig.~\ref{fig:random-linear_whiten} we repeat this experiment after having whitened the images, to neutralize the \emph{PC-bias}. Now convergence rate is identical, whether the labels are original or shuffled. Clearly, in deep linear networks, the \emph{PC-bias} gives a full account for this phenomenon.

\begin{figure}[h!]
\begin{minipage}{.49\textwidth}
    \begin{subfigure}{.485\textwidth}
      \centering
      \includegraphics[width=1\linewidth]{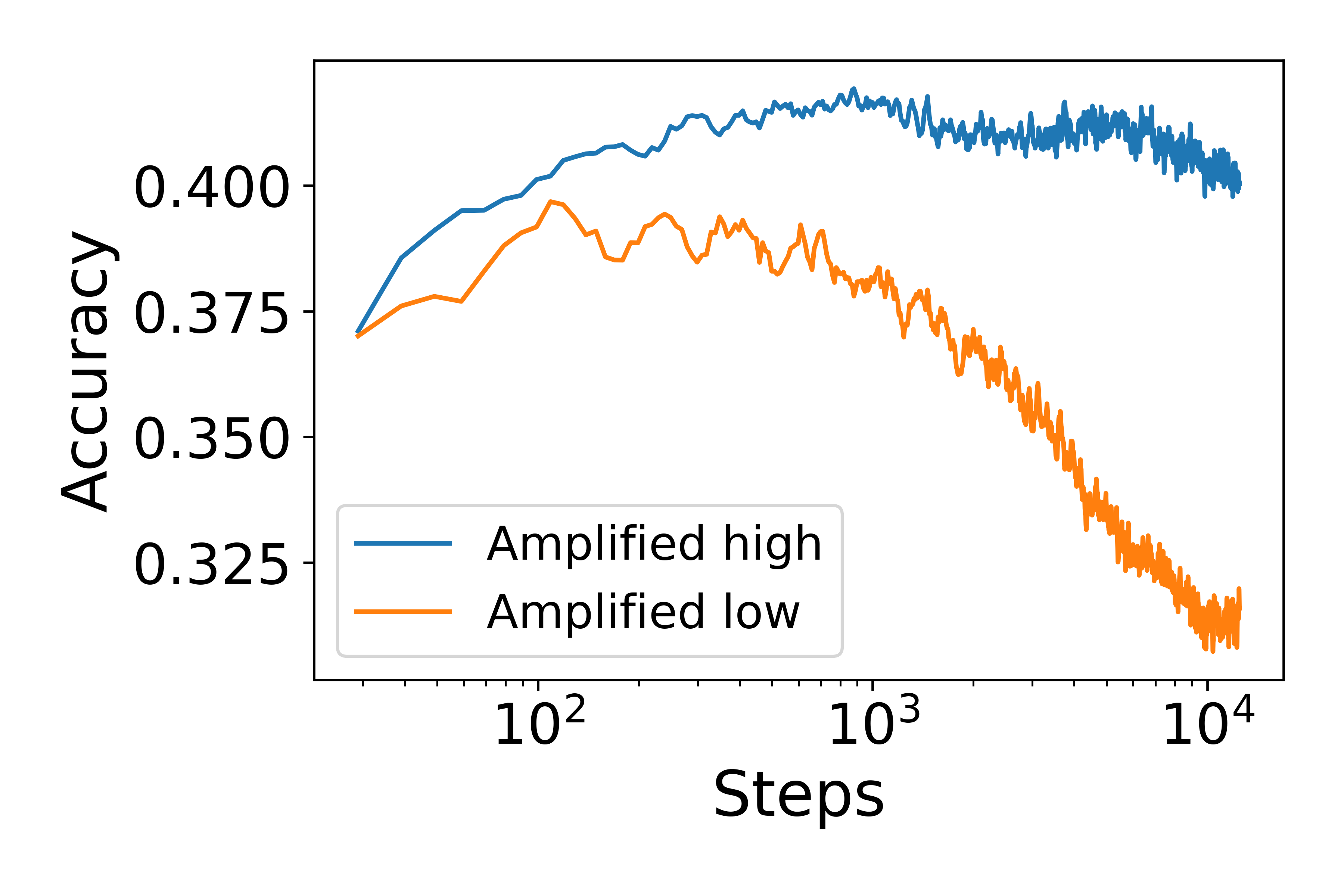}
      \caption{linear network}
      \label{subfig:first_and_last_principal_directions_noise_linear}
    \end{subfigure}
    \begin{subfigure}{.485\textwidth}
      \centering
      \includegraphics[width=1\linewidth]{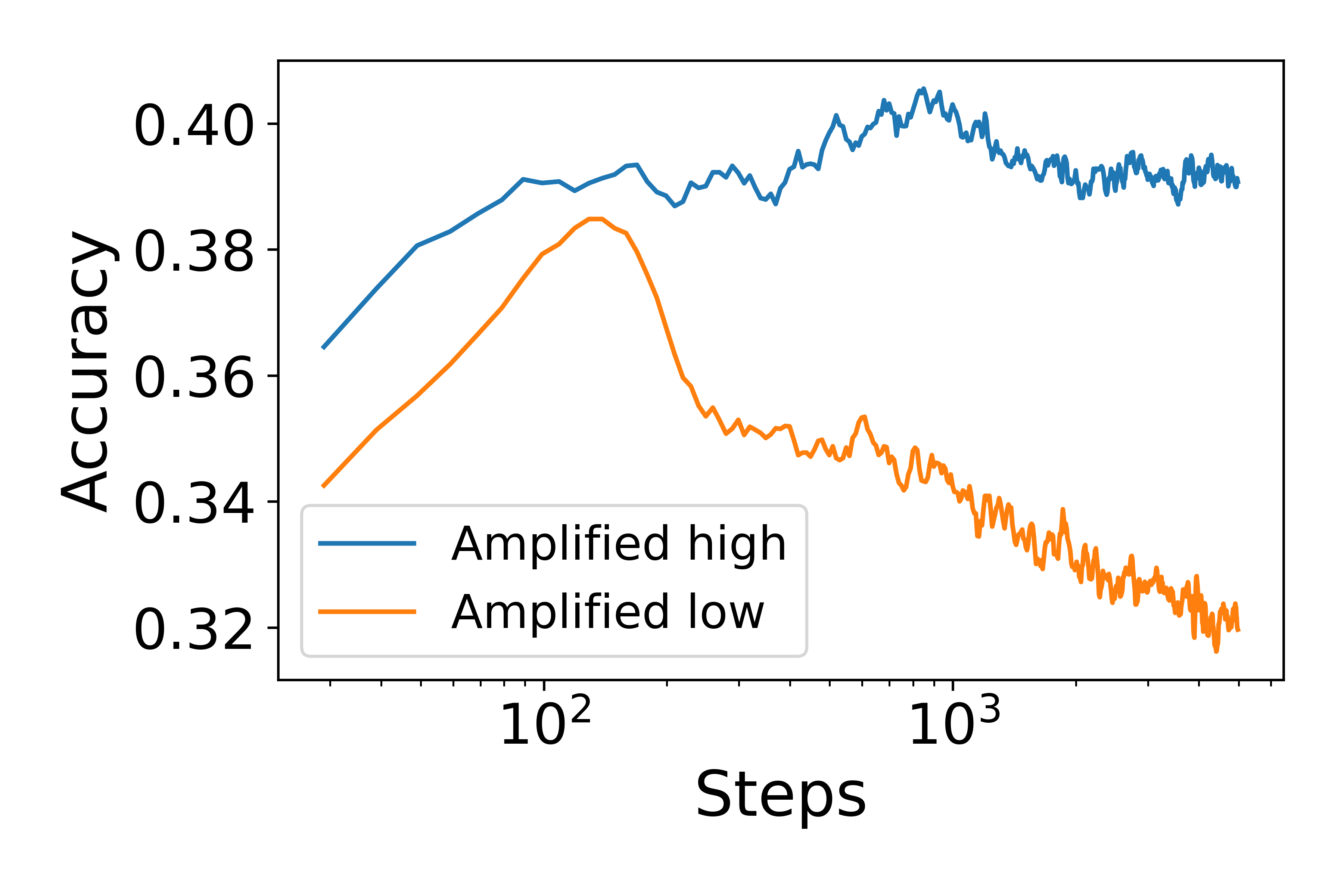}
      \caption{non-linear network}
      \label{subfig:first_and_last_principal_directions_noise_non_linear}
    \end{subfigure}
\caption{Comparing the accuracy trajectory when  amplifying the highest (blue line) and lowest (orange line) principal components.}
\label{fig:first_and_last_principal_directions_noise}
\end{minipage}
\hspace{0.3cm}
\begin{minipage}{.49\textwidth}
    \begin{subfigure}{.485\textwidth}
      \centering
      \includegraphics[width=1\linewidth]{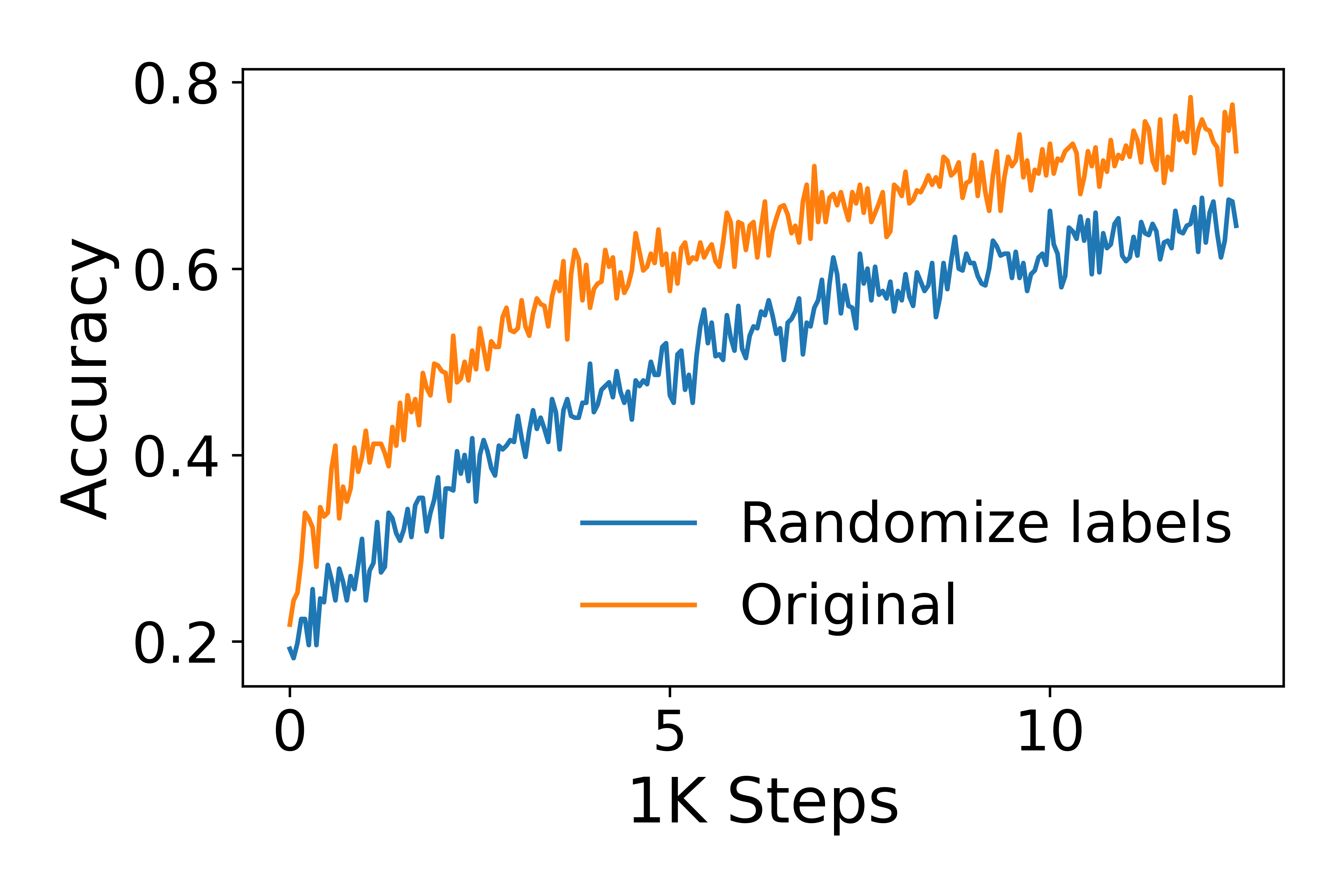}
      \caption{original data}
      \label{fig:random-labels_linear_case}
    \end{subfigure}
    \begin{subfigure}{.485\textwidth}
      \centering
      \includegraphics[width=1\linewidth]{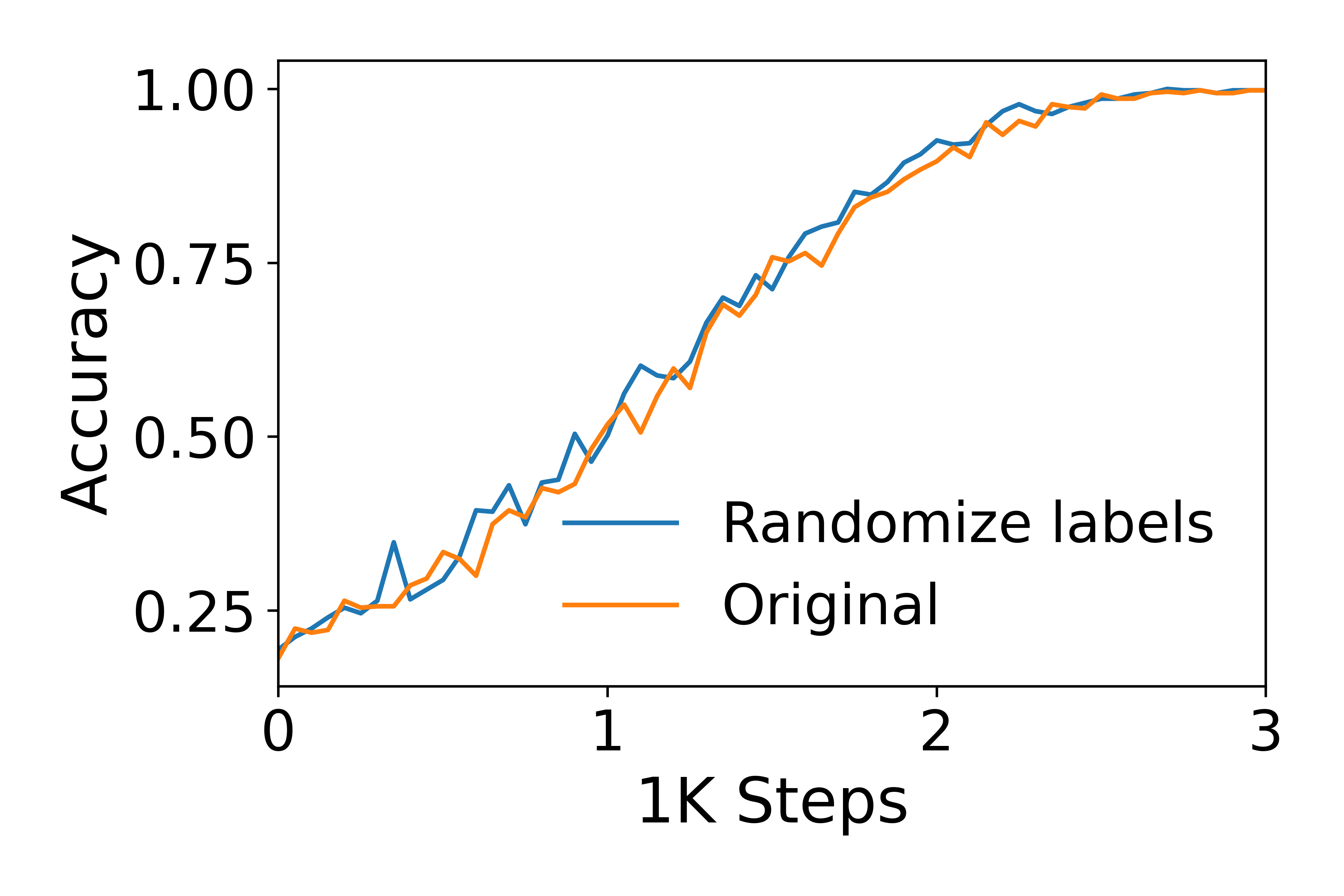}
      \caption{whitened data}
      \label{fig:random-linear_whiten}
    \end{subfigure}
\caption{Learning curves when using real and shuffled labels: $10$ two-layered linear networks, (a) before and (b) after  whitening.}
\label{fig:learning-curves}
\end{minipage}
\end{figure}


To further check the relevance of this account to non-linear networks, we artificially generate datasets where only the first $P$ principal components are discriminative, while the remaining components become noise by design. We constructed two such datasets: in one the labels are correlated with the original labels, while in the other they are not. Specifically, PCA is used to reduce the dimensionality of a two-class dataset to $P$, and the optimal linear separator in the reduced representation is computed. Next, all the labels of points that are incorrectly classified by the optimal linear separator are switched, so that the train and test sets are linearly separable by this separator. Note that the modified labels are still highly correlated with the original labels (for $P=500$: $p=0.82$, $r<10^{-10}$). The second dataset is generated by repeating the process while starting from randomly shuffled labels. This dataset is likewise fully separable when projected to the first $P$ components, but its labels are uncorrelated with the original labels (for $P=500$: $p=0.06$, $r<10^{-10}$).

The mean training accuracy of $10$ non-linear networks with $P$=$10$,$50$,$500$ is plotted in Fig.~\ref{fig:learning_curve_change_labels} (first dataset) and Fig.~\ref{fig:learning_curve_change_labels_randomize} (second dataset). In both cases, the lower $P$ is (namely, only the first few principal components are discriminative), the faster the data is learned by the non-linear network. Whether the labels are real or shuffled makes little qualitative difference, as predicted by the \emph{PC-bias}.

\begin{figure}[htb]
\begin{center}
    \begin{subfigure}{.4\textwidth}
      \centering
      \includegraphics[width=1\linewidth]{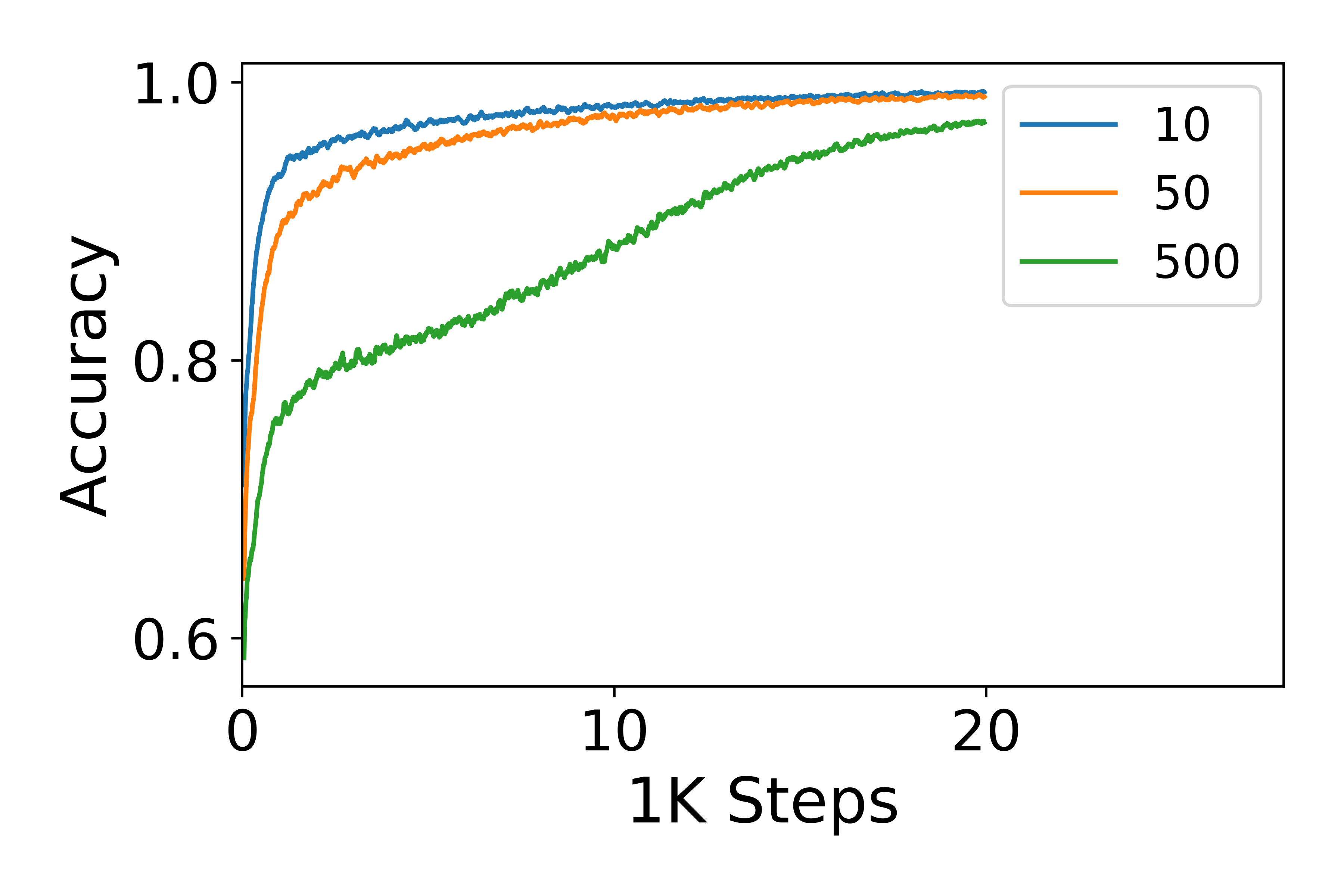}
      \caption{Original labels}
      \label{fig:learning_curve_change_labels}
    \end{subfigure}
    \hspace{1.5cm}
    \begin{subfigure}{.4\textwidth}
      \centering
      \includegraphics[width=1\linewidth]{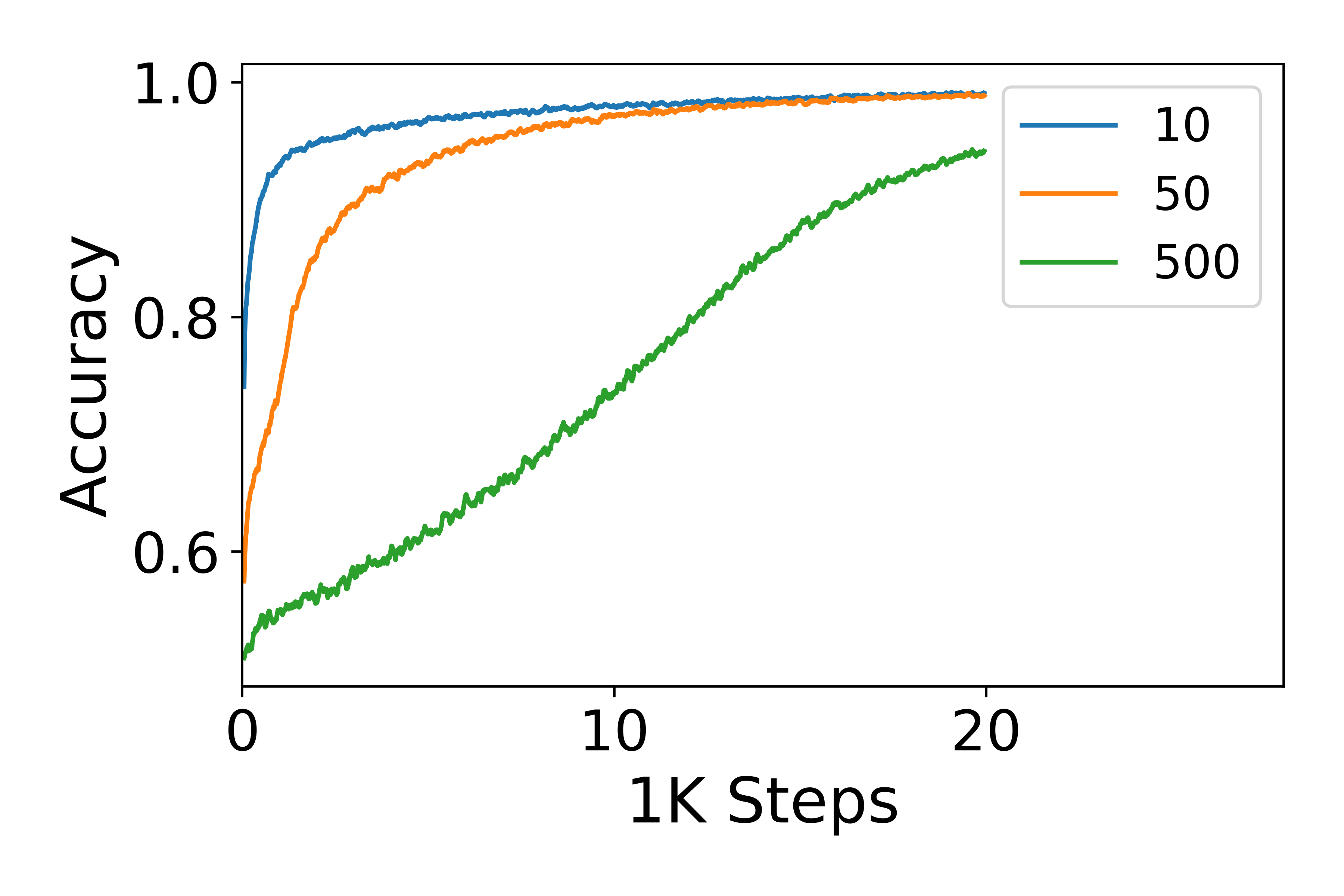}
      \caption{Shuffled labels}
      \label{fig:learning_curve_change_labels_randomize}
    \end{subfigure}
\end{center}
\caption{Learning curves of st-VGG networks trained on 3 datasets, which are linearly separable after projection to the highest $P$ principal components (see legend).}
\label{fig:learning_curves_with_label_change_all}
\end{figure}

\section{Summary and Discussion}

When trained with gradient descent, the convergence rate of the over-parameterized deep linear network model is provably governed by the eigendecomposition of the data. Specifically, we show that parameters corresponding to the most significant principal components converge faster than the least significant components. Empirical evidence is provided for the relevance of this result to more realistic non-linear networks. We term this effect \emph{PC-bias}. This result provides a complementary account for some prevalent empirical observations, including the benefit of early stopping and the slower convergence rate with shuffled labels. 

We use the \emph{PC-bias} to explain the \emph{Learning Order Constancy (LOC)}. Different empirical schemes are used to show that examples learned at earlier stages are more distinguishable by the data's higher principal components, which may indicate that networks' training relies more heavily on higher principal components early on. A causal link between the \emph{PC-bias} and the \emph{LOC-effect} is established, as the \emph{LOC-effect} diminishes when the \emph{PC-bias} is eliminated by  whitening the images. Finally, we analyze these findings given a related phenomenon termed \emph{spectral bias}. While the \emph{PC-bias} may be more prominent early on, the \emph{spectral bias} may be more important in later stages of learning.

\section*{Acknowledgments}
We thank our two reviewers for the elaborated and insightful suggestions, which contributed to this work. This work was supported in part by a grant from the Israeli Ministry of Science and Technology, and by the Gatsby Charitable Foundations.

\bibliography{bib}

\appendix
\section*{Appendix}

\section{Random Matrices}
\label{app:random}

\subsection{Multiplication of Random Matrices}
\label{sec:random-mat}

In this section, we present and prove some statistical properties of general random matrices and their multiplications. Let $\{Q_\rn\in\R^{m_{\rn}\times m_{\rn-1}}\}_{n=1}^\tL$ denote a set of random matrix whose elements are sampled iid from a distribution with mean $0$ and variance $\sigma_\rn^2$, whose kurtosis is bounded by $c$, and 
whose support is compact where the norm of each element is bounded by $c'\sigma_\rn^2$ for fixed constants $c,c'$. Let
\begin{align}
\label{eq:ABdef}
\hQ^\rl&=\prod_{\rn=\rl}^{1} Q_\rn = Q_{\rl} \cdot\ldots\cdot Q_1,  & \Bl^\rl={\hQ^\rl}^\top \hQ^{\rl}\in\R^{m_0\times m_0}, \\
\label{eq:AAdef}
\sQ^\rl&=\prod_{\rn=\tL}^{\rl+1} Q_\rn = Q_\tL \cdot\ldots\cdot Q_{\rl+1},  &\Al^\rl=\sQ^{\rl} {\sQ^{\rl}}^\top\in\R^{m_\tL\times m_\tL}.
\end{align}

\begin{theorem}
\label{thm:qtq}
For random matrices $\Al^\ri$ and $\Bl^\ri$ as defined in (\ref{eq:ABdef})-(\ref{eq:AAdef})
\begin{align}
\label{eq:qtq}
\E(\Bl^\rl)&=\beta_\rl I, \quad\quad\beta_\rl = \prod\limits_{\rn=1}^{\rl}m_\rn \sigma_\rn^{2}, \\
\label{eq:qqt}
\E(\Al^\rl)&=\alpha_\rl I, \quad\quad\alpha_\rl = \prod\limits_{\rn=\rl+1}^{\tL}m_{\rn-1} \sigma_\rn^{2}.
\end{align}
\end{theorem}

\begin{proof}
We only prove (\ref{eq:qtq}), as the proof of  (\ref{eq:qqt}) is similar. To simplify the presentation, we use the following auxiliary notations: $V=Q_1$, $U=\prod_{\rn=\rl}^{2}Q_\rn\implies\hQ^\rl = U \hV$. 

The proof proceeds by induction on $\rl$.
\begin{itemize}

\item $\rl=1$: 
\begin{equation*}
\begin{split}
\E[\Bl^1_{ij}] &= \E[\sum_{k=1}^{m_1} \hV_{ki}\hV_{kj}] = \sum_{k=1}^{m_1} \E[\hV_{ki}]\E[\hV_{kj}]=0  \qquad {i\neq j},\\
\E[\Bl^1_{ii}] &= \E[\sum_{k=1}^{m_1} \hV_{ki}\hV_{ki}]=\sum_{k=1}^{m_1} \E[\hV_{ki}^2]=m_1 \sigma_1^{2}.\\
\end{split}
\end{equation*}
Thus $\E(\Bl^1)=\beta_1 I$.

\item Assume that (\ref{eq:qtq}) holds for $\rl-1$. 
\begin{equation*}
\Bl^\rl_{ij} = \sum_{k} \hQ^\rl_{ki}\hQ^\rl_{kj} = \sum_{k} \sum_{\nu}U_{k\nu}\hV_{\nu i}  \sum_{\rho}U_{k\rho}\hV_{\rho j}, \\
\end{equation*}
and therefore 
\begin{equation}
\label{eq:18}
\E[\Bl^\rl_{ij}]=\sum_{k} \sum_{\nu}\sum_{\rho}\E[U_{k\nu}\hV_{\nu i}U_{k\rho}\hV_{\rho j}]=\sum_{\nu}\sum_{\rho}\E[\hV_{\nu i}\hV_{\rho j}] \sum_{k} \E[U_{k\nu}U_{k\rho}].
\end{equation}
The last transition follows from the independence of $U$ and $\hV$. From (\ref{eq:18}), where we denote $\Bl'=U^\top U$ 
\begin{equation*}
\begin{split}
\E[\Bl^\rl_{ij}]&= 
\sum_{\nu}\E[\hV_{\nu i}]\sum_{\rho}\E[\hV_{\rho j}] \E[(U^\top U)_{\nu\rho}]=0\qquad {i\neq j},\\
\E[\Bl^\rl_{ii}] &= \sum_{\nu}\sum_{\rho}\E[\hV_{\nu i}\hV_{\rho i}] \E[(U^\top U)_{\nu\rho}]= \sum_{\nu=1}^{m_1}\E[\hV_{\nu i}^2] \E[\Bl'_{\nu\nu}] = m_1 \sigma_1^{2} \prod\limits_{\rn=2}^{\rl}m_{\rn} \sigma_\rn^{2} = \beta_\rl.
\end{split}
\end{equation*}
where the last transition above follows from the induction assumption applied to $\Bl'=U^\top U$. Thus $\E(\Bl^\rl)=\beta_l I$ and (\ref{eq:qtq}) follows.
\end{itemize}

\end{proof}

Let $\cm$ denote the width of the smallest hidden layer, $\cm=\min\left({m_1,\ldots,m_{\ttL-1}}\right)$, and assume that $\max\left({m_1,\ldots,m_{\ttL-1}}\right)\leq\cm+\Dm$ for a fixed constant $\Dm$. $m_0,m_\ttL$ are likewise fixed constants ($m_0$ corresponds to the input dimension $q$ and $m_\ttL$ corresponds to the number of classes $\mL$), while $\cm$ is not bounded. Assume that the distribution of the random matrices $\{Q_\rn\}_{\rn=1}^\ttL$ is normalized as specified in Def.~\ref{def:1}, where specifically
\begin{equation}
\label{eq:normalization}
 \sigma^2_\rn = \frac{2}{m_{\rn-1}+m_\rn} ~\forall \rn\in[2\ldots\ttL-1], \qquad\sigma^2_1=\frac{1}{m_1},\qquad \sigma^2_\ttL=\frac{1}{m_{\ttL-1}}.
\end{equation}
%
It follows that asymptotically, when $\cm\to\infty$, we can write
\begin{alignat*}{4}
m_\rn\sigma_\rn^2&=1+\omverm ~~\rn\in&[1&\ldots \ttL-1], & \qquad m_{\ttL}\sigma_{\ttL}^2 = \frac{m_{\ttL}}{\cm}+ \omverm, 
\\
m_{\rn-1} \sigma_\rn^{2}&=1+\omverm ~~ \rn\in&[2&\ldots \ttL], &\qquad m_{0}\sigma_{1}^2 = \frac{m_0}{\cm}+ \omverm. 
\end{alignat*}

\begin{corollary}
\label{cor:alpha_beta}
From \thm\ref{thm:qtq}, using the initialization scheme specified in Def.~\ref{def:1} and (\ref{eq:normalization})
\begin{alignat*}{4}
\E(\Bl^\rl)&= I+ \omverm  ~~\forall\rl\in&[1&\ldots \ttL-1], & \qquad \E(\Bl^{\ttL}) = \frac{m_\ttL}{\cm}I+ \omverm = \omverm,\\
\E(\Al^\rl)&= I+ \omverm  ~~ \forall\rl\in&[1&\ldots \ttL-1], &\qquad \E(\Al^{0})  = \frac{m_0}{\cm}I+ \omverm = \omverm.
\end{alignat*}
\end{corollary}

Recall that in our asymptotic matrix notations, $\omverm$ is short-hand for a matrix which, for large enough $\cm$, is upper bounded element-wise by $C\frac{1}{\cm}$ where $C$ denotes a fixed full rank matrix, and similarly $O(\mu^2)$ for small enough $\mu$. When $\Bl^\rl$ is concerned $C\in\R^{m_0\times m_0}$, and when $\Al^\rl$ is concerned  $C\in\R^{m_\ttL\times m_\ttL}$. Henceforth, for clarity and simplicity of notations, when we discuss asymptotic properties of functions of a certain matrix $M$, including its expected value $\E(M)$ or variance $\var(M)$, it is to be understood that the properties are considered element-wise.

\begin{theorem}
\label{thm:var}
Given random matrices $\Al^\ri$ and $\Bl^\ri$ as defined in (\ref{eq:ABdef})-(\ref{eq:AAdef}), and using the initialization scheme specified in Def.~\ref{def:1} and (\ref{eq:normalization}), we have
\begin{equation*}
\var(\Bl^\rl)= \omverm,\qquad\var(\Al^\rl)= \omverm\quad\forall l.
\end{equation*}
\end{theorem}

\begin{proof}
Once again, we only provide a detailed proof for $\var(\Bl^\rl)$, as the proof for $\var(\Al^\rl)$ is similar.  From \cor\ref{cor:alpha_beta}, and since element-wise $Var(Z)=\E(Z^2)-[\E(Z)]^2$, it is sufficient to show that the following (somewhat stronger) assertion is valid:

\begin{equation}
\label{eq:varBl}
\E[(\Bl^\rl_{ij})^2]=\begin{cases}
\omverm & i\neq j  \\
1+\omverm  & i=j,\rl<\ttL \\
\omverm  & i=j,\rl=\ttL 
\end{cases}, \qquad
\E[\Bl^\rl_{ii}\Bl^\rl_{jj}] = 1+\omverm ~~i\neq j 
\end{equation}
The proof proceeds by induction on $\rl$.

\paragraph{Base case $\rl=1$.}
Below $Q$ stands for $Q_1$, to simplify index notations.
{\small
\begin{equation*}
\E[(\Bl^1_{ij})^2] = \E\bigg[\sum_{\nu=1}^{m_1}Q_{\nu i}Q_{\nu j}\sum_{\rho=1}^{m_1}Q_{\rho i}Q_{\rho j}\bigg ] = \begin{cases}
\sum\limits_{\nu=1}^{m_1} \E[Q_{\nu i}^2]\E[Q_{\nu j}^2] = \frac{1}{m_1} & i\neq j  \\
\sum\limits_{\nu=1}^{m_1} \sum\limits_{{\substack{\rho=1 \\ \rho\neq\nu}}}^{m_1}   \E[Q_{\nu i}^2]\E[Q_{\rho i}^2] +  \sum\limits_{\nu=1}^{m_1} \E[Q_{\nu i}^4] = 1+\omverm  & i=j 
\end{cases}
\end{equation*}
}
Above we use the assumed fixed bound on the kurtosis of the distribution of $Q$.
\begin{equation*}
\E[\Bl^\rl_{ii}\Bl^\rl_{jj}] = \E\bigg[\sum_{\nu=1}^{m_1}Q_{\nu i}Q_{\nu i}\sum_{\rho=1}^{m_1}Q_{\rho j}Q_{\rho j}\bigg ] = \sum\limits_{\nu=1}^{m_1} \sum\limits_{\rho=1}^{m_1}   \frac{1}{m_1}\frac{1}{m_1} = 1.
\end{equation*}

\paragraph{Induction step.}
Assume that (\ref{eq:varBl}) holds for $\rl-1$, and similarly to the above, let $Q$ stand for $Q_\rl$ to simplify index notations. Let $\frac{1}{\tilde m_\rl}$ denote the variance of $Q_\rl$ as defined in (\ref{eq:normalization}).
{\small
\begin{equation}
\begin{split}
\E[(\Bl^{\rl}_{ij})^2] &= \E\bigg[\sum_{\nu=1}^{m_\rl} \sum_{\rho=1}^{m_\rl} Q_{\nu i}\Bl^{\rl-1}_{\nu\rho}Q_{\rho j} \sum_{\alpha=1}^{m_\rl} \sum_{\beta=1}^{m_\rl} Q_{\alpha i}\Bl^{\rl-1}_{\alpha\beta}Q_{\beta j} \bigg] \quad \texp{by~independence}\\
&=\sum_{\nu=1}^{m_\rl} \sum_{\rho=1}^{m_\rl} \sum_{\alpha=1}^{m_\rl} \sum_{\beta=1}^{m_\rl} ~\E[Q_{\nu i} Q_{\rho j} Q_{\alpha i}Q_{\beta j}] \E[\Bl^{\rl-1}_{\nu\rho}\Bl^{\rl-1}_{\alpha\beta} ].
\end{split}
\label{eq:Bsquare}
\end{equation}
}
When ${i\neq j}$, using the independence of the elements of $Q_l$ and the induction assumption
\begin{equation*}
\begin{split}
\E[(\Bl^{\rl}_{ij})^2] &= \sum_{\substack{\nu=1 \\ \alpha=\nu}}^{m_\rl}\sum_{\substack{\rho=1 \\ \beta=\rho}}^{m_\rl}
\E[Q_{\nu i}^2]\E[Q_{\rho j}^2]\E[(\Bl^{\rl-1}_{\nu\rho})^2 ] \\
&= \sum_{\nu=1}^{m_\rl}\sum_{\substack{\rho=1 \\ \rho\neq\nu}}^{m_\rl}\frac{1}{\tilde m_\rl}\frac{1}{\tilde m_\rl}\E[(\Bl^{\rl-1}_{\nu\rho})^2 ] + \sum_{\substack{\nu=1 \\ \rho=\nu}}^{m_\rl}\frac{1}{\tilde m_\rl}\frac{1}{\tilde m_\rl}\E[(\Bl^{\rl-1}_{\nu\nu})^2 ] =  \omverm \quad\forall l.
\end{split}
\end{equation*}
When $i=j$, we collect below all the terms in (\ref{eq:Bsquare}) which are not 0:
\begin{equation*}
\begin{split}
E[(\Bl^{\rl}_{ii})^2] &=\sum\limits_{\substack{\nu=1 \\ \rho=\nu}}^{m_\rl} \sum\limits_{{\substack{\alpha=1 \\ \beta=\alpha \\ \alpha\neq\nu}}}^{m_\rl} \E[Q_{\nu i}^2]\E[Q_{\alpha i}^2]\E[\Bl^{\rl-1}_{\nu\nu}\Bl^{\rl-1}_{\alpha\alpha} ] +
\sum\limits_{\substack{\nu=1 \\ \alpha=\nu}}^{m_\rl} \sum\limits_{{\substack{\rho=1 \\ \beta=\rho \\ \rho\neq\nu}}}^{m_\rl} \E[Q_{\nu i}^2]\E[Q_{\rho i}^2]\E[(\Bl^{\rl-1}_{\nu\rho})^2 ]  \\
&+\sum\limits_{\substack{\nu=1 \\ \beta=\nu}}^{m_\rl} \sum\limits_{{\substack{\rho=1 \\ \alpha=\rho \\ \rho\neq\nu}}}^{m_\rl} \E[Q_{\nu i}^2]\E[Q_{\rho i}^2]\E[(\Bl^{\rl-1}_{\nu\rho})^2 ] 
+\sum\limits_{\substack{\nu=1 \\ \rho=\alpha=\beta=\nu }}^{m_\rl} \E[Q_{\nu i}^4]\E[(\Bl^{\rl-1}_{\nu\nu} )^2].
\end{split}
\end{equation*}
From the induction assumption and since the kurtosis of $Q$ is bounded by the assumption
\begin{equation}
\begin{split}
E[(\Bl^{\rl}_{ii})^2] &=\sum\limits_{\nu=1}^{m_\rl} \sum\limits_{{\substack{\alpha=1 \\ \alpha\neq\nu}}}^{m_\rl} \E[Q_{\nu i}^2]\E[Q_{\alpha i}^2]\E[\Bl^{\rl-1}_{\nu\nu}\Bl^{\rl-1}_{\alpha\alpha} ] +\omverm\\
&=\sum\limits_{\nu=1}^{m_\rl} \sum\limits_{{\substack{\alpha=1 \\ \alpha\neq\nu}}}^{m_\rl} \frac{1}{\tilde m_\rl} \frac{1}{\tilde m_\rl} \E[\Bl^{\rl-1}_{\nu\nu}\Bl^{\rl-1}_{\alpha\alpha} ]+\omverm = \begin{cases} 1 + \omverm  &\rl<\ttL \vspace{0.2cm} \\ \omverm& \rl=\ttL\end{cases}
\end{split}
\label{eq:Bii-sq}
\end{equation}
To justify the last transition, recall that the initialization scheme specified in Def.~\ref{def:1} implies that $m_\rl\frac{1}{\tilde m_\rl}=1+\omverm~\forall\rl\in[1\ldots \ttL-1]$. Making use once again of the induction assumption, it follows from (\ref{eq:Bii-sq}) that now $E[(\Bl^{\rl}_{ii})^2]=1+\omverm$. However, since $m_\ttL$ is fixed whereas $\frac{1}{\tilde m_\ttL}=\omverm$, it follows that $E[(\Bl^{\ttL}_{ii})^2]=\omverm$.

A similar argument will show that $\E[\Bl^\rl_{ii}\Bl^\rl_{jj}] = 1+\omverm$ when $i\neq j$.
\end{proof}

\begin{theorem}
\label{thm:cheby}
Let $\{X(\cm)\}$ denote a sequence of random matrices where $\E[X(\cm)]= \repI+ \omverm$ and $\var [X(\cm)]= \omverm$. Then $X(\cm)\xrightarrow{p}  \repI $, where $\xrightarrow{p}$ denotes element-wise convergence in probability as $\cm\to\infty$. 

\end{theorem}

\begin{proof}
For every element $(i,j)$ of matrix $X(\cm)$, we need to show that $\forall\varepsilon,\delta>0~\exists \cm'\in\mathbb{N}$, such that $\forall \cm>\cm'$
\begin{equation*}
Prob\left(\vert X_{ij}(\cm) -\repI_{ij} \vert >\varepsilon. \right)<\delta 
\end{equation*}
Henceforth we use $x,f$ as shorthand for $X_{ij}(\cm),F_{ij}$ respectively. Since $\E[X(\cm)]=\repI+ \omverm$ where by definition the asymptotic behavior occurs element-wise, it follows that $\forall\varepsilon\!>\!0$ $\exists \cm_1\in\mathbb{N}$ such that $\forall \cm>\cm_1$ we have
\begin{equation*}
\vert \E(x)-f\vert  < \frac{\varepsilon}{2},
\end{equation*}
in which case
\begin{equation*}
Prob\left(\vert x -f\vert >\varepsilon \right ) \leq Prob\left(\vert x -\E(x )\vert >\frac{\varepsilon}{2} \right).
\end{equation*}
Since $\var [X(\cm)]= \omverm$, it follows that $\forall\varepsilon,\delta>0,~\exists \cm_2\in\mathbb{N} ~~\ni~\forall \cm>\cm_2$
\begin{equation*}
\var(x) < \frac{\varepsilon^2}{4}\delta,
\end{equation*}
from the above, and using Chebyshev's inequality
\begin{equation*}
Prob\left(\vert x -f\vert >\varepsilon \right ) < \frac{4\var(x)}{\varepsilon^2} < \delta,
\end{equation*}
$\forall \cm>\cm'$, where $\cm'=\max\{\cm_1,\cm_2\}$.

\end{proof}

\begin{corollary}
\label{cor:limB}
Let $\Al^\ri(\cm)$ and $\Bl^\ri(\cm)$ denote a sequence of random matrices as defined in (\ref{eq:ABdef})-(\ref{eq:AAdef}), corresponding to multi-layer linear models for which $\cm=\min\left({m_1,...,m_{L-1}}\right)$. Then
\begin{alignat*}{6}
&\Bl^\rl(\cm)\xrightarrow{p}  I ~~\forall\rl\in&[1&\ldots \ttL-1], & \qquad &\Bl^{\ttL}(\cm)\xrightarrow{p} 0, \\
&\Al^\rl(\cm)\xrightarrow{p}  I  ~~ \forall\rl\in&[1&\ldots \ttL-1], &\qquad &\Al^{(0)}(\cm)\xrightarrow{p}  0.
\end{alignat*}
\end{corollary}
\begin{proof}
This result follows from \cor\ref{cor:alpha_beta}, \thm\ref{thm:var}, and \thm\ref{thm:cheby}. 
\end{proof}

\subsection{The Dynamics of Random Matrices}
\label{sec:random-mat-dyn}

We now formalize a dynamical system, which captures the evolution of the weight matrices $\{W_l\}$ by gradient descent as seen in (\ref{eq:delta_w})-(\ref{eq:delta_hW}). Specifically, given random matrices as defined in (\ref{eq:ABdef})-(\ref{eq:AAdef}), consider a dynamical system whereby $Q_j\rightarrow Q_j - \Delta Q_j~\forall j$, and
\begin{equation}
\begin{split}
\Delta Q_j  &= \mu\Big ( \prod_{\rn=\ttL}^{j+1}Q_\rn \Big )^\top G_r\Big ( \prod_{\rn=j-1}^{1}Q_\rn \Big )^\top= \mu [ \sQ^j ]^\top G_r[ \hQ^{j-1} ]^\top, \\
G_r &= \hQ^{\ttL}\SXX-\SYX.
\end{split}
\label{eq:deltaQ}
\end{equation}
Denoting $\hQ^\rl\rightarrow\hQ^\rl - \Delta\hQ^\rl,~\Bl^\rl\rightarrow \Bl^\rl-\Delta \Bl^\rl$ and applying the product rule
\begin{align}
\label{eq:Ql}
\Delta\hQ^\rl &= \sum_{j =1}^{\rl}  \Big ( \prod_{\rn=\rl}^{j+1}Q_\rn \Big ) \Delta Q_j\Big ( \prod_{\rn=j-1}^{1}Q_\rn \Big )= \sum_{j =1}^{\rl}  \Big ( \prod_{\rn=\rl}^{j+1}Q_\rn \Big ) \Delta Q_j\hQ^{j-1}, \\
\label{eq:U-dyn}
\Delta \Bl^\rl  &=   [{\Delta\hQ^\rl}^\top\hQ^\rl + {\hQ^\rl}^\top\Delta\hQ^\rl].
\end{align}

Recall that in our notations, $O(\mu^2)$ is short-hand for a matrix which, for small enough $\mu$, is upper bounded element-wise by $C\mu^2$ where $C$ denotes a \emph{fixed} full rank matrix. Similarly, $\omverm$ is short-hand for a matrix which, for large enough $\cm$, is upper bounded element-wise by $C\frac{1}{\cm}$. Since $m_0=m_q$ and $m_\ttL=m_\mL$, when $\Bl^\rl$ is concerned $C\in\R^{q\times q}$, and when $\Al^\rl$ is concerned  $C\in\R^{\mL\times \mL}$. In addition, we will use the notation $O(\varepsilon)$ as short-hand for a matrix that is upper bounded element-wise by $C\varepsilon$, where $C$ denotes a \emph{fixed} full rank matrix.

\begin{theorem}
\label{thm:delB}
Let $\Bl^\ri(\cm)={\hQ^\rl(\cm)}^\top \hQ^{\rl}(\cm)$ denote a sequence of random matrices as defined in (\ref{eq:ABdef}) for $\{Q_\rn\}_{\rn=1}^\ttL$, whose dynamics is captured by (\ref{eq:deltaQ})-(\ref{eq:U-dyn}). Assume 
that $\Bl^\rl(\cm)$ is full rank $\forall l,\cm$. If 
\begin{equation*}
\begin{split}
&\Bl^\rl(\cm) \xrightarrow{p} I+O(\mu^2)\qquad\forall\rl\in[1\ldots \ttL-1],\qquad\qquad\Bl^{\ttL}(\cm) \xrightarrow{p} O(\mu^2),\\
&\var[\Bl^\rl(\cm) ]= O(\mu^2)+\omverm \qquad\forall\rl, \\
&\E[\hQ^\ttL(\cm)] = O(\mu)+\omverm.
\end{split}
\end{equation*}
then
\begin{alignat*}{3}
&\Delta \Bl^\rl(\cm) \xrightarrow{p}  O(\mu^2), && \qquad \var[\Delta\Bl^\rl(\cm) ]= O(\mu^2)+\omverm \qquad\forall\rl, \\
&\E[\Delta\hQ^\ttL(\cm)]=O(\mu)+\omverm. && 
\end{alignat*}
\end{theorem}

\begin{proof}
$\Bl^\rl(\cm) \xrightarrow{p} I+O(\mu^2)$ and $\Bl^{\ttL}(\cm) \xrightarrow{p} O(\mu^2)$ implies that $\forall\varepsilon,\delta>0~\exists \hat\cm\in\mathbb{N}$, such that $\forall \cm>\hat\cm$ and with probability larger than $1-\delta$:
\begin{equation}
\label{eq:Beps}
\Bl^\rl(\cm) = I +O(\mu^2)+ O(\varepsilon)\quad\forall\rl\leq L-1, \qquad\qquad\Bl^{\ttL}(\cm) =O(\mu^2)+O(\varepsilon).
\end{equation}
To evaluate $\Delta \Bl^\rl$ from (\ref{eq:U-dyn}), we start from (\ref{eq:Ql}) to obtain
\begin{equation}
\label{eq:deltaB}
{\hQ^\rl}^\top\hspace{-3pt}\Delta\hQ^\rl = \sum_{j =1}^{\rl} {\hQ^\rl}^\top \Big ( \prod_{\rl}^{j+1}Q_\rn \Big ) \Delta Q_j \hQ^{j-1}.
\end{equation}
We make use of the Moore–Penrose pseudo-inverse of matrices $\{\hQ^\rl\}$, defined as
\begin{equation*}
[ \hQ^\rl ]^+  =  ({\hQ^\rl}^\top \hQ^{\rl})^{-1}{\hQ^\rl}^\top.
\end{equation*}
This definition is valid since by assumption $\Bl^\rl$ is full rank $\forall \rl$, and is therefore invertible. Additionally 
\begin{equation}
\label{eq:pseudo}
[ \hQ^\rl ]^+ \cdot [(\hQ^\rl)^\top]^+ =  ({\hQ^\rl}^\top \hQ^{\rl})^{-1}{\hQ^\rl}^\top\cdot {\hQ^\rl}({\hQ^\rl}^\top \hQ^{\rl})^{-1} = ({\hQ^\rl}^\top \hQ^{\rl})^{-1} = (\Bl^\rl)^{-1}.
\end{equation}
We next use (\ref{eq:ABdef}) and (\ref{eq:pseudo}) to simplify $T_j$---the $j^\mathrm{th}$ term in the sum (\ref{eq:deltaB})
\begin{alignat}{4}
T_j&=\mu{\hQ^\rl}^\top \left ( \prod_{\rl}^{j+1}Q_\rn \right ) [ \sQ^j ]^\top G_r[ \hQ^{j-1} ]^\top {\hQ^{j-1}} & \nonumber \\ &=\mu{\hQ^\rl}^\top \left ( \prod_{\rl}^{j+1}Q_\rn \right ) \cdot\hQ^{j} [ \hQ^{j} ]^+ \cdot[(\hQ^{j})^\top]^+ [\hQ^{j}]^\top \cdot[ \sQ^j ]^\top G_r \Bl^{j-1} & \texq{\lem\ref{lem:mat-flip}} \nonumber \\
&= \mu{\hQ^\rl}^\top \hQ^\rl [ \hQ^{j} ]^+ [(\hQ^{j})^\top]^+ \hQ^\ttL  G_r \Bl^{j-1}& \label{eq:tj-final} \\
&= \mu\Bl^\rl [\Bl^{j}]^{-1}{\hQ^{\ttL}}^\top [\hQ^{\ttL} \SXX-\SYX] \Bl^{j-1}  &\texq{using~(\ref{eq:deltaQ}),(\ref{eq:pseudo}) } \nonumber \\
&=\mu\Bl^\rl [\Bl^{j}]^{-1} [\Bl^{\ttL}\SXX -{\hQ^{\ttL}}^\top \SYX]\Bl^{j-1}  & \nonumber \\
&=-\mu{\hQ^{\ttL}}^\top \SYX +  O(\mu^2)+ O(\varepsilon). &\texq{using~(\ref{eq:Beps}) }\nonumber 
\end{alignat}
By assumption $\E[\hQ^\ttL]=O(\mu)+O(\varepsilon)$, and therefore
\begin{equation}
\label{eq:tij}
\E(T_j)= O(\mu^2) +  O(\varepsilon) \quad\implies\quad \E[{\hQ^\rl}^\top\hspace{-3pt}\Delta\hQ^\rl] = \sum_{j =1}^{\rl} \E(T_j) = O(\mu^2) +  O(\varepsilon).
\end{equation}
Finally, since ${\Delta\hQ^\rl}^\top\hQ^\rl = [{\hQ^\rl}^\top\Delta\hQ^\rl]^\top$, from (\ref{eq:U-dyn}) and (\ref{eq:tij})
\begin{equation}
\label{eq:deltaBl}
\E[\Delta \Bl^\rl ] = \E[{\Delta\hQ^\rl}^\top\hQ^\rl] + \E[{\Delta\hQ^\rl}^\top\hQ^\rl]^\top= O(\mu^2) + O(\varepsilon).
\end{equation}

To conclude the proof, we will show that $\forall\varepsilon',\delta'>0~\exists \hat\cm'\in\mathbb{N}$, such that $\forall \cm>\hat\cm'$ 
\begin{equation}
\label{eq:91}
Prob\left(\vert \Delta \Bl^\rl-O(\mu^2)\vert >\varepsilon' \right)<\delta' \qquad \mathrm{element\mbox{-}wise}.
\end{equation}
Let $b$ denote an element of matrix $\Delta \Bl^\rl-O(\mu^2)$. Recall that (\ref{eq:deltaBl}) is true element-wise with probability $(1-\delta)$ $\forall \varepsilon,\delta$ and $\forall \cm>\hat\cm$. Therefore, $\forall\varepsilon',\delta'>0$, we can choose $\varepsilon,\delta$ and the corresponding $\hat\cm'$ such that 
\begin{equation}
\label{eq:92}
Prob\left(\vert \E(b)\vert <\frac{\varepsilon'}{2} \right)>(1-\frac{\delta'}{2})\quad \forall \cm>\hat\cm',
\end{equation}
\begin{equation*}
\begin{split}
Prob\left(\vert b\vert >\varepsilon' \right ) &= Prob\left(\vert b\vert >\varepsilon',~ \vert\E(b)\vert <\frac{\varepsilon'}{2}\right ) + Prob\left(\vert b\vert >\varepsilon',~ \vert\E(b)\vert \geq\frac{\varepsilon'}{2}\right ) \\
&\leq Prob\left(\vert b -\E(b )\vert >\frac{\varepsilon'}{2} \right) + Prob\left( \vert\E(b)\vert \geq\frac{\varepsilon'}{2}\right ).
\end{split}
\end{equation*}
$\var(\Delta \Bl^\rl)=O(\mu^2)+\omverm$ implies that $\forall\varepsilon',\delta'>0~\exists \hat\cm''\in\mathbb{N}$, such that $\forall \cm>\hat\cm''$  
\begin{equation}
\label{eq:93}
\var(b)<  \frac{\delta'\varepsilon'^2}{8},
\end{equation}
Using Chebychev inequality, (\ref{eq:92}) and (\ref{eq:93}), we finally have
\begin{equation*}
Prob\left(\vert b\vert >\varepsilon' \right ) \leq \frac{4\var(b)}{\varepsilon'^2} + \frac{\delta'}{2} < \delta', \qquad \forall \cm>\max\{\hat\cm',\hat\cm''\}.
\end{equation*}
We can now conclude that (\ref{eq:91}) is true element-wise. 

We will only provide a sketch of the proof for $\var[\Delta\Bl^\rl(\cm) ]$ and $\E[\Delta\hQ^\ttL(\cm)]$, as a detailed proof follows very similar steps, and principles, to the proof presented above. To analyze the variance of $\Delta \Bl^\rl(\cm)$, we start from (\ref{eq:tj-final}) and observe that $\Delta \Bl^\rl(\cm) = -\rl\mu{\hQ^{\ttL}}^\top \SYX +O(\mu^2)+  O(\varepsilon)$, from which (and the theorem's assumptions) it can be shown that $\var[\Delta \Bl^\rl(\cm) ]= O(\mu^2)+\omverm$. Similarly to (\ref{eq:tj-final}), we can derive that $\E[\Delta\hQ^\rl(\cm)] = -\mu L G_r + O(\mu^2) + \omverm$, from which it follows that $\E[\Delta\hQ^\ttL(\cm)]=O(\mu)+\omverm$. 
\end{proof}

\begin{theorem}
\label{thm:delA}
Let $\Al^\ri(\cm)=\sQ^{\rl}(\cm){\sQ^\rl(\cm)}^\top $ denote a sequence of random matrices as defined in (\ref{eq:AAdef}) but where $\Al^{(0)}(\cm)=\sQ^{0}\SXX {\sQ^{0}}^\top$, whose dynamics is captured by (\ref{eq:deltaQ}). Assume that $\Al^\rl(\cm)$ is full rank $\forall l,\cm$. If 
\begin{equation*}
\begin{split}
&\Al^\rl(\cm) \xrightarrow{p} I+O(\mu^2)~~\forall\rl\in[1\ldots \ttL-1],\qquad\Al^{0}(\cm) \xrightarrow{p} O(\mu^2)\\
&\var[\Al^\rl(\cm) ]= O(\mu^2)+\omverm \qquad\forall\rl\\
&\E[\sQ^{0}(\cm)]=O(\mu)+\omverm,
\end{split}
\end{equation*}
then
\begin{alignat*}{3}
&\Delta \Al^\rl(\cm) \xrightarrow{p}  O(\mu^2),&&\qquad \var[\Delta\Al^\rl(\cm) ]= O(\mu^2)+\omverm, \qquad\forall\rl\\
&\E[\Delta\sQ^{0}(\cm)]=O(\mu)+\omverm.
\end{alignat*}
\end{theorem}
The proof is mostly similar to \thm\ref{thm:delB}.

\subsection{Some Useful Lemmas}
\label{app:lemmas}

\begin{lemma}
\label{lem:deriv}
Given function $G(W) =  \frac{1}{2}\Vert \rmat  W \lmat  X-Y \Vert_F^2 $, its derivative is the following
\begin{equation*}
\frac{d G(W)}{d W} = \rmat ^\top \rmat  W \lmat X (\lmat X)^\top - \rmat ^\top Y(\lmat X)^\top = \rmat ^\top[\rmat  W \lmat \SXX-\SYX]\lmat ^\top.
\end{equation*}
\end{lemma}

\begin{lemma}
\label{lem:mean_var_Q}
Assume $\hQ=\prod_{\rn=\tL}^{1} Q_\rn $, where $Q_\rn\in\R^{m_{\rn}\times m_{\rn-1}}$ denotes a random matrix whose elements are sampled iid from a distribution with mean $0$ and variance $\sigma_\rn^2$ $\forall i,j$, initialized as defined in (\ref{eq:normalization}). Then
\begin{equation}
\label{eq:mean_var_Q}
\E[\hQ_{ij}]=0, \qquad \var[\hQ_{ij}]=\omverm.
\end{equation}
\end{lemma}

\begin{proof}
By induction on $\tL$. Clearly for $\tL=1$: 
\begin{equation*}
\E[\hQ_{ij}]=\E[(Q_1)_{ij}]=0,\qquad \var[\hQ_{ij}]=\var[(Q_1)_{ij}]=\sigma_1^2.
\end{equation*}

Assume that (\ref{eq:mean_var_Q}) holds for $\tL-1$. Let $V=\prod_{\rn=\tL-1}^{1} Q_\rn$, $U=Q_\tL$. It follows that
\begin{equation*}
\E[\hQ_{ij}]=\E[(UV)_{ij}]= \sum_{k} \E[U_{ik}V_{kj}]= \sum_{k} \E[U_{ik}]\E[V_{kj}] = 0.
\end{equation*}
where the last transition follows from the independence of $U$ and $\hV$. In a similar manner
\begin{equation*}
\begin{split}
\var[\hQ_{ij}]&= \E[\hQ_{ij}^2] = \E[(\sum_{k} U_{ik}V_{kj})^2]=\E[\sum_{k} U_{ik}V_{kj}\sum_{l} U_{il}V_{lj}] =\sum_{k} \E[U_{ik}^2] \E[V_{kj}]^2\ \\
&= m_{\tL-1}\sigma_{\tL}^2\frac{1}{m_{\tL-1}}\prod_{\rn=1}^{\tL-1}m_\rn\cdot\sigma^2_\rn= \frac{1}{m_{\tL}}\prod_{\rn=1}^{\tL}m_\rn\cdot\sigma^2_\rn.
\end{split}
\end{equation*}

With the initialization scheme defined in (\ref{eq:normalization}), $\var(\hQ_{ij})=\omverm$.
\end{proof}

\begin{lemma}
\label{lem:mat-flip}
Let $C\in\R^{k\times m}$ and $V\in\R^{m\times k}$ both of rank $k$, where $k< m$. Then 
\begin{equation*}
CV = I ~\implies~C = CVV^{+}.
\end{equation*}
\end{lemma}

\begin{proof}
By definition $C=V^{+}$ , hence
\begin{equation*}
C=V^{+}V C = C VV^{+}.
\end{equation*}
\end{proof}


\section{Supplementary Proofs}
\label{app:suppl-proofs}

\subsection{Deep Linear networks}
\label{app:deep-thm1}

Here we prove \prop\ref{thm:1} as defined in Section~\ref{sec:deep}.

\textbf{Proposition~\ref{thm:1}.} \textit{
Let $G_r^{(t)}$ (Def.~\ref{def:err_mat}) denote the gradient matrix at time ${t}$. Let $A_{\ri}^{(t)}$ and $B_{\ri}^{(t)}$ denote the gradient scale matrices, which are defined in (\ref{eq:gradient-scale}). The compact representation $\hW^{(t)}$ obeys the following dynamics:
\begin{equation*}
\hW^{(t+1)}=\hW^{(t)} - \mu \sum_{\ri=1}^L \Al_{\ri}^{(t)} \cdot G_r^{(t)} \cdot \Bl_{\ri-1}^{(t)} +O(\mu^2).
\end{equation*}
}

\begin{proof}
At time ${t}$, the gradient step $\Delta W_l^{(t)}$ of layer $l$ is defined by differentiating $L(\iX)$ with respect to $W_l^{(t)}$. Henceforth we omit index ${t}$ for clarity. First, we rewrite $L(\iX)$ as follows:
\begin{equation*}
L(\iX;W_l)= \frac{1}{2}  \Vert \left ( \prod_{j=L}^{l+1}  W_j \right )  W_l \left (\prod_{j=l-1}^{1} W_j \right ) X-Y \Vert_F^2. \nonumber 
\end{equation*}
Differentiating $L(\iX;W_l)$ to obtain the gradient $\Delta W_l = \frac{\partial L(\iX;W_l)}{\partial W_l}$, using \lem\ref{lem:deriv} above, we get
\begin{equation}
\label{eq:delta_w}
\Delta W_l = \left ( \hspace{-1pt}\prod_{j=L}^{l+1}  W_j \hspace{-4pt} \right ) ^\top \hspace{-4pt}[\hW\SXX-\SYX]\hspace{-3pt}\left (\prod_{l-1}^{1} W_j \hspace{-4pt}\right )^\top.
\end{equation}
Finally,
\begin{equation}
\label{eq:delta_hW}
\Delta\hW = \prod_{l=L}^1 (W_l-\mu\Delta W_l) - \prod_{l=L}^1 W_l = - \mu \sum_{\ri=1}^L \left (\prod_{\rn=L}^{\rl+1}W_\rn \right )\Delta W_\rl \left(\prod_{\rn=\rl-1}^{1}W_\rn\right ) +O(\mu^2).
\end{equation}
Substituting $\Delta W_l$ and $G_r$ (as specified in Def.~\ref{def:err_mat}) into the above completes the proof.
\end{proof}

\subsection{Weight Evolution}
\label{app:sec:evolution}

The proofs in this section assume that the norm of the data matrix $X$ is bounded, that the norms of the initial random weight matrices $W^{(0)}_l$ are bounded $\forall l,\cm$, and that the norms of the gradient scale matrices are also bounded $\forall l,\cm$ in the relevant time range ${t}\leq\bar {t}$. The last assertion follows from our initial assumption, that the support of the distribution of the weight matrices is compact, and additionally that the norm of each element is bounded by $c\sigma_\rn^2$ for fixed a constant $c$.\footnote{These assumptions can be relaxed to having bounds with high probability, which follow from the analysis of random matrices in \appl\ref{app:random}, but for simplicity we assume fixed bounds.} 

We start by proving \thm\ref{thm:AlsBls}, which is stated in Section~\ref{sec:deep}.
\newline
\noindent
\textbf{Theorem~\ref{thm:AlsBls}.} \textit{Let $\hW$ denote the compact representation of a deep linear network, where $\cm=\min\left({m_1,...,m_{L-1}}\right)$ denotes the size of its smallest hidden layer and $\cm\geq \{m_0,m_L\}$. At each layer $l$, assume weight initialization $W^{(0)}_l$ obtained by sampling from a distribution with mean $0$ and variance $\sigma_l^2$, normalized as specified in Def.~\ref{def:1}. Let $\left(\Bl^{(t)}_\ri(\cm)\right)_{\cm=1}^\infty$ and $\left( \Al^{(t)}_\ri(\cm)\right)_{\cm=1}^\infty$ denote two sequences of gradient scale matrices as defined in (\ref{eq:gradient-scale}), where the $\cm^{th}$ element of each series corresponds to a network whose smallest hidden layer has $\cm$ neurons.  Let $\xrightarrow{p}$ denote element-wise convergence in probability as $\cm\to\ \infty$. Then $\forall {t},l$}:
\begin{equation*}
\begin{split}
&\Bl^{(t)}_\rl(\cm) \xrightarrow{p} I+O(\mu^2)~~\forall\rl\in[1\ldots \ttL-1],\qquad\Bl^{(t)}_{\ttL}(\cm) \xrightarrow{p} \frac{m_L}{\cm}I + O(\mu^2)\xrightarrow{p}O(\mu^2),\\
&\var[\Bl^{(t)}_\rl(\cm) ]= O(\mu^2)+\omverm \qquad\forall\rl,
\end{split}
\end{equation*}
and
\begin{equation*}
\begin{split}
&\Al^{(t)}_\rl(\cm) \xrightarrow{p} I+O(\mu^2)~~\forall\rl\in[1\ldots \ttL-1],\qquad\Al^{(t)}_{0}(\cm) \xrightarrow{p} \frac{m_0}{\cm}I + O(\mu^2)\xrightarrow{p}O(\mu^2),\\
&\var[\Al^{(t)}_\rl(\cm) ]= O(\mu^2)+\omverm \qquad\forall\rl.
\end{split}
\end{equation*}
\begin{proof}
We will only work out the detailed proof of the first assertion concerning $\Bl^{(t)}_\rl(\cm)$, as the proof of the second assertion is similar. Specifically, we prove by induction on time ${t}$ a stronger claim:
\begin{equation}
\label{eq:basic-thm-ext}
\begin{split}
&\Bl^{(t)}_\rl(\cm) \xrightarrow{p} I+O(\mu^2)~~\forall\rl\in[1\ldots \ttL-1],\qquad\Bl^{(t)}_{\ttL}(\cm) \xrightarrow{p} O(\mu^2)\\
&\var[\Bl^{(t)}_\rl(\cm) ]= O(\mu^2)+\omverm \qquad\forall\rl \\
&\E[{\hW}^{(t)}(\cm)]=O(\mu)+\omverm.
\end{split}
\end{equation}
For the corresponding series $\{\hW(\cm)\}$.

\paragraph{Base case $t=0$.}
From (\ref{eq:gradient-scale})
\begin{equation*}
\Bl_{\ri}^{(0)} \coloneqq \Big(\prod_{j=l}^{1} W_j^{(0)}\Big )^\top \Big(\prod_{j=l}^{1} W_j^{(0)}\Big ).
\end{equation*}
Recall that all weight matrices $\{W^{(0)}_\rl\}_{\rl=1}^\ttL$ are initialized by sampling from a distribution with mean $0$ and variance $\sigma_l^2 = O(\frac{1}{\cm})$. In \appl\ref{sec:random-mat} we prove some statistical properties of such matrices, and their corresponding \emph{gradient scale matrices} $A_\ri^{(0)}$ and $B_\ri^{(0)}$. This analysis culminates in \cor\ref{cor:alpha_beta} and \thm\ref{thm:var}, which can now be used directly to infer that 
\begin{equation*}
\begin{split}
&\E(\Bl^{(0)}_\rl)= I+ \omverm~\forall \rl\in[1\ldots \ttL-1],\qquad \E(\Bl^{(0)}_{\ttL})= \omverm\\
&\var[\Bl_{\ri}^{(0)}(\cm) ]= \omverm \forall l.
\end{split}
\end{equation*}
Finally, \thm\ref{thm:cheby} further proves that such matrices also satisfy
\begin{equation*}
\Bl^{(0)}_\rl(\cm) \xrightarrow{p} I~~\forall\rl\in[1\ldots \ttL-1],\qquad\Bl^{(0)}_{\ttL}(\cm) \xrightarrow{p} 0.
\end{equation*}
Similarly, from \lem\ref{lem:mean_var_Q} we have that $\E[{\hW}^{(0)}(\cm)]=0~\forall \cm$, and we therefore conclude that the assertion in (\ref{eq:basic-thm-ext}) is true at $t=0$. 

\paragraph{Induction step.}
In \appl\ref{sec:random-mat-dyn} we analyze random matrices which are defined similarly to the gradient scale matrices, and whose dynamics correspond to the update rule defined in (\ref{eq:canon_deep_linear}). The main result is stated in \thm\ref{thm:delB}, which can be used directly now to show that if assertion (\ref{eq:basic-thm-ext}) holds for $B_{\ri}^{(t)}$ and ${\hW}^{(t)}(\cm)$, it also holds for $B_{\ri}^{(t+1)}$ and ${\hW}^{(t+1)}(\cm)$. 

First, let us verify that the conditions of \thm\ref{thm:delB} are met. By construction, the dynamics captured by (\ref{eq:deltaQ})-(\ref{eq:U-dyn}) describes the dynamics of $\{W_l\}$, a result that follows from the proof of \prop\ref{thm:1}, where specifically (\ref{eq:delta_w})-(\ref{eq:delta_hW}) imply (\ref{eq:deltaQ})-(\ref{eq:Ql}). In addition, as by assumption $\cm\geq m_0$ and $m_0=q$, $\Bl^{(t)}_\rl(\cm)\in\R^{q\times q}$ is full rank $\forall l,\cm$ with probability 1. We may therefore use \thm\ref{thm:delB} to conclude, based on the induction assumption, that
\begin{equation*}
\begin{split}
&\Delta \Bl^{(t)}_\rl(\cm) \xrightarrow{p}  O(\mu^2),\quad \var[\Delta\Bl^{(t)}_\rl(\cm) ]= O(\mu^2)+\omverm \qquad\forall\rl\\
&\E[\Delta\hW^{(t)}(\cm)]=O(\mu)+\omverm.
\end{split}
\end{equation*}
Noting that $\Bl^{(t+1)}_\rl(\cm)=\Bl^{(t)}_\rl(\cm)+\Delta \Bl^{(t)}(\cm)~\forall\rl$ and $\hW^{(t+1)}(\cm)=\hW^{(t)}(\cm)+\Delta \hW^{(t)}_\rl(\cm)$, the assertion in (\ref{eq:basic-thm-ext}) follows for all $t$.
\end{proof}

We proceed to prove \thm\ref{thm:convergence-rate}, which is stated in Section~\ref{sec:evolution}. 
\newline
\noindent
\textbf{Theorem~\ref{thm:convergence-rate}.} \textit{Let $\bw_j^{(t)}$ denote the $j^\mathrm{th}$ column of the compact representation matrix $\hW^{(t)}$, and $\bw_j^{opt}$ the $j^\mathrm{th}$ column of the optimal solution of (\ref{eq:multi-loss}). Assume that the data is rotated to its principal coordinate system, and let $d_j$ denote the $j^\mathrm{th}$ singular value of the data. Then there exists $\hat {t}$ such that $\forall \delta, \varepsilon$ and $\forall {t}\leq\hat {t}$, $\exists \hat \cm,\hat\mu$ such that $\forall \mu<\hat\mu, \cm\geq\hat\cm$
\begin{equation*}
Prob\bigg (\Big\Vert\bw_j^{(t+1)}- \big [ \lambda_j^{t} \bw_j^{(0)} + [1- \lambda_j^{t}]\bw_j^{opt}\big ]\Big\Vert < \varepsilon\bigg ) > (1-\delta), \quad\quad \lambda_j = 1-\mu d_j L.
\end{equation*}
}

\begin{proof}
First let us derive a bound on the total gradient magnitude $\Vert G_r^{(t)}\Vert$. Since by assumption the norm of the data $\Vert X\Vert$ and the norms of $W^{(0)}_l~\forall l$ are bounded, it follows that the loss at time $t=0$ is also bounded. Let $U_1$ denote this bound, and $U_2$ denote the data bound $\Vert X\Vert^2\leq U_2$. Let $L^{(t)}(\iX)$ denote the loss at time ${t}$. Since the loss is decreasing, it follows that $L^{(t)}(X,Y)\leq U_1~\forall {t}$, and therefore
\begin{equation*}
\Vert G_r^{(t)}\Vert^2=\Vert [\hW^{(t)} X -Y]X^\top\Vert^2 \leq 2L^{(t)}(\iX) \Vert X\Vert^2\leq 2U_1 U_2.
\end{equation*}

We now introduce the notation $\Bl^{(t)}_\rl = I+O(\mu^2)+\Delta^{(t)}_{B\rl}$ and $\Al^{(t)}_\rl = I+O(\mu^2)+\Delta^{(t)}_{A\rl}$. Let $U_3$ denote a tight bound so that  $\{\Vert \Delta^{(t)}_{A\rl}\Vert^2,\Vert \Delta^{(t)}_{B\rl}\Vert^2\}\leq U_3~\forall \rl,{t}\leq\bar{t}$, where we further assume that $U_3\leq 1$ (this last assumption will be justified later). Starting from (\ref{eq:canon_deep_linear})
\begin{align}
\hW^{(t+1)} =&\hW^{(t)} - \mu \sum_{\ri=1}^L A_{\ri}^{(t)} \cdot G_r^{(t)} \cdot B_{\ri-1}^{(t)} +O(\mu^2) \nonumber \\
=& \hW^{(t)} - \mu \sum_{\ri=1}^L [I+O(\mu^2)+\Delta^{(t)}_{A\rl}] ~ G_r^{(t)} ~[I+O(\mu^2)+\Delta^{(t)}_{B\rl}] +O(\mu^2)\label{eq:234}\\
=& \hW^{(t)} - \mu \left [LG_r^{(t)}  +\bigg (\sum_{\ri=1}^L \Delta^{(t)}_{A\rl}\bigg) G_r^{(t)} +G_r^{(t)} \sum_{\ri=1}^L \Delta^{(t)}_{B\rl}  + \sum_{\ri=1}^L \Delta^{(t)}_{A\rl}~G_r^{(t)}~\Delta^{(t)}_{B\rl}\right ] +O(\mu^2).\nonumber 
\end{align}
The last transition is valid because the norms of $G_r^{(t)},\Delta^{(t)}_{A\rl},\Delta^{(t)}_{B\rl}$ are bounded. It follows that
\begin{equation*}
\Vert\hW^{(t+1)} -[\hW^{(t)} - \mu LG_r^{(t)} ]\Vert^2 \leq \mu L2U_1 U_2[U_3+U_3+U_3^2] +O(\mu^2)\leq \mu 6LU_1 U_2U_3 +O(\mu^2).
\end{equation*}
Thus $\exists\hat\mu$ such that $\forall \mu<\hat\mu$
\begin{equation}
\label{eq:bound}
\Vert\hW^{(t+1)} -[\hW^{(t)} - \mu LG_r^{(t)} ]\Vert^2\leq \hat\mu12L U_1 U_2U_3.
\end{equation}

From \thm\ref{thm:AlsBls}, $\forall {t}$ and $\forall l\leq\ttL-1$, $\Bl^{(t)}_\rl(\cm) \xrightarrow{p} I+O(\mu^2)$ and $\Al^{(t)}_\rl(\cm) \xrightarrow{p} I+O(\mu^2)$ . In other words, if we fix $\hat {t}$ and consider all the iterations leading to $\hat {t}$, then $\forall\varepsilon,\delta>0~\exists \hat\cm\in\mathbb{N}$, such that $\forall l,\cm>\hat\cm,{t}\leq\hat {t}$
\begin{equation*}
\begin{split}
&Prob\left(\Big \vert\Bl^{(t)}_\rl(\cm) -[I+O(\mu^2)]\Big\vert^2\leq \min\bigg\{\frac{\varepsilon}{\hat\mu 12 LU_1 U_2},1\bigg\} \right)>(1-\delta) \qquad \mathrm{element\mbox{-}wise}\\
&Prob\left(\Big\vert\Al^{(t)}_\rl(\cm) -[I+O(\mu^2)]\Big\vert^2\leq\min\bigg\{\frac{\varepsilon}{\hat\mu 12 LU_1 U_2},1\bigg\} \right)>(1-\delta) \qquad \mathrm{element\mbox{-}wise} \\
&\implies \qquad Prob\left(U_3\leq \min\bigg\{\frac{\varepsilon}{\hat\mu 12 LU_1 U_2},1\bigg\} \right) >(1-\delta). \texp{\mathit{U}_3~is~tight}
\end{split}
\end{equation*}
Finally, from (\ref{eq:bound}) and $\forall \mu<\hat\mu$
\begin{equation}
\label{eq:ineq}
Prob\left(\Vert\hW^{(t+1)} -[\hW^{(t)} - \mu LG_r^{(t)} ]\Vert^2\leq \varepsilon \right)>(1-\delta).
\end{equation}
Next, we evaluate
\begin{equation}
\label{eq:update-approx}
\tilde\hW^{(t+1)} = \hW^{(t)} - \mu LG_r^{(t)}. 
\end{equation}
We first shift to the principal coordinate system defined in Def~\ref{def:canon}. In this representation $G_r^{(t)} = W^{(t)} D-M$, where $D=diag(\{d_j\}_{j=1}^q)$ is a diagonal matrix whose elements are the principal eigenvalues of the data $\{d_j\}_{j=1}^q$, arranged in decreasing order. Since $D$ is diagonal, (\ref{eq:update-approx}) can be written separately for each column of $\tilde\hW^{(t)}$, and we get for each column $\bw_j^{(t)}$
\begin{equation*}
\tilde\bw_j^{(t+1)}= \bw_j^{(t)} - \mu L   [\bw_j^{(t)} d_j-\bmm_j], \qquad j\in K.
\end{equation*}
This is a telescoping series, whose solution is
\begin{equation*}
\begin{split}
\tilde\bw_j^{(t+1)}&=(1-\mu L d_j )\bw_j^{(t)} + \mu L \bmm_j = \ldots \\
&= (1-\mu L d_j )^{t} \bw_j^{(0)} + \mu L d_j\left [\sum_{\nu=1}^{{t}}   (1-\mu L d_j )^{\nu-1}\right ] \frac{\bmm_{j}}{d_j} \\
&= \lambda_j^{t} \bw_j^{(0)} + [1- \lambda_j^{t}]\frac{\bmm_{j}}{d_j}, \qquad \lambda_j = 1-\mu L d_j. 
\end{split}
\end{equation*}
As $\bw_j^{opt}=\frac{\bmm_{j}}{d_j}$ and using (\ref{eq:ineq}), the assertion in the theorem follows.

\end{proof}

We conclude by proving \thm\ref{thm:teles}. To this end we introduce another notation, $A^{(t)}=\sum_{\ri=1}^L A_{\ri}^{(t)}$, and let $U_4$ denote the bound on $\Al_\rl$ where $\Vert A_{\ri}^{(t)}\Vert^2\leq U_4~~\forall \rl, {t}\leq\bar{t}$.\\
\noindent
\textbf{Theorem~\ref{thm:teles}.} \textit{
Let $\bw_j^{(t)}$ denote the $j^\mathrm{th}$ column of the compact representation matrix $\hW^{(t)}$, and $\bw_j^{opt}$ the $j^\mathrm{th}$ column of the optimal solution of (\ref{eq:multi-loss}). Assume that the data is rotated to its principal coordinate system, and let $d_j$ denote the $j^\mathrm{th}$ singular value of the data. Then there exists $\breve{t}$ such that $\forall \delta, \varepsilon$ and $\forall {t}\leq\breve{t}$, $\exists \breve \cm,\breve\mu$ such that $\forall \mu<\breve\mu, \cm\geq\breve\cm$
\begin{equation*}
\begin{split}
Prob\Bigg (\bigg\Vert \bw_j^{(t+1)} -&\bigg [ \prod_{\nunu=1}^{{t}}(I-\mu d_j A^{(\nunu)})\bw_j^{(0)}
+ \\
&\mu d_j\Big (\sum_{\nunu=1}^{{t}}  \prod_{\rhorho=\nunu+1}^{{t}} (I-\mu d_j A^{(\rhorho )})A^{(\nunu)}\Big ) \bw_j^{opt} \bigg ]\bigg\Vert < \varepsilon\Bigg ) > (1-\delta).
\end{split}
\end{equation*}
}
\begin{proof}
We use the same bound notations as in the proof of \thm\ref{thm:convergence-rate}, but where only $\Vert \Delta^{(t)}_{B\rl}\Vert^2\leq U_3~\forall \rl,{t}\leq\bar{t}$ a tight bound, $U_3\leq 1$. Similarly to (\ref{eq:234}),
\begin{equation*}
\begin{split}
&\hW^{(t+1)} = \hW^{(t)} - \mu \left [\bigg (\sum_{\ri=1}^L \Al^{(t)}_\rl\bigg) G_r^{(t)}  + \sum_{\ri=1}^L \Al^{(t)}_\rl~G_r^{(t)}~\Delta^{(t)}_{B\rl}\right ] +O(\mu^2)\\ \implies\qquad & \Vert\hW^{(t+1)} -[\hW^{(t)} - \mu A^{(t)}G_r^{(t)} ]\Vert^2\leq \mu  2LU_1 U_2 U_4+O(\mu^2). 
\end{split}
\end{equation*}
From \thm\ref{thm:AlsBls}, $\forall\varepsilon,\delta>0~\exists \breve \cm,\breve{t},\breve\mu$, such that $\forall \cm\geq\breve\cm,{t}\leq\breve{t},\mu\leq\breve\mu$
\begin{align}
&Prob\left(U_3\leq \min\bigg\{\frac{\varepsilon}{\breve\mu 4 L U_1 U_2 U_4},1\bigg\} \right) >(1-\delta)\nonumber \\
\implies \quad &Prob\left(\Vert\hW^{(t+1)} -[\hW^{(t)} - \mu A^{(t)} G_r^{(t)} ]\Vert^2\leq \varepsilon \right) >(1-\delta). \label{eq:thm7-ed}
\end{align}
As before, we evaluate
\begin{equation*}
\tilde\hW^{(t+1)} = \hW^{(t)} - \mu A^{(t)}  G_r^{(t)}.
\end{equation*}
In the principal coordinate system this expression is separable by columns, thus
\begin{equation*}
\tilde\bw_j^{(t+1)}=\bw_j^{(t)} - \mu \sum_{\ri=1}^L A_{\ri}^{(t)} (d_j \bw_j^{(t)} -\bmm_j), \qquad j\in[K].
\end{equation*}
Once again we have a telescoping series, whose solution is
\begin{equation*}
\tilde\bw_j^{(t+1)}= \prod_{\nunu=1}^{{t}}\big(I-\mu d_j A^{(\nunu)}\big)\bw_j^{(0)}+ \mu d_j\Big [\sum_{\nunu=1}^{{t}}  \prod_{\rhorho=\nunu+1}^{{t}}\big (I-\mu d_j A^{(\rhorho )}\big)A^{(\nunu)}\Big ] \frac{\bmm_j}{d_j}.
\end{equation*}
As $\bw_j^{opt}=\frac{\bmm_{j}}{d_j}$ and using (\ref{eq:thm7-ed}), the assertion in the theorem follows.
\end{proof}

\subsection{Adding Non-Linear ReLU Activation}
\label{app:relu}

In this section, we analyze the two-layer model with ReLU activation, where only the weights of the first layer are being learned \citep{arora2019fine}. Similarly to (\ref{eq:multi-loss}), the loss is defined as
\begin{equation*}
W^* =  \argmin\limits_W\frac{1}{2}\sum_{i=1}^n \Vert f(\bx_i)-\by_i \Vert^2, \qquad f(\bx_i) = \ba^\intercal \cdot \sigma(W \bx_i), ~~ \ba\in\R^{\cm},~ W\in\R^{\cm\times d}.
\end{equation*}
where $\cm$ denotes the number of neurons in the hidden layer and $\ba$ a fixed vector. We consider a binary classification problem with 2 classes, where $y_i=1$ for $\bx_i\in C_1$, and $y_i=-1$ for $\bx_i\in C_2$. $\sigma(\cdot)$ denotes the ReLU activation function applied element-wise to vectors, where $\sigma(u) =u$ if $u\geq 0$, and $0$ otherwise.  

At time ${t}$, each gradient step is defined by differentiating the loss with respect to $W$. Due to the non-linear nature of the activation function $\sigma(\cdot)$, we separately\footnote{Since the ReLU function is not everywhere differentiable, the following may be considered the definition of the update rule.} differentiate each row of $W$, denoted $\bw_r$ where $r\in[\cm]$, as follows:

\begin{equation*}
\begin{split}
\bw^{(t+1)}_r &- \bw^{(t)}_r = -\mu \frac{\partial L(\iX)}{\partial \bw_r}{\bigg |}_{{\textstyle\mathstrut} {\bw_r = \bw^{(t)}_r}} = -\mu\sum_{i=1}^n \Big [ \ba^\intercal\cdot \sigma(W^{(t)} \bx_i)-y_i\Big ] \frac{\partial f(\bx_i)}{\partial \bw_r}{\bigg |}_{{\textstyle\mathstrut} \bw_r = \bw^{(t)}_r}\\
&= -\mu \sum_{i=1}^n \Big [\sum_{j=1}^\cm a_j\sigma(\bw^{(t)}_j\cdot\bx_i)-y_i\Big ] a_r \bx_i^\intercal \mathbbm{1}_{{\textstyle\mathstrut} {\bw_r^{(t)}}} (\bx_i) \\
&= -\mu a_r\sum_{i=1}^n \mathbbm{1}_{{\textstyle\mathstrut} {\bw_r^{(t)}}} (\bx_i) \Big [\Psi^{(t)}(\bx_i) \cdot\bx_i-y_i\Big ]  \bx_i^\intercal,  \qquad\qquad \textrm{where}~\Psi^{(t)}(\bx_i) = \sum_{j=1}^\cm a_j \bw^{(t)}_j \mathbbm{1}_{{\textstyle\mathstrut} {\bw_j^{(t)}}} (\bx_i).
\end{split}
\end{equation*}
Above $\mathbbm{1}_{{\textstyle\mathstrut} {\bw_r^{(t)}}}(\bx_i)$ denotes the indicator function that equals 1 when $\bw^{(t)}_r\cdot\bx_i\geq 0$, and $0$ otherwise. 

To proceed, we make two assumptions:
\begin{enumerate}
    \item Each point $\bx_i$ is drawn from a symmetric distribution $\mathcal{D}$ with density $f_\mathcal{D}(\bX)$, such that: $f_\mathcal{D}(\bx_i)=f_\mathcal{D}(-\bx_i)$.
    \item $W$ and $\ba$ are initialized so that $\bw^{(0)}_{{\textstyle\mathstrut}2i}\hspace{-3pt}=\hspace{-2pt} - \bw^{(0)}_{{\textstyle\mathstrut}2i-1}$ and $a_{{\textstyle\mathstrut}2i} \hspace{-2pt}=\hspace{-2pt} - a_{{\textstyle\mathstrut}2i-1}$ $\forall i\in[\frac{\cm}{2}]$.
\end{enumerate}
\textbf{Theorem~\ref{thm:dynamics_relu}.}
Retaining the assumptions stated above, at the beginning of the learning, the temporal dynamics of the model can be shown to obey the following update rule:
\begin{equation*}
W^{(t+1)} \approx W^{(t)}-\mu \frac{1}{2}\Big [ (\ba\ba^\intercal) W^{(t)}\SXX - \tilde M^{(t)} \Big ].
\end{equation*}
Above $\tilde M^{(t)}$ denotes the difference between the centroids of the 2 classes, computed in the half-space defined by $\bw^{(t)}_r\cdot\bx\geq 0$.

\begin{proof}
It follows from Assumption~2 that at the beginning of training $\mathbbm{1}_{{\textstyle\mathstrut} {\bw_{2j}^{(0)}}} (\bx_i)+\mathbbm{1}_{{\textstyle\mathstrut} {\bw_{2j-1}^{(0)}}} (\bx_i)=1$, $\forall \bx_i$ such that $\bw_{{\textstyle\mathstrut}2j-1}\bx_i \neq  \bw_{{\textstyle\mathstrut}2j}\bx_i \neq 0$, and $\forall j\in[\frac{\cm}{2}]$. Consequently
\begin{equation*}
\Psi^{(0)}(\bx_i) = \sum_{j=1}^\cm a_j \bw^{(0)}_j \mathbbm{1}_{{\textstyle\mathstrut} {\bw_j^{(t)}}} (\bx_i) = \frac{1}{2} \sum_{j=1}^\cm a_j \bw^{(0)}_j = \frac{1}{2}\ba^\intercal W^{(0)}.
\end{equation*}
$\forall \bx_i$ such that $\bw_{{\textstyle\mathstrut}2j-1}\bx_i \neq  \bw_{{\textstyle\mathstrut}2j}\bx_i \neq 0$. Finally
\begin{equation*}
\bw^{(1)}_r - \bw^{(0)}_r = -\mu a_r\big [ \frac{1}{2}\ba^\intercal W^{(0)}\hspace{-6pt}\sum_{\substack{i=1\\ \bw^{(0)}_r\bx_i\geq 0}}^n \hspace{-6pt} \bx_i  \bx_i^\intercal - \hspace{-6pt}\sum_{\substack{i=1\\ \bw^{(0)}_r\bx_i\geq 0}}^n \hspace{-6pt} y_i  \bx_i^\intercal \big ].
\end{equation*}

Next, we note that Assumption~1 implies
\begin{equation*}
\iE[\hspace{-6pt}\sum_{\substack{i=1\\ {\bw\cdot\bx_i\geq 0}}}^n \hspace{-6pt} \bx_i  \bx_i^\intercal ] = \frac{1}{2} \iE[\sum_{i=1}^n  \bx_i  \bx_i^\intercal ] = \frac{1}{2} \iE[\SXX].
\end{equation*}
for any vector $\bw$. Thus, if the sample size $n$ is large enough, at the beginning of training we expect to see
\begin{equation*}
\bw^{(t+1)}_r - \bw^{(t)}_r \approx -\mu \frac{a_r}{2} [ \ba^\intercal W^{(t)}\SXX - \tilde\bmm^{(t)}_r  ],\quad \forall r.
\end{equation*}
where row vector $\tilde\bmm^{(t)}_r$ denotes the vector difference between the centroids of classes $C_1$ and $C_2$, computed in the half-space defined by $\bw^{(t)}_r\cdot\bx\geq 0$. Finally (for small ${t}$)
\begin{equation*}
W^{(t+1)} - W^{(t)} \approx -\mu \frac{1}{2}\Big [ (\ba\ba^\intercal) W^{(t)}\SXX - \tilde M^{(t)} \Big ].
\end{equation*}
where $\tilde M^{(t)}$ denotes the matrix whose $r$-th row is $a_r \tilde\bmm^{(t)}_r$. This equation is reminiscent of the single-layer linear model dynamics $\hW^{(t+1)}=\hW^{(t)} - \mu G_r^{(t)}$, and we may conclude that when it holds and using the principal coordinate system, the rate of convergence of the $j$-th column of $W^{(t)}$ is governed by the singular value $d_j$ .
\end{proof}

\section{Additional Empirical Results}
\label{app:implication}

\subsection{Weight Initialization}
\label{app:sec:weight-initialization}

We evaluate empirically the weight initialization scheme from Def.~\ref{def:1} and (\ref{eq:normalization}). When compared to Glorot uniform initialization, the only difference between the two schemes lies in how the first and last layers are scaled. Thus, to highlight the difference between the methods, we analyze a fully connected linear network with a single hidden layer, whose dimension (the number of hidden neurons) is much larger than the input and output dimensions. We trained $N$=$10$ such networks on a binary classification problem, once with the initialization suggested in (\ref{eq:normalization}), and again with Glorot uniform initialization. While both initialization schemes achieve the same final accuracy upon convergence, our proposed initialization variant converges faster on both train and test datasets (see Fig.~\ref{fig:different_initlizations}).

\begin{figure}[thb!]
\begin{minipage}{.6\textwidth}
\begin{subfigure}{1\textwidth}
  
    \begin{subfigure}{.48\textwidth}
      \centering
      \includegraphics[width=1\linewidth]{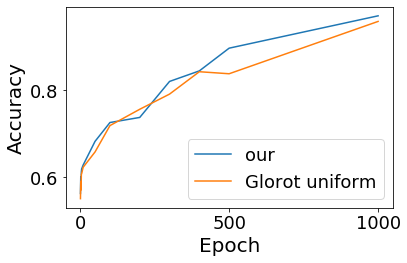}
      \caption{train}
    \end{subfigure}
    \begin{subfigure}{.48\textwidth}
      \centering
      \includegraphics[width=1\linewidth]{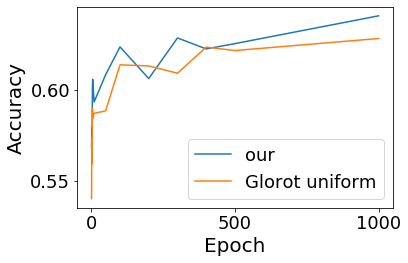}
      \caption{test}
    \end{subfigure}
\end{subfigure}
\caption{Learning curves of a fully connected linear network with one hidden layer, trained on the dogs and cats dataset, and initialized by either Glorot uniform initialization (orange), or the initialization proposed in (\ref{eq:normalization}) (blue).}
\label{fig:different_initlizations}
\end{minipage}
\hspace{0.2cm}
\begin{minipage}{.4\textwidth}
    \begin{subfigure}{1\textwidth}
      \centering
      \includegraphics[width=.8\linewidth]{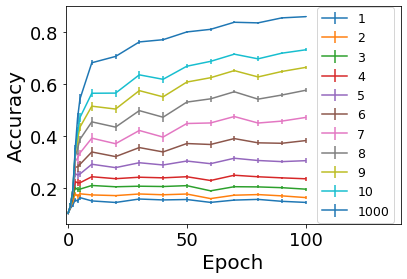}
    \end{subfigure}
    \caption{Evaluations on test-sets projected to the first $P$ principal components, for different values of $P$ (see legend) of $10$ VGG-19 models trained on CIFAR-10}
      \label{app:fig:acc_on_pca_components}
\end{minipage}
\end{figure}


\subsection{Divergence of Gradient Scale Matrices}
\label{app:divergence}

We now discuss some additional results, complementary to Section~\ref{sec:evolution}. In Fig.~\ref{fig:gradient_scale_validation}, we visualize how the gradient scale matrices $B_\ri^{(t)}(\cm)$ remain approximately equal to $I$ much longer than $A_\ri^{(t)}(\cm)$, for a specific layer in a 5-layered model. In Fig.~\ref{fig:app:gradient_scale_validation} we show a similar trend in all 5 layers of the network. Note that in the input and output layers, $B_\ri^{(t)}(\cm)$ and $A_\ri^{(t)}(\cm)$ are $I$ by definition, and hence their convergence is self-evident.

\begin{figure}[thb!]
\begin{subfigure}{.23\textwidth}
  \centering
  \includegraphics[width=1\linewidth]{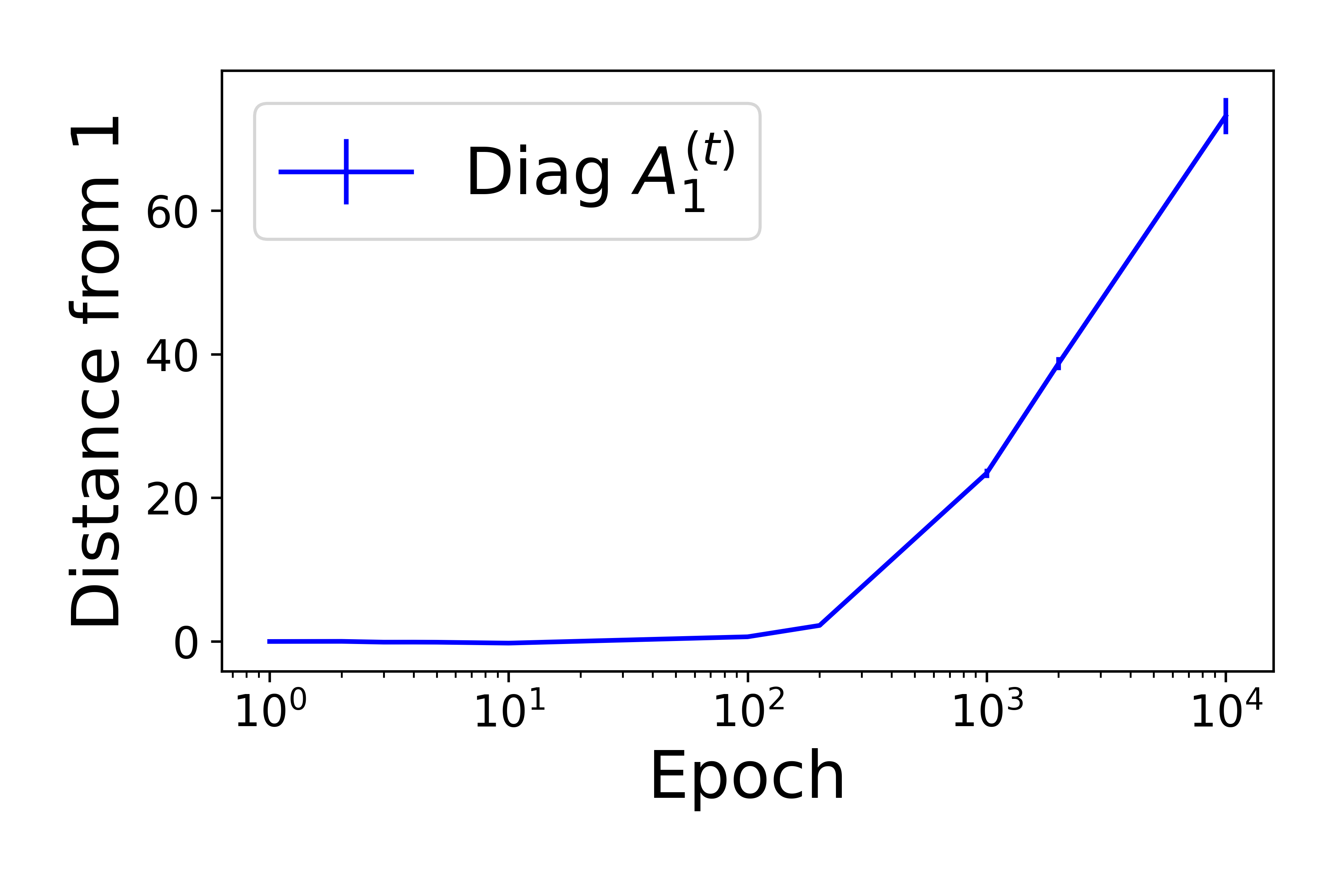}
 \caption{Diagonal, layer 1}
\end{subfigure}
\begin{subfigure}{.23\textwidth}
  \centering
  \includegraphics[width=1\linewidth]{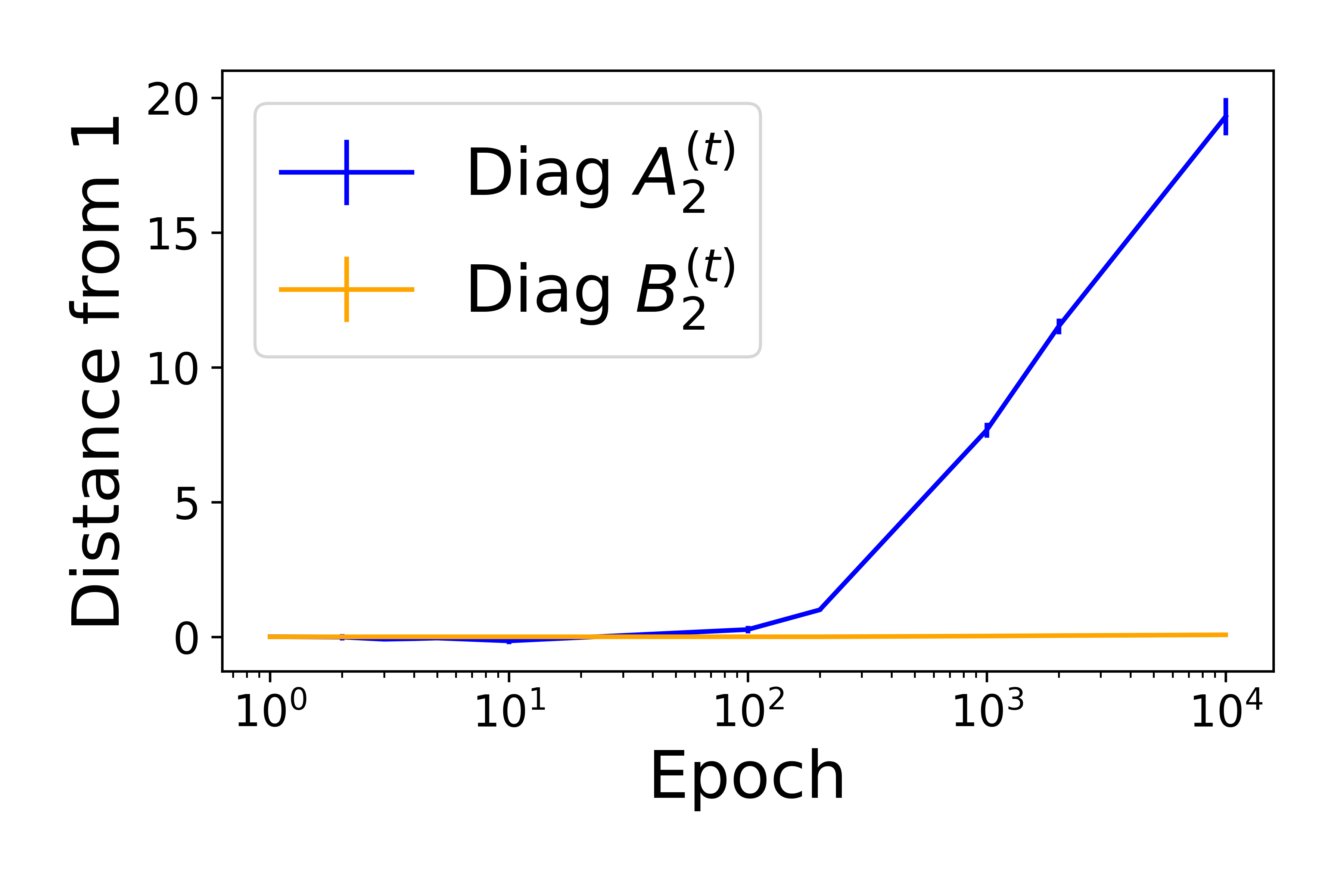}
 \caption{Diagonal, layer 2}
\end{subfigure}
\begin{subfigure}{.23\textwidth}
  \centering
  \includegraphics[width=1\linewidth]{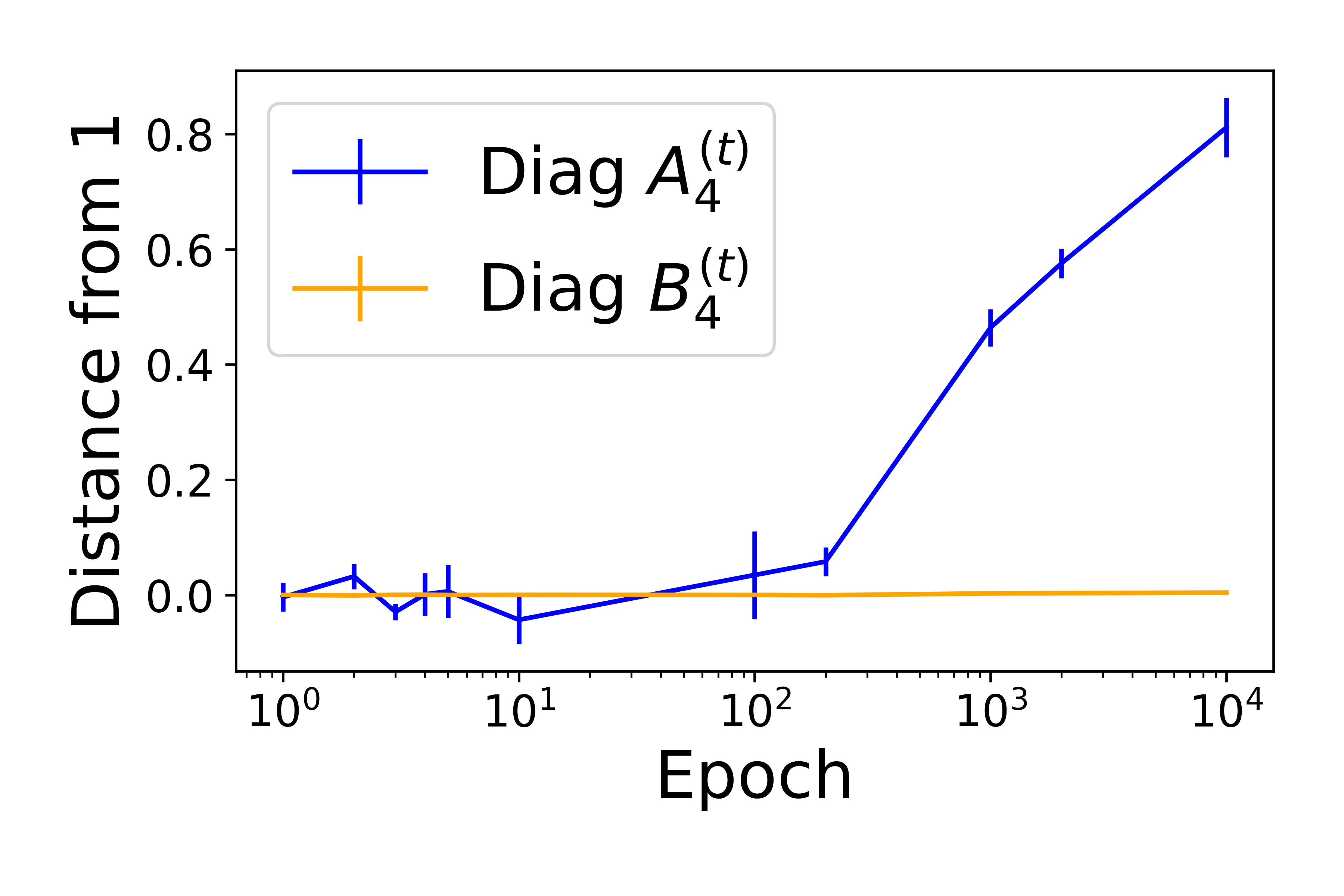}
 \caption{Diagonal, layer 4}
\end{subfigure}
\begin{subfigure}{.23\textwidth}
  \centering
  \includegraphics[width=1\linewidth]{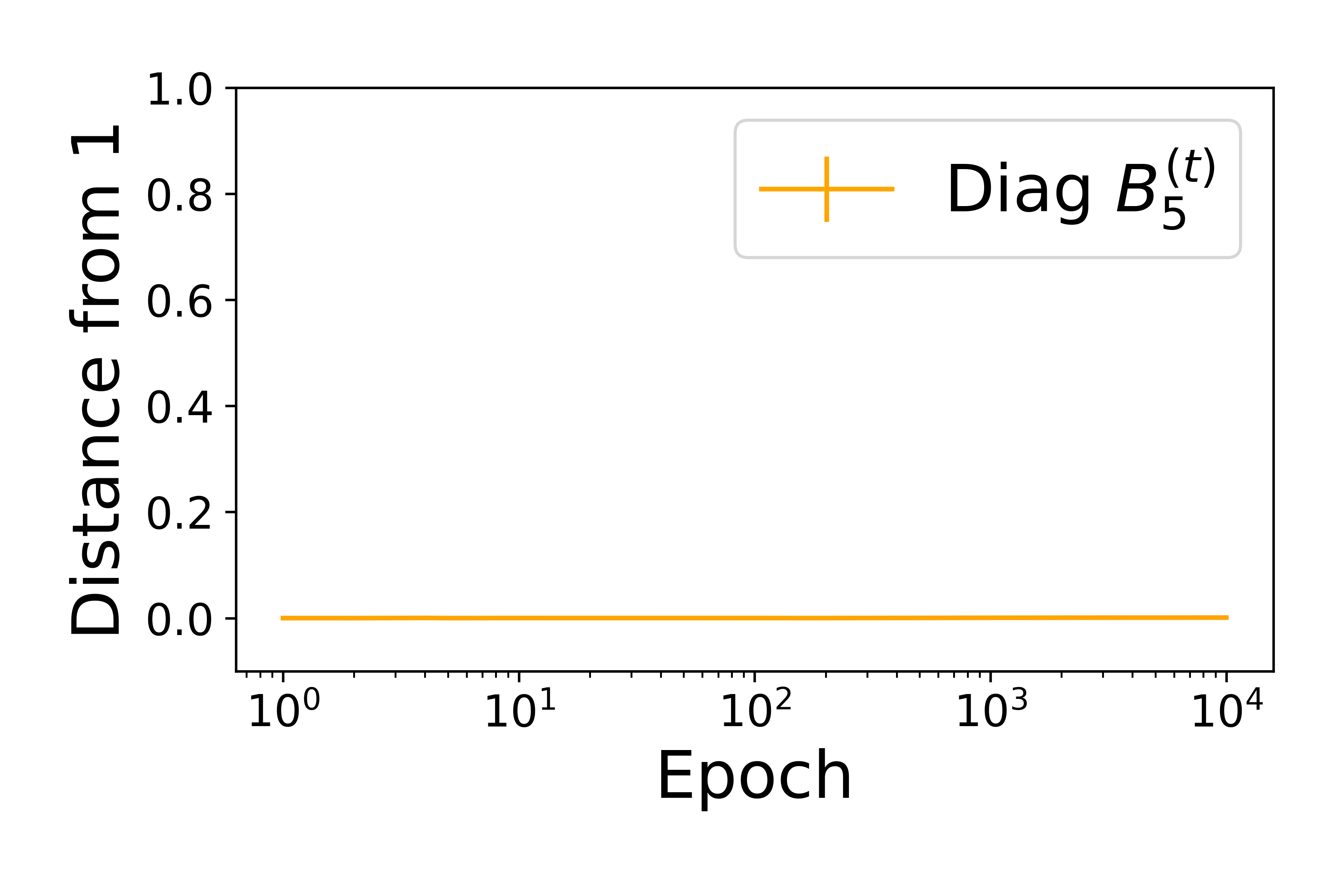}
 \caption{Diagonal, layer 5}
\end{subfigure}

\begin{subfigure}{.23\textwidth}
  \centering
  \includegraphics[width=1\linewidth]{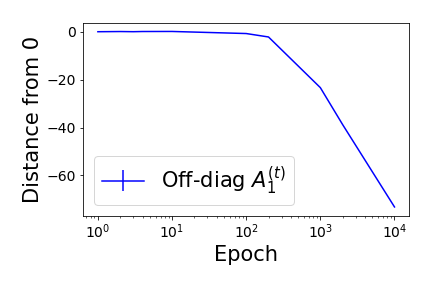}
 \caption{Off-diagonal, layer 1}
\end{subfigure}
\begin{subfigure}{.23\textwidth}
  \centering
  \includegraphics[width=1\linewidth]{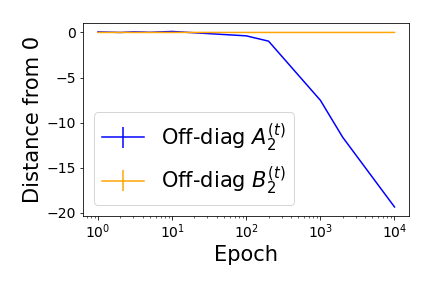}
 \caption{Off-diagonal, layer 2}
\end{subfigure}
\begin{subfigure}{.23\textwidth}
  \centering
  \includegraphics[width=1\linewidth]{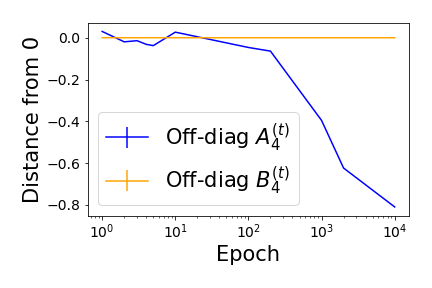}
 \caption{Off-diagonal, layer 4}
\end{subfigure}
\begin{subfigure}{.23\textwidth}
  \centering
  \includegraphics[width=1\linewidth]{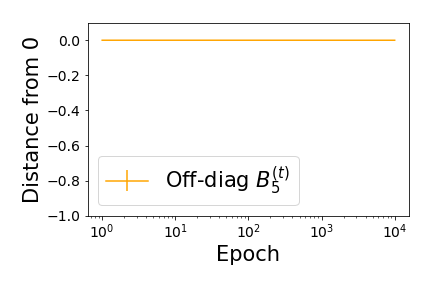}
 \caption{Off-diagonal, layer 5}
\end{subfigure}
\caption{Similar to Fig.~\ref{fig:gradient_scale_validation}, showing the dynamics of $A^{(t)}_l$ and $B_l^{(t)}$ in additional layers, when training $10$ $5$-layered linear networks on the small mammals dataset. (a-d) The empirical $L_2$-distance of the diagonal elements of $A^{(t)}_3$ and $B_3^{(t)}$ from their analytical value of $\alpha_l^{(t)}$ and $\beta_l^{(t)}$---$diag\left(A^{(t)}_l-\alpha_l^{(t)}I\right)$ and  $diag\left(B_l^{(t)} - \beta_l^{(t)}I\right)$---respectively. (e-h) The empirical $L_2$-distance of the off-diagonal elements of $A^{(t)}_l$ and $B_l^{(t)}$ from 0. The networks reach their maximal test accuracy in epoch $t=100$, before the divergence of $A^{(t)}_l$.}
\label{fig:app:gradient_scale_validation}
\end{figure}

\subsection{Empirical Validation of the PC-bias on Original Model}

\subsubsection[]{Models with $L_2$ Loss}
\label{app:sec:stds_square_loss}
In Section~\ref{sec:results}, we relaxed some of the theoretical assumptions while pursuing the empirical investigation, in order to match more commonly used models. Specifically, we changed the initialization to the commonly used Glorot initialization, replaced the $L_2$ loss with the cross-entropy loss, and employed SGD instead of the deterministic GD.
As can be expected from the theoretical results, without this relaxation the PC-bias becomes even more pronounced. To support this claim, we repeated the experiments of Section~\ref{sec:validation} and Fig.~\ref{fig:stds_and_distances_all_cases} with the assumptions of the theoretical analysis, showing the results in Fig.~\ref{app:fig:stds_cats_and_dogs_2_matrices_square_loss}. 

\begin{figure}[h!]
\begin{subfigure}{.16\textwidth}
      \centering
     \includegraphics[width=1\linewidth]{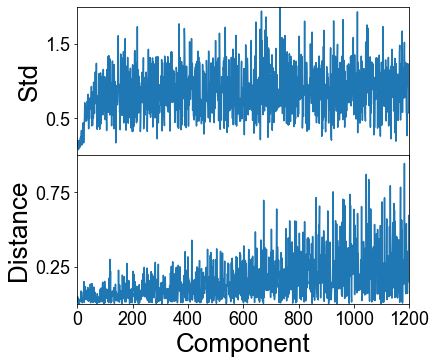}
    \end{subfigure}
    \begin{subfigure}{.16\textwidth}
      \centering
      \includegraphics[width=1\linewidth]{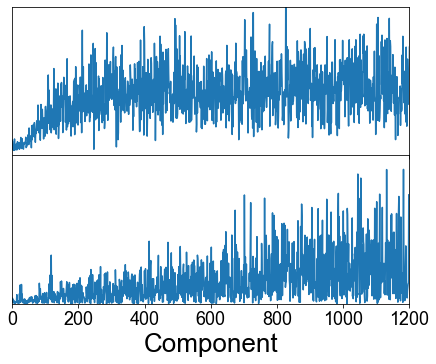}
    \end{subfigure}
    \begin{subfigure}{.16\textwidth}
      \centering
      \includegraphics[width=1\linewidth]{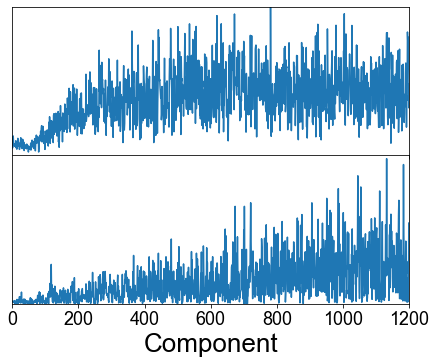}
    \end{subfigure}
    \begin{subfigure}{.16\textwidth}
      \centering
      \includegraphics[width=1\linewidth]{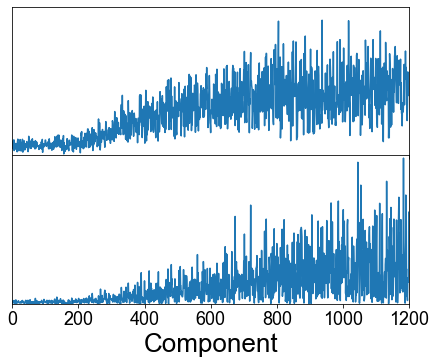}
    \end{subfigure}
    \begin{subfigure}{.16\textwidth}
      \centering
      \includegraphics[width=1\linewidth]{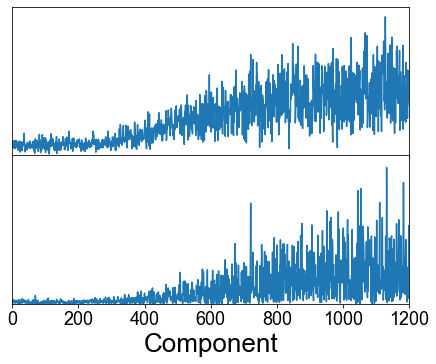}
    \end{subfigure}
    \begin{subfigure}{.16\textwidth}
      \centering
      \includegraphics[width=1\linewidth]{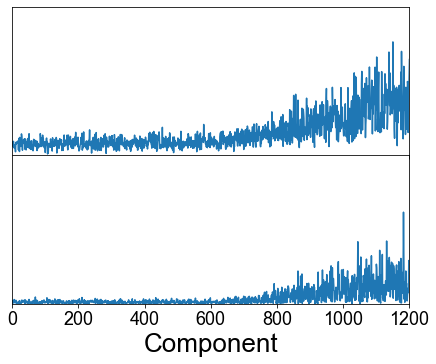}
    \end{subfigure}
\caption{Same as Fig.~\ref{fig:stds_and_distances_all_cases}, using the $L_2$ loss and the initialization scheme proposed in Def.~\ref{def:1}. Epochs plotted: $1,5,10,50,100,500$.}
\label{app:fig:stds_cats_and_dogs_2_matrices_square_loss}
\end{figure}

\subsubsection{Distance of Convergence in Each Principal Direction}
\label{app:sec:distance_to_convergence}
In Figs.~\ref{fig:stds_and_distances_all_cases}, \ref{app:fig:stds_cats_and_dogs_2_matrices_square_loss}, we see that weights in directions corresponding to larger principal components converge faster, in the sense that the std across different repetitions drops to zero faster along these directions. However, the optimal solution in each direction is drastically different. Directions corresponding to larger principal components tend to have lower values in their optimal solution. This could serve as a possible explanation for the convergence phenomenon---it might be possible that in all directions the distance of the current solution drops at the same speed, but as directions corresponding to higher principal values start nearer to the optimal solution, they seem to converge faster.

To test this hypothesis, we repeated the experiment from the previous section (Fig.~\ref{app:fig:stds_cats_and_dogs_2_matrices_square_loss}), and computed the distance in each direction from the optimal solution, normalized by the maximal difference achieved in each direction. We plot these results in Fig.~\ref{app:fig:rate_of_change_for_square_loss}. As suggested by our theory, we see that the normalized distance in each direction also drops faster in directions corresponding to larger principal components.

\begin{figure}[h!]
\begin{subfigure}{.19\textwidth}
      \centering
     \includegraphics[width=1\linewidth]{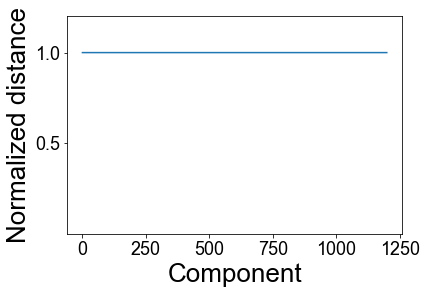}
    \end{subfigure}
    \begin{subfigure}{.19\textwidth}
      \centering
      \includegraphics[width=1\linewidth]{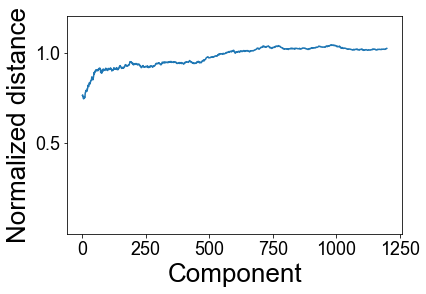}
    \end{subfigure}
    \begin{subfigure}{.19\textwidth}
      \centering
      \includegraphics[width=1\linewidth]{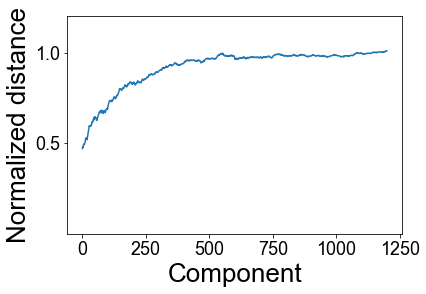}
    \end{subfigure}
    \begin{subfigure}{.19\textwidth}
      \centering
      \includegraphics[width=1\linewidth]{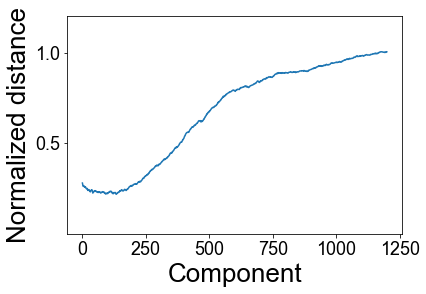}
    \end{subfigure}
    \begin{subfigure}{.19\textwidth}
      \centering
      \includegraphics[width=1\linewidth]{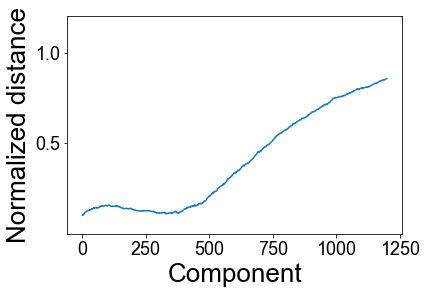}
    \end{subfigure}
\caption{Rate of convergence along the principal directions in different epochs. The value of the $X$-axis corresponds to the index of a principal eigenvalue, from the most significant to the least significant. The value of the $Y$-axis is the $L_2$ distance between the mean value of the weights and the optimal solution, normalized by the maximum distance across all epochs. The results are plotted for epochs: $1, 2, 10, 50, 200$, for $10$ $5$-layered linear networks trained on the cats and dogs dataset with the $L_2$ loss.}
\label{app:fig:rate_of_change_for_square_loss}
\end{figure}

\begin{figure}[h!]
\begin{subfigure}{.48\textwidth}
      \centering
     \includegraphics[width=1\linewidth]{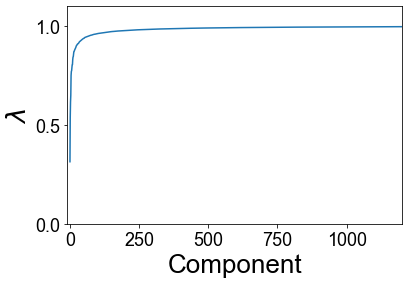}
     \caption{$\lambda_j$ for Fig.~\ref{fig:stds_cats_and_dogs_2_matrices}}
    \end{subfigure}
    \begin{subfigure}{.48\textwidth}
      \centering
      \includegraphics[width=1\linewidth]{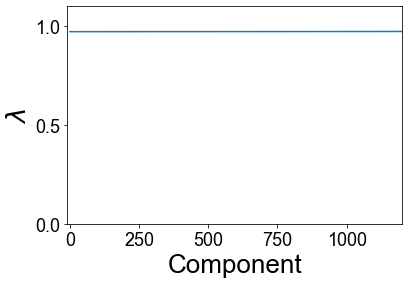}
      \caption{$\lambda_j$ for Fig.~\ref{fig:stds_cats_dogs_1_matrix_whitening}}
    \end{subfigure}
\caption{In the same experimental setup as in Fig.~\ref{fig:stds_and_distances_all_cases}, we show the evolution of $\lambda_j$ with epochs.}
\label{app:fig:lambda_j_figure}
\end{figure}

\subsection{Projection to Higher PC's}
\label{app:sec:project_to_higher_pc}

In Section~\ref{sec:pc_bias_non_linear_networks} we described an evaluation methodology, based on the creation of a modified \emph{test-set} by projecting each test example on the span of the first $P$ principal components. We repeat this experiment with VGG-19 networks on CIFAR-10, and plot the results in Fig.~\ref{app:fig:acc_on_pca_components}.

\subsection{Spectral Bias}
\label{app:sec:spectral_bias}

The \emph{spectral bias}, discussed in Section~\ref{sec:empricial-freq-bias}, can also induce a similar learning order in different networks. To support the discussion in Section~\ref{sec:empricial-freq-bias}, we analyze the relation between the \emph{spectral bias} and \emph{accessibility}, in order to clarify its relation to the \emph{Learning Order Constancy} and the \emph{PC-bias} (\S\ref{sec:spectral-bias-accessibility}) . First, however, we expand the scope of the empirical evidence for this effect to the classification scenario and real image data (\S\ref{sec:spectral-bias-in-classification}).

\subsubsection{Spectral Bias in Classification}
\label{sec:spectral-bias-in-classification}
\label{app:spectral}

\citet{rahaman2019spectral} showed that when regressing a 2D function by a neural network, the model seems to approximate the lower frequencies of the function before its higher frequencies. Here we extend this empirical observation to the classification framework. Thus, given frequencies $\kappa=(\kappa_1,\kappa_2,...,\kappa_m)$ with corresponding phases $\boldsymbol\varphi=(\varphi_1,\varphi_2,...,\varphi_m)$, we consider the mapping $\lambda:[-1,1]\rightarrow\mathbb{R}$ given by
\begin{equation}
\label{eq:spectral-bias}
    \lambda(z)=\sum_{i=1}^m \sin(2\pi \kappa_i z+\varphi_i):=\sum_{i=1}^m freq_i(z).
\end{equation}
Above $\kappa$ is strictly monotonically increasing, while $\boldsymbol\varphi$ is sampled uniformly. 

The classification rule is defined by $\lambda(z)\lessgtr 0$. We created a binary dataset whose points are fully separated by $\lambda(z)$, henceforth called the \emph{frequency dataset} (see the visualization in Fig.~\ref{fig:spectral_bias_classification_visualization} and details in \S\ref{app:datasets}). When training on this dataset, we observe that the frequency of the corresponding separator increases as learning proceeds, in agreement with the results of \citet{rahaman2019spectral}.

\begin{figure*}[htb]
\begin{subfigure}{1\textwidth}
    \begin{subfigure}{.98\textwidth}
      \centering
      \includegraphics[width=1\linewidth]{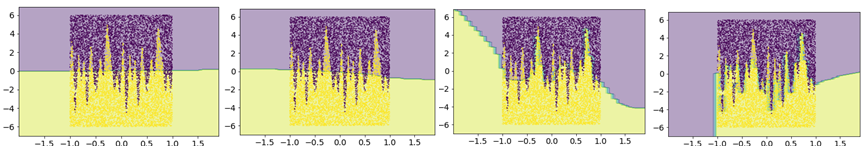}
    \end{subfigure}
\end{subfigure}

\caption{Visualization of the separator learned by st-VGG when trained on the frequency dataset, as captured in advancing epochs (from left to right): $1$, $100$, $1000$, $10000$. Each point represents a training example (yellow for one class and purple for the other). The background color represents the classification that the network predicts for points in that region. }
\label{fig:spectral_bias_classification_separating_line_visualization}
\end{figure*}
To visualize the decision boundary of an st-VGG network trained on this dataset as it evolves with time, we trained $N$=$100$ st-VGG networks. Since the data lies in $\R^2$, we can visualize it and the corresponding network's inter-class boundary at each epoch as shown in Fig.~\ref{fig:spectral_bias_classification_separating_line_visualization}
. We can see that the decision boundary incorporates low frequencies at the beginning of the learning, adding the higher frequencies only later on. The same qualitative results are achieved with other instances of st-VGG as well. We note that while the decision functions are very similar in the region where the training data is, at points outside this region they differ drastically across networks.

\subsubsection{Spectral Bias: Relation to Accessibility}
\label{sec:spectral-bias-accessibility}

\begin{figure}[thb!]
\begin{minipage}{.4\textwidth}
\begin{subfigure}{.87\textwidth}
    \begin{subfigure}{1\textwidth}
      \centering
      \includegraphics[width=1\linewidth]{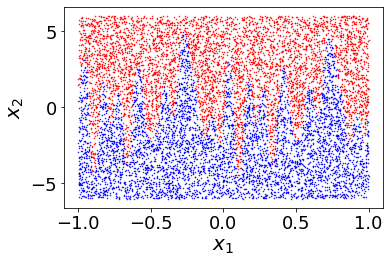}
    \end{subfigure}
\end{subfigure}

\caption{Visualization of the classification dataset used to extend \citet{rahaman2019spectral} to a classification framework.}
\label{fig:spectral_bias_classification_visualization}
\end{minipage}
\hspace{0.2cm}
\begin{minipage}{.6\textwidth}
\begin{subfigure}{1\textwidth}
  
    \begin{subfigure}{.48\textwidth}
      \centering
      \includegraphics[width=1\linewidth]{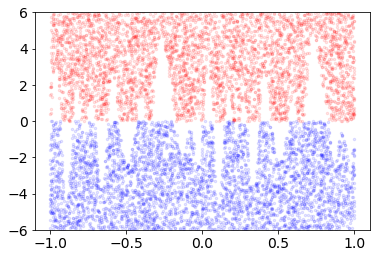}
      \caption{2 classes freq 0}
      \label{subfig:freq-first-2d-data-visual-freq0}
    \end{subfigure}
    \begin{subfigure}{.48\textwidth}
      \centering
      \includegraphics[width=1\linewidth]{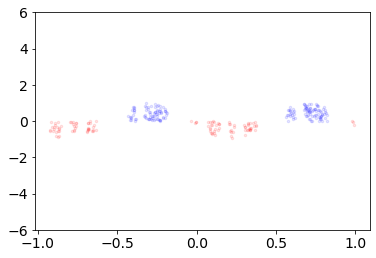}
      \caption{2 classes freq 1}
      \label{subfig:freq-first-2d-data-visual-freq1}
    \end{subfigure}

\end{subfigure}
\caption{Visualization of the \emph{critical frequency}, showing all the points in the 2D-frequency dataset with \emph{critical frequency} of (a) 0, and (b) 1.}
\label{fig:freq-first-2d-data-visual}
\end{minipage}
\end{figure}

To connect between the learning order, which is defined over examples, and the Fourier analysis of classifiers, we define for each example its \emph{critical frequency}, which characterizes the smallest number of frequencies needed to correctly classify the example. To illustrate, consider the \emph{frequency dataset} defined above. Here, the \emph{critical frequency} is defined as the smallest $j\in [m]$ such that $\lambda_j(z)\!=\!\sum_{i=1}^j freq_i(z)$ classifies the example correctly (see  Fig.~\ref{fig:freq-first-2d-data-visual}).

\begin{figure}[thp]
    \begin{subfigure}{1\textwidth}
        \begin{subfigure}{.49\textwidth}
          \centering
          \includegraphics[width=.85\linewidth]{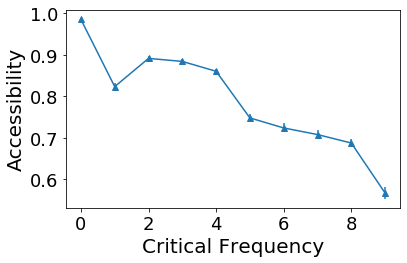}
          \caption{}
          \label{subfig:criticial_freq_accessiblity_correlation}
        \end{subfigure}
        \begin{subfigure}{.49\textwidth}
          \centering
          \includegraphics[width=.85\linewidth]{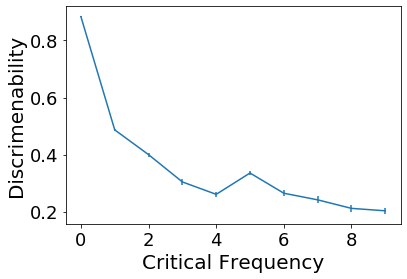}
          \caption{}
          \label{subfig:critical_freq_discremnability_correlation}
        \end{subfigure}
    \end{subfigure}
    \caption{(a) Correlation between \emph{critical frequency} and \emph{accessibility} score in the 2D-frequency dataset. (b) Correlation between \emph{discriminability} and \emph{critical frequency} in the 2D-frequency dataset.}
    \label{fig:accessiblity_first_freq_discrimination}
\end{figure}

In this binary classification task, we observe a strong connection between the order of learning and the \emph{critical frequency}. Specifically, we trained $N$=$100$ st-VGG networks on the \emph{frequency dataset}, and correlated the \emph{accessibility} scores with the \emph{critical frequency} of the examples (see \fig\ref{subfig:criticial_freq_accessiblity_correlation}). We see a strong negative correlation ($r=-0.93$, $p<10^{-2}$), suggesting that examples whose \emph{critical frequency} is high are learned last by the networks.

To see the effect of the \emph{spectral bias} in real classification tasks and extend the above analysis to natural images, we need to define a score that captures the notion of \emph{critical frequency}. To this end, we define the \emph{discriminability} measure of an example as the percentage out of its $k$ neighbors that share the same class as the example. Intuitively, an example has a low \emph{discriminability} score when it is surrounded by examples from other classes, which forces the learned boundary to incorporate high frequencies. In \fig\ref{subfig:critical_freq_discremnability_correlation} we plot the correlation between the \emph{discriminability} and the \emph{critical frequency} for the 2D frequency dataset. The high correlation ($r$=$-0.8$, $p<10^{-2}$) indicates that \emph{discriminability} indeed captures the notion of \emph{critical frequency}.

\paragraph{Discussion.} We find that the \emph{spectral bias} is connected to the LOC effect. Examples that can be classified correctly by a low-frequency function are also learned faster by the neural network. However, as the definition of such frequency is not trivial in higher dimensions, it is yet unclear if this result can be extended to image classification.

\section{Methodology}
\label{app:methods}

\subsection{Implementation Details and Hyperparameters}
\label{app:sec:hyper_params}

The results reported in Section \ref{sec:implication}represent the mean performance of $100$ st-VGG and linear st-VGG networks, trained on the small mammals dataset. The results reported in Section \ref{sec:implication} represent the mean performance of $10$ two-layered fully connected linear networks trained over the cats and dogs dataset. The results in Fig.~\ref{fig:learning_curves_with_label_change_all} represent the mean performance of the $100$ st-VGG network trained on the small mammals dataset. In every experimental setup, the network's hyper-parameters were coarsely grid-searched to achieve good performance over a validation set, for a fair comparison. Other hyper-parameters exhibit similar results.

\subsection{Generalization Gap}
\label{app:sec:noisy_images}

In Section~\ref{sec:implication} we discuss the evaluation of networks on datasets with amplified principal components. Examples of these images are shown in Fig.~\ref{fig:small_mammals_amplification_visualization}: the top row shows examples of the original images, the middle row shows what happens to each image when its $1.5\%$ most significant principal components are amplified, and the bottom row shows what happens when its $1.5\%$ least significant principal components are amplified. Amplification involved a factor of $10$, which is significantly smaller than the ratio between the values of the first and last principal components of the data. After amplification, all the images were re-normalized to have $0$ mean and std $1$ in every channel as customary.

\subsection{Architectures}
\label{app:architectures}

\paragraph{st-VGG.}
A stripped version of VGG which we used in many of the experiments. It is a convolutional neural network, containing 8 convolutional layers with 32, 32, 64, 64, 128, 128, 256, 256 filters respectively. The first 6 layers have filters of size $3\times3$, and the last 2 layers have filters of size $2\times2$. Every other layer is followed by a $2\times2$ max-pooling layer and a $0.25$ dropout layer. After the convolutional layers, the units are flattened, and there is a fully connected layer with 512 units followed by a $0.5$ dropout. The batch size we used was $100$. The output layer is a fully-connected layer with output units matching the number of classes in the dataset, followed by a softmax layer. We trained the network using the SGD optimizer, with cross-entropy loss. When training st-VGG, we used a learning rate of $0.05$.

\paragraph{Linear st-VGG.}
A linear version of the st-VGG network. In linear st-VGG, we change the activation function to the identity function, and replace max-pooling with average pooling with a similar stride. Similar to the non-linear case, a cross-entropy loss is being used. The network does not contain a softmax layer.

\paragraph{Linear fully connected network.}
An $L$-layered fully connected network. Each layer contains $1024$ weights, initialized with Glorot uniform initialization. $0.5$ dropout is used before the output layer. Networks are trained with an SGD optimizer, without momentum or $L_2$ regularization. Unless stated otherwise, a cross-entropy loss is being used. The network does not contain a softmax layer.

\subsection{Datasets}
\label{app:datasets}
In all the experiments and all the datasets, the data was always normalized to have $0$ mean and std $1$, in each channel separately.

\paragraph{Small Mammals.}
The small mammals dataset used in our experiments is the relevant super-class of the CIFAR-100 dataset. It contains $2500$ train images divided into 5 classes equally, and $500$ test images. Each image is of size $32\times 32\times 3$. This dataset was chosen due to its small size.

\paragraph{Cats and Dogs.}
The cats and dogs dataset is a subset of CIFAR-10. It uses only the 2 relevant classes, to create a binary problem. Each image is of size $32\times 32\times 3$. The dataset is divided to $20000$ train images ($10000$ per class) and $2000$ test images ($1000$ per class). This dataset is used when a binary problem is required.

\paragraph{ImageNet-20.}
The ImageNet-20 dataset \citep{russakovsky2015imagenet} is a subset of ImageNet containing $20$ classes. This data resembles ImageNet in terms of image resolution and data variability, but contains a smaller number of examples to reduce computation time. The dataset contains $26000$ train images ($1300$ per class) and $1000$ test images ($50$ per class). The choice of the $20$ classes was arbitrary, and contained the following classes: boa constrictor, jellyfish, American lobster, little blue heron, Shih-Tzu, scotch terrier, Chesapeake Bay retriever, komondor, snow leopard, tiger, long-horned beetle, warthog, cab, holster, remote control, toilet seat, pretzel, fig, burrito and toilet tissue.

\paragraph{Frequency dataset}
A binary 2D dataset, is used in Section~\ref{sec:empricial-freq-bias}, to examine the effects of spectral bias in classification. The data is define by the mapping $\lambda:[-1,1]\rightarrow\mathbb{R}$ given in (\ref{eq:spectral-bias}) by
\begin{equation*}
    \lambda(z)=\sum_{i=1}^m \sin(2\pi \kappa_i z+\varphi_i):=\sum_{i=1}^m freq_i(z),
\end{equation*}
with frequencies $\kappa=(\kappa_1,\kappa_2,...,\kappa_m)$ and corresponding phases $\boldsymbol\varphi=(\varphi_1,\varphi_2,...,\varphi_m)$. The classification rule is defined by $\lambda(z)\lessgtr 0$.

In our experiments, we chose $m=10$, with frequencies $\kappa_1=0,\kappa_2=1,\kappa_3=2,...,\kappa_{10}=9$. Other choices of $m$ yielded similar qualitative results. The phases were chosen randomly between $0$ and $2\pi$, and were set to be: $\varphi_1=0$, $\varphi_2=3.46$, $\varphi_3=5.08$, $\varphi_4=0.45$, $\varphi_5=2.10$, $\varphi_6=1.4$, $\varphi_7=5.36$, $\varphi_8=0.85$, $\varphi_9=5.9$, $\varphi_{10}=5.16$. As the first frequency is $\kappa_1=0$, the choice of $\varphi_0$ does not matter, and is set to $0$. The dataset contained $10000$ training points, and $1000$ test points, all uniformly distributed in the first dimension between $-1$ and $1$ and in the second dimension between $-2\pi$ and $2\pi$. The labels were set to be either $0$ or $1$, in order to achieve perfect separation with the classification rule $\lambda(z)$.

\end{document}